\documentclass{article}

\PassOptionsToPackage{numbers, sort&compress}{natbib}

\usepackage[preprint]{neurips_2022}

\newcommand*{\textcite}{\citet}
\newcommand*{\parencite}{\citep}
\newcommand*{\autocite}{\citep}

\usepackage{hyperref}       
\usepackage{bookmark}		
\bookmarksetup{open,numbered}

\usepackage{algorithm}
\usepackage{algpseudocode}
\newcommand*{\STATE}{\State}
\newcommand*{\FOR}{\For}
\newcommand*{\ENDFOR}{\EndFor}

\usepackage{algorithmrules}

\usepackage[utf8]{inputenc} 
\usepackage[T1]{fontenc}    
\usepackage{url}            
\usepackage{booktabs}       
\usepackage{amsfonts}       
\usepackage{nicefrac}       
\usepackage{microtype}      
\usepackage{xcolor}         

\hypersetup{
	colorlinks,
	linkcolor={red!50!black},
	citecolor={blue!50!black},
	urlcolor={blue!80!black}
}

\makeatletter
\renewcommand{\theHALG@line}{\thealgorithm.\arabic{ALG@line}}
\makeatother

\usepackage{amsmath}
\usepackage{amssymb}
\usepackage{mathtools}
\usepackage{amsthm}
\usepackage{suffix}  
\usepackage{nicefrac}
\usepackage{soul}
\usepackage{aligned-overset}  
\usepackage{interval}
\intervalconfig{soft open fences}

\usepackage{etoolbox}
\AtBeginEnvironment{algorithmic}{\setlength{\abovedisplayskip}{2pt}}
\AtBeginEnvironment{algorithmic}{\setlength{\belowdisplayskip}{2pt}}

\usepackage[acronym,nomain,numberedsection=autolabel]{glossaries}
\makeglossaries
\glsdisablehyper  

\usepackage[capitalize,noabbrev]{cleveref}
\crefname{algorithmrules}{Protocol}{Protocols}

\crefname{equation}{}{}

\theoremstyle{plain}
\newtheorem{theorem}{Theorem}[section]

\newtheorem{lemma}[theorem]{Lemma}
\newtheorem{corollary}[theorem]{Corollary}
\theoremstyle{definition}
\newtheorem{definition}[theorem]{Definition}
\newtheorem{assumption}[theorem]{Assumption}
\crefname{assumption}{Assumption}{Assumptions}
\theoremstyle{remark}
\newtheorem{remark}[theorem]{Remark}

\title{Online Meta-Learning in Adversarial Multi-Armed Bandits}

\author{%
  Ilya Osadchiy \\
  Faculty of Electrical Engineering\\
  Technion - Israel Institute of Technology\\
  Haifa, Israel\\
  \texttt{osadchiy.ilya@gmail.com} \\
  \And
  Kfir Y. Levy \\
  Faculty of Electrical Engineering\\
  Technion - Israel Institute of Technology\\
  Haifa, Israel\\
  \texttt{kfiryehud@gmail.com} \\
  \And
  Ron Meir \\
  Faculty of Electrical Engineering\\
  Technion - Israel Institute of Technology\\
  Haifa, Israel\\
  \texttt{rmeir@ee.technion.ac.il} \\
}

\newcommand*{\ie}{i.e.}
\newcommand*{\eg}{e.g.}

\newcommand*{\wrt}{w.r.t.}
\newcommand*{\otherwise}{\ensuremath{\text{otherwise}}}

\newcommand*{\Holder}{H\"{o}lder}

\newcommand*{\pth}[1]{\mathopen{}\left( #1\right)\mathclose{}}                 
\newcommand*{\brk}[1]{\mathopen{}\left[ #1\right]\mathclose{}}                 
\newcommand*{\braces}[1]{\mathopen{}\left\lbrace #1\right\rbrace\mathclose{} } 
\newcommand*{\abs}[1]{\mathopen{}\left| #1 \right|\mathclose{}}
\newcommand*{\norm}[1]{\mathopen{}\left\lVert #1 \right\rVert\mathclose{}}
\newcommand*{\angbr}[1]{\mathopen{}\left\langle #1 \right\rangle\mathclose{}}  

\renewcommand*{\exp}[1]{\mathrm{exp}\pth{#1}}

\newcommand*{\diff}{\ensuremath{\mathrm{d}}}
\newcommand*{\PartDiv}[2]{\frac{\partial #1}{\partial #2}}
\WithSuffix\newcommand\PartDiv*[3]{\frac{\partial^#1 #2}{\partial #3^#1}}
\newcommand*{\Grad}[1]{\nabla_{#1}}

\newcommand*{\innerProd}[2]{\angbr{#1,#2}}

\newcommand*{\onesVector}{\ensuremath{\mathbf{1}}}
\newcommand*{\basisVector}[1]{\ensuremath{\mathrm{e}_{#1}}}

\newcommand*{\reals}{\ensuremath{\mathbb R}}
\DeclareMathOperator*{\argmin}{argmin}

\newcommand*{\simplex}[1]{\Delta_#1}

\newcommand*{\strongcvx}[1]{\ensuremath{#1\text{-strongly convex}}}

\newcommand*{\betaDivergence}[3]{\ensuremath{D_B^{\pth{#1}}\pth{#2 \| #3}}}

\newcommand*{\Prob}{\ensuremath{\mathbb{P}}}
\newcommand*{\Indic}[1]{\ensuremath{\mathbb{I}_{\braces{#1}}}}

\newcommand*{\E}{\mathbb{E}} 
\newcommand*{\Excpt}[2]{\underset{#1\sim #2\,\,}{\E}}

\newcommand*{\Var}[1]{\ensuremath{\mathrm{Var}\pth{#1}}}
\newcommand*{\EntropySymbol}{H}

\newcommand*{\bigOSymbol}{\ensuremath{O}}  
\newcommand*{\bigO}[2]{\ensuremath{\bigOSymbol_{#1}\pth{#2}}}  
\newcommand*{\bigTheta}[2]{\ensuremath{\Theta_{#1}\pth{#2}}}  
\newcommand*{\bigOmega}[2]{\ensuremath{\Omega_{#1}\pth{#2}}}  

\newcounter{relctr} 
\everydisplay\expandafter{\the\everydisplay\setcounter{relctr}{0}} 

\newcommand\labelrel[2]{%
  \begingroup
    \refstepcounter{relctr}%
    \stackrel{\textnormal{(\alph{relctr})}}{\mathstrut{#1}}%
    \originallabel{#2}%
  \endgroup
}
\AtBeginDocument{\let\originallabel\label} 

\makeatletter
\newcommand{\subalign}[1]{%
	\vcenter{%
		\Let@ \restore@math@cr \default@tag
		\baselineskip\fontdimen10 \scriptfont\tw@
		\advance\baselineskip\fontdimen12 \scriptfont\tw@
		\lineskip\thr@@\fontdimen8 \scriptfont\thr@@
		\lineskiplimit\lineskip
		\ialign{\hfil$\m@th\scriptstyle##$&$\m@th\scriptstyle{}##$\hfil\crcr
			#1\crcr
		}%
	}%
}
\makeatother

\newcommand*{\tcref}[1]{\ensuremath{\text{\cref{#1}}}}
\newacronym{mab}{MAB}{multi-armed bandit}
\newacronym{ftl}{FTL}{follow the leader}
\newacronym{ftrl}{FTRL}{follow the regularized leader}
\newacronym{omd}{OMD}{online mirror descent}
\newacronym{inf}{INF}{implicitly normalized forecaster}
\newacronym{ewoo}{EWOO}{exponentially-weighted online-optimization}
\newacronym{eps-ewoo}{$\alpha$-EWOO}{$\alpha$-regularized exponentially-weighted online-optimization}
\newacronym{meta-inf}{meta-INF}{implicitly normalized forecaster with meta-learning}

\newcommand{\decision}{\ensuremath{x}}
\newcommand{\innNumRounds}{\ensuremath{T}}
\newcommand{\innCurRound}{\ensuremath{t}}
\newcommand*{\bestArmTrue}{\ensuremath{j^*}}
\newcommand*{\bestArmEst}{\ensuremath{\hat{j}}}
\newcommand{\innInitial}{\ensuremath{\phi}}
\newcommand{\lossFunc}{\ensuremath{f}}
\newcommand{\innLearningRate}{\ensuremath{\eta}}
\newcommand{\innPseudoLearningRate}{\ensuremath{v}}
\newcommand{\outNumRounds}{\ensuremath{S}}
\newcommand{\outCurRound}{\ensuremath{s}}
\newcommand{\innComparator}{\ensuremath{u}}
\newcommand*{\initComparator}{\ensuremath{\psi}}
\newcommand*{\initComparatorEst}{\ensuremath{\hat{\psi}}}
\newcommand*{\knownPrior}{
	\mspace{2mu}%
	\widetilde{\mspace{-2mu}\rule{0pt}{1.2ex}\smash[t]{\psi}}%
}
\newcommand*{\initLoss}{\ensuremath{\lossFunc^{\text{init}}}}
\newcommand*{\lrBLimit}{\ensuremath{D}}
\newcommand*{\lrAlpha}{\ensuremath{\problemScale}}
\newcommand*{\lrEpsilon}{\ensuremath{\alpha}}
\newcommand*{\lrGamma}{\ensuremath{\gamma}}
\newcommand*{\lrLoss}{\ensuremath{\lossFunc^{\text{lr}}}}
\newcommand*{\lrLossSurrogate}{\ensuremath{\tilde{\lossFunc}^{\text{lr}}}}
\newcommand*{\mabDim}{\ensuremath{d}}
\newcommand*{\numGoodArms}{\ensuremath{k}}
\newcommand*{\badArmsWeight}{\ensuremath{\zeta}}

\newcommand*{\mabDecision}{\ensuremath{y}}
\newcommand*{\lossEst}{\ensuremath{\hat{\lossFunc}}}
\newcommand*{\gapMin}{\ensuremath{\Delta}}
\newcommand*{\probLowerLimit}{\ensuremath{\delta}}

\newcommand*{\truncatedSimplex}[1]{\ensuremath{\mathcal{K}_{#1}}}

\newcommand*{\wrongArmProb}{\ensuremath{\epsilon_\probLowerLimit}}
\newcommand*{\anyWrongArmProb}{\ensuremath{{\mabDim \epsilon_\probLowerLimit}}}
\newcommand*{\q}{\ensuremath{q}}
\newcommand*{\EntropyTsallisScaled}[2]{\ensuremath{\tilde{\EntropySymbol}^{\pth{#1}}_{T}\pth{#2}}}
\newcommand*{\EntropyTsallisHalfSc}[1]{\EntropyTsallisScaled{\nicefrac{1}{2}}{#1}}
\newcommand*{\betaDivergenceHalf}[2]{\betaDivergence{\nicefrac{1}{2}}{#1}{#2}}
\newcommand{\regretInner}{\ensuremath{\mathrm{{Reg}}^{\mathrm{inner}}}}
\newcommand{\regretTotal}{\ensuremath{\mathrm{{Reg}}^{\mathrm{tot}}}}
\newcommand{\regretKnownPrior}{\ensuremath{\mathrm{{Reg}}^{\mathrm{prior}}}}
\newcommand{\regretLr}{\ensuremath{\mathrm{{Reg}}^{\mathrm{lr}}}}
\newcommand{\regretInit}{\ensuremath{\mathrm{{Reg}}^{\mathrm{init}}}}

\newcommand{\boundLrRegret}{\ensuremath{\mathrm{U}_{\mathrm{reg}}^\mathrm{lr}}}
\newcommand{\boundInitRegret}{\ensuremath{\mathrm{U}_{\mathrm{reg}}^\mathrm{init}}}
\newcommand{\boundInitEstimationCost}{\ensuremath{\mathrm{U}_{\mathrm{cost}}^{\initComparatorEst}}}
\newcommand{\boundExplCost}{\ensuremath{\mathrm{U}_{\mathrm{cost}}^\mathrm{expl}}}
\newcommand{\problemScale}{\ensuremath{\sigma}}

\newcommand{\algInner}{\ensuremath{\mathcal{A}_{\mathrm{in}}}}

\newcommand{\algOuterLr}{\ensuremath{\mathcal{A}_{\mathrm{out}}^{\mathrm{lr}}}}
\newcommand{\algOuterInit}{\ensuremath{\mathcal{A}_{\mathrm{out}}^{\mathrm{init}}}}

\begin{document}

\maketitle

\begin{abstract}
We study meta-learning for adversarial multi-armed bandits.
We consider the online-within-online setup, in which a player (learner) encounters a sequence of multi-armed bandit episodes.
The player’s performance is measured as regret against the best arm in each episode, according to the losses generated by an adversary.
The difficulty of the problem depends on the empirical distribution of the per-episode best arm chosen by the adversary.
We present an algorithm that can leverage the non-uniformity in this empirical distribution, and derive problem-dependent regret bounds.
This solution comprises an inner learner that plays each episode separately, and an outer learner that updates the hyper-parameters of the inner algorithm between the episodes.
In the case where the best arm distribution is far from uniform, it improves upon the best bound that can be achieved by any online algorithm executed on each episode individually without meta-learning.

\end{abstract}

\renewcommand{\proofname}{Proof sketch}

\section{Introduction}

Meta-learning is an important area of recent research, which considers the possibility to improve the performance of a machine leaning algorithm by applying it to many similar tasks.
This includes both practical applications \autocite{vanschoren2018meta,hospedales2020meta} and theoretical analysis.
One particular setting for which new performance guarantees has been recently provided is online-within-online meta-learning \autocite{khodak2019adaptive,denevi2019online}.
In this setting the learner (player) performs a sequence of tasks (episodes), each constituting a sequence of rounds.
On each round the learner makes a decision and suffers the corresponding loss, after which the loss function is revealed.
The performance is measured be the regret, comparing to the best static decision per task.

The assumption that the whole loss function becomes known to the player is referred to as the full-information setting.
However, in many applications this assumption is not realistic.
In the partial information setting only the value of the loss function at the decision made by the player (\ie{} the incurred loss) is revealed.
This work considers online-within-online meta-learning with partial information, specifically the adversarial \gls{mab}.

The \gls{mab} is a well-established framework to analyze sequential decision making with partial feedback.
It is formulated as a game that consists of $\innNumRounds$ rounds.
In each round the player chooses one out of $\mabDim$ arms (actions).
Each arm has a loss associated with it on the current round, and the player incurs the loss of the arm that he has chosen.
This loss is then revealed to the player, but not the loss of the other arms.
This work focuses on adversarial \gls{mab} \parencite{auer2002nonstochastic}, in which the losses are chosen by an adversary.
Other variants of \gls{mab} exist, such as the stochastic version in which each arm has an associated stationary distribution from which the losses are sampled.
We refer the reader to \autocite{lattimore2020bandit} for a comprehensive introduction to \gls{mab}.

A practical motivation for the online-within-online \gls{mab} setting is the domain of online medicine, and particularly mobile health, where bandits are becoming increasingly used in order to facilitate effective decision making based on continually incoming data from personal devices \cite{tewari2017ads}. One of the major challenges in this domain is the proper initialization of decision policies. We believe such issues could be alleviated within the meta-learning context we present, since decision making initialization is improved by our approach, as information is acquired across episodes.

To the best of our knowledge, this is the first work that considers online-within-online meta-learning for adversarial \acrlong*{mab}, and develops performance guarantees for them.
We formulate the online-within-online \acrlong*{mab} setting, including the quantification of meta-learning difficulty.
We propose an algorithm that takes advantage of task similarity in this setting.
We derive a problem-dependent bound on the expected regret of this algorithm. Under some assumptions on the problem size, and on the non-uniformity of the best arm's empirical distribution, it improves upon the best bound that can be achieved by any online algorithm executed on each episode individually without meta-learning.

To showcase the benefit of our approach consider a setting with $\mabDim$ arms, where the tasks are similar in the following sense: the empirical distribution of the best-arm-per-episode  is concentrated (with high probability) among a small group of $\numGoodArms \ll \mabDim$ arms. In this case we are able to show that asymptotically%
\footnote{Assuming a large enough number of episodes and rounds per episode.}
the total regret bound of our method is smaller by a factor of $\mabDim^{\nicefrac{1}{4}}$ compared to a standard online learner that does not apply meta-learning. Our method achieves this improvement without any prior knowledge regarding the empirical distribution of the best arms, but rather meta-learns to exploit this structure throughout the learning process.

\section{Related work}

Several recent works \autocite{khodak2019adaptive,denevi2019online,meunier2021meta} introduce algorithms for online-within-online meta-learning with provable regret bounds.
Each of these approaches consists of an inner algorithm that is applied to each episode and an outer algorithm that updates hyper-parameters of the inner algorithm.
\textcite{khodak2019adaptive} propose to use the episode-specific regret bound of the inner algorithm as the loss function for the outer algorithm.
This bound depends on the learning rate and the initialization point, and the outer algorithm treats each of these parameters as a separate online learning problem.
They quantify the problem difficulty as the generalized empirical variance of the per-episode best-in-hindsight points.
\textcite{meunier2021meta} further generalize this idea and propose a unified way to treat all the hyper-parameters.
\textcite{denevi2019online} propose to use a regularized empirical error as a performance measure instead of the regret, and their outer algorithm updates the initialization point based on the (sub-)gradient of this measure.
However, these works consider full information feedback, which is not always available.
Our work builds upon the framework of \textcite{khodak2019adaptive}, and extends it to the bandit setting.

The adversarial \acrlong*{mab} problem was introduced in the seminal work of \textcite{auer2002nonstochastic}.
They proposed the Exp3 algorithm, and proved that its regret bound meets the lower bound up to a logarithmic factor.
This factor has been later eliminated by the \gls{inf} algorithm, that was first introduced in \autocite{audibert2009minimax}, and further analyzed in \autocite{audibert2011minimax,abernethy2015fighting,orabona2019modern,zimmert2021tsallis}.

A setting that bears some similarity to the online-within-online adversarial \gls{mab} is known as the switching bandit, or as tracking a sequence of best arms.
In this problem, instead of using a static best arm as the comparator for regret calculation, a piecewise constant sequence of arms is used, with a limited number of arm switches.
To cope with this setting, the Exp3.S algorithm \autocite{auer2002nonstochastic} was proposed.
The difficulty of the best arms sequence is quantified only as the number of switches.
In contrast, in this work we add a measure of task similarity, which allows deriving a problem-dependent regret bound.
Additionally, we assume that the learner knows when the task is switched.
Such an assumption without the task similarity would degenerate the problem into a set of unrelated adversarial \gls{mab} instances, which could be solved by repeated application of Exp3 or \gls{inf}, and Exp3.S would have no benefit over this straightforward idea in this context.

A concurrent work \cite{balcan2022meta} deals with a setting that is very similar to the one discussed in this work. It makes fewer assumptions and proposes ways to learn more hyper-parameters.
However, their bound depends on the empirical distribution of arms estimated to be the best by the algorithm.
Such distribution is not known prior to playing the game, and may have significantly higher entropy compared to the empirical distribution of best arms, as no guarantees are provided on the probability of correct best arm identification.
In contrary, we make an additional assumption that allows us to bound the regret in terms of empirical distribution of true best arms, which is defined by the adversary and represents the difficulty of the problem.

Finally, there is an expanding recent literature on meta-learning for stochastic \gls{mab}s (\eg{} \autocite{maes2012meta,yang2020provable,cella2020meta,kveton2021meta,ortega2019meta,azizi2022non,peleg2022metalearning}).
However, this problem differs significantly from the adversarial setting.
An interesting direction for future work would be an algorithm that achieves optimal performance in online-within-online meta-learning for both stochastic and adversarial \gls{mab} (best of both worlds), similarly to \autocite{zimmert2021tsallis,bubeck2012best}.
In the context of online-within-online learning, there are two behaviors that can be stochastic or adversarial: the losses inside each episode, and the change of loss pattern between the episodes, giving rise to at least four variants of the problem.
Each of these variants is an interesting research direction, and so is a unified algorithms that simultaneously copes with all of them.


\section{Problem formulation}
\label{sec:problem}

\subsection{Online-within-online adversarial MAB}
\label{ssec:problem-owo-mab}

%

We now define the protocol of online-within-online adversarial \acrlong*{mab}, and introduce metrics for quantification of the meta-learning difficulty.

Consider a repeated game between a learner and an adversary, that consists of $\outNumRounds$ episodes, with $\innNumRounds_\outCurRound$ rounds in each episode $\outCurRound$.
In each round the player chooses one out of $\mabDim$ actions, and suffers a loss determined by the adversary for this action in this round.
This loss is revealed to the player, but not the losses of other  actions that she could choose. At the end of the episode, the loss accumulated by the player is compared to the loss of the best constant action for that episode, to yield the per-episode regret. At the end of the game, the total regret is calculated as a sum of all per-episode regrets.
This setup extends the one in \cite{khodak2019adaptive}, where full information about the loss function is revealed to the learner at each round, to the case of bandit information.

We consider the case of an oblivious adversary, that decides upon the loss of each action in each round of each episode before the game begins.
In this case the best action per episode is independent of the player's actions, and the notions of expected regret and pseudo-regret coincide \autocite{bubeck2012regret}.
For simplicity, the number of rounds is assumed to be equal across episodes,
\begin{math}
\innNumRounds_\outCurRound = \innNumRounds
\end{math}
for all $s$, and losses are bounded between $0$ and $1$.

%
%

\begin{definition}
\label{def:true-best-arm}
Let $ \lossFunc_{\outCurRound, \innCurRound, i}$ be the loss of arm $i$ on round $\innCurRound$ of episode $\outCurRound$. The best arm in episode $\outCurRound$ is
\begin{math}
j^*_{\outCurRound}
\triangleq
\argmin_{ i = 1, \ldots, \mabDim } \sum_{\innCurRound=1}^{\innNumRounds} \lossFunc_{\outCurRound, \innCurRound, i}.
\end{math}
\end{definition}

\begin{definition}
\label{def:total-regret}
Let $ \mabDecision_{\outCurRound, \innCurRound} $ be the arm played by the learner in round $\innNumRounds$ of episode $\outCurRound$.
The total regret of the learner is
\begin{math}
\regretTotal
\triangleq
\sum_{\outCurRound=1}^{\outNumRounds} \sum_{\innCurRound=1}^{\innNumRounds} \pth{ \lossFunc_{\outCurRound, \innCurRound, \mabDecision_{\outCurRound, \innCurRound}} - \lossFunc_{\outCurRound, \innCurRound, j^*_{\outCurRound}} }~.
\end{math}
\end{definition}

In the worst case the adversary can choose the best arm uniformly at random, and there is not much motivation for meta-learning.
Following \autocite{khodak2019adaptive}, we derive a regret bound that depends on the variability of the best-in-hindsight points between episodes.
We denote the empirical distribution of the best arm by $ \initComparator $, and its non-uniformity can be quantified by the Tsallis entropy $ \EntropyTsallisScaled{\q}{\initComparator} $ (see Definition \ref{def:EntropyTsallisScaled} and the discussion that follows it).
As we will see, the value of $ \EntropyTsallisScaled{\q}{\initComparator} $ is not revealed to the learner,
but affects its regret bound.

\begin{definition}
\label{def:initComparator}
The empirical distribution of the best arm is
\begin{math}
\initComparator_i
\triangleq
\frac{1}{\outNumRounds}
\sum_{\outCurRound=1}^{\outNumRounds}
\Indic{i = \bestArmTrue_{\outCurRound}}, \; i = 1 ,\ldots, \mabDim
.
\end{math}
\end{definition}
\begin{definition}
\label{def:EntropyTsallisScaled}
The Tsallis entropy%
\footnote{Compared to the common definition of Tsallis entropy, this version is scaled by factor $\nicefrac{1}{\q}$, which makes it \strongcvx{1} \wrt{} $ \norm{\cdot}_2 $ for any $\q$, instead of \strongcvx{\q}. This does not affect the regret bounds, because it can be compensated by the learning rate.}
of a distribution $\decision$ is 
\begin{equation*}
\begin{aligned}
\EntropyTsallisScaled{\q}{\decision}
&\triangleq
\frac{1}{\q \pth{1 - \q}}
\pth{
	\sum_{i=1}^{\mabDim} \decision_i^\q
	-
	1
}
,\quad
\q \in \interval[open]{0}{1}
.
\end{aligned}
\end{equation*}
\end{definition}

Tsallis entropy can be seen as a generalization of the standard Shannon entropy, and converges to it for $ \q \to 1 $.
It is used as the regularization function by the \gls{inf} algorithm, which achieves the optimal regret bound on \gls{mab} for $\q=\nicefrac{1}{2}$.
Subsequently \parencite{hazan2019introduction,orabona2019modern} it appears in regret bounds of derived algorithms either directly or as a Bregman divergence of negative Tsallis entropy $ - \EntropyTsallisScaled{\q}{x} $, also known as beta-divergence \autocite{cichocki2010families,basu1998robust}.

\begin{definition}
\label{def:betaDivergence}
Beta-divergence between distributions $x, y$ is 
\begin{equation*}
\begin{aligned}
\betaDivergence{\q}{x}{y}
&{\triangleq}
\frac{1}{\q \pth{1 - \q}}
\sum_{i=1}^{\mabDim}
\pth{
	\pth{1 - \q}
	y_i^\q
	+
	\frac{
		\q
		x_i
	}{
		y_i^{1 - \q}
	} 
	-
	x_i^\q
}
.
\end{aligned}
\end{equation*}
\end{definition}

In order to make use of the non-uniformity of $\initComparator$, and hence facilitate meta-learning, the learner must be able to identify which arm was the best in hindsight, at least with some probability.
The difficulty of such identification can be quantified by the gap between the average loss of the best arm and the average loss of any other arm.
We assume that this gap is bounded from below by $ \gapMin > 0 $ in each episode, see \cref{thm:assum:minimal-gap}.


\begin{assumption}
	\label{thm:assum:minimal-gap}
	On each episode, the average loss of the best arm is lower than the average loss of any other arm by at least $\gapMin > 0$, 
	\begin{math}
	\frac{1}{\innNumRounds} \sum_{\innCurRound=1}^{\innNumRounds} \pth{ \lossFunc_{\outCurRound, \innCurRound, i} - \lossFunc_{\outCurRound, \innCurRound, j^*_{\outCurRound}}} \geq \gapMin
	,\quad
	\forall i \neq j^*_\outCurRound
	,\;
	\forall \outCurRound
	.
	\end{math}
\end{assumption}

%

For problems where this assumption does not hold (\eg{} several arms have the same average loss), it might be possible to relax this requirement.
However, in such cases the empirical distribution of the best arm $\initComparator$ would not be informative or even not well-defined.
Therefore, another way to describe the best arm choice and to quantify its non-uniformity is needed.
For example, the relaxed assumption could be that the arms are divided into ``good'' and ``bad'' on each episode, and there is a gap between the average losses of these sets.
Alternatively, a certain structure of gaps between each pair of arms can be assumed, and some weighted version of entropy can be used to quantify the non-uniformity.
We leave these refinements to future work.

\subsection{Adversarial MAB with known prior}
\label{ssec:known-prior}

An auxiliary setting that we consider  in this paper is the case of a known prior.
A single episode of the \gls{mab} game is considered, but the adversary is forced to choose the best arm $\bestArmTrue$ randomly according to a distribution $\knownPrior$,
\ie{}
\begin{math}
\bestArmTrue \sim \knownPrior
\end{math}.
Moreover, $\knownPrior$ is \emph{known} to the player before the game begins. This represents the best performance that one can hope to achieve with meta-learning, and is used for comparison in \cref{sec:total-regret}.
Additionally, this setting is used in \cref{sec:q-choice}.

We note two special cases of this setting.
One is the uniform prior
\begin{math}
\knownPrior_i = \nicefrac{1}{\mabDim}
,\; \forall i
\end{math}.
This represents the worst-case scenario, and provides no benefit over a \gls{mab} with unknown prior.
Indeed, this case is used to construct the lower bound on the expected regret of any algorithm for adversarial \gls{mab} \autocite{auer2002nonstochastic}.

Another interesting case is an adversary that chooses the best arm from a small set of arms that does not grow with $\mabDim$.
Let this set constitute the first $\numGoodArms$ arms, and assume that most of the probability weight is divided uniformly between them.
A small probability $ \badArmsWeight $ is divided uniformly between all other arms, to prevent degenerate solutions that ignore those arms altogether.
This will be referred to as \emph{few-good-arms case}.
\begin{definition}
\label{def:few-good-arms}
The ``Few-good-arms'' scenario with $ \numGoodArms $ good arms and $ \badArmsWeight > 0 $ cumulative probability of bad arms, has prior distribution
\begin{math}
\knownPrior_i = 
\frac{1 - \badArmsWeight}{\numGoodArms}
\text{ if } i \leq \numGoodArms
,\;
\knownPrior_i = 
\frac{\badArmsWeight}{\mabDim - \numGoodArms}
\,\otherwise.
\end{math}
\end{definition}

\section{Proposed algorithm}
\label{sec:proposed-alg}

The proposed solution,
called \acrshort{meta-inf},
consists of three components:
the inner algorithm $\algInner$
and two outer algorithms $\algOuterLr, \algOuterInit$.
As described in \cref{alg:meta-mab-overview},
at the beginning of each episode $\outCurRound$ the algorithms $\algOuterLr, \algOuterInit$ determine the learning rate $\innLearningRate_\outCurRound > 0$ and the initialization point $\innInitial_\outCurRound$ respectively.
The algorithm $\algInner$, with hyper-parameters $(\innLearningRate_\outCurRound, \innInitial_\outCurRound)$, is then used by the player during episode $\outCurRound$.  
At the end of the episode the outer algorithms are updated using an upper bound on the regret of $\algInner$ as the loss function.

In contrast to the full information setting \cite{khodak2019adaptive}, $\algInner$ does not observe the true losses and therefore cannot compute the true best arm in hindsight to be used in its regret upper bound.
Instead, the best arm is estimated based on the unbiased estimate of the losses.
In \cref{ssec:best-arm-identification} we show that by enforcing enough exploration, 
the estimated best arm can be used to construct an estimate of the regret upper bound.
This estimated bound is then used by $\algOuterLr, \algOuterInit$ to update $\innLearningRate$ and $\innInitial$. 
Next we describe and analyze the building blocks of our approach.


\subsection{Inner algorithm}
\label{ssec:inner-alg}

\begin{algorithm}[tb]
	\caption{\Acrshort{meta-inf}. The comments are references to appropriate lines in \cref{alg:inner,alg:outer-lr,alg:outer-init}.}
	\label{alg:meta-mab-overview}
	\begin{algorithmic}[1]
		\STATE \textbf{Parameters:}
		\begin{math}
		\probLowerLimit,
		\lrEpsilon
		\end{math}.
		\STATE Initialize $\algOuterLr, \algOuterInit$ with parameters
		\begin{math}
		\probLowerLimit,
		\lrEpsilon
		\end{math}.
		\algorithmiccomment{\ref*{alg:outer-lr}.\ref{algline:outer-lr:initialize}, \ref*{alg:outer-init}.\ref{algline:outer-init:initialize}}
		\FOR{$\outCurRound = 1, \ldots, \outNumRounds$}
		\STATE Use $\algOuterLr, \algOuterInit$ to set
		\begin{math}
		\innInitial_\outCurRound,
		\innLearningRate_\outCurRound
		\end{math}.
		\algorithmiccomment{\ref*{alg:outer-lr}.\ref{algline:outer-lr:predict}, \ref*{alg:outer-init}.\ref{algline:outer-init:predict}}
		\STATE Initialize $\algInner$  with parameters
		\begin{math}
		\innInitial_\outCurRound,
		\innLearningRate_\outCurRound,
		\probLowerLimit
		\end{math}.
		\algorithmiccomment{\ref*{alg:inner}.\ref{algline:inner:initialize}}
		\FOR{$\innCurRound = 1, \ldots, \innNumRounds$}
		\STATE Use $\algInner$ to choose arm $\mabDecision_{\outCurRound, \innCurRound}$ and play it.
		\algorithmiccomment{\ref*{alg:inner}.\ref{algline:inner:predict}}
		\STATE Update $\algInner$ with loss $ \lossFunc_{\outCurRound, \innCurRound, \mabDecision_{\outCurRound, \innCurRound}} $.
		\algorithmiccomment{\ref*{alg:inner}.\ref{algline:inner:observe}}
		\ENDFOR
		\STATE Use $\algInner$ to estimate the best arm $ \bestArmEst_{\outCurRound} $.
		\algorithmiccomment{\ref*{alg:inner}.\ref{algline:inner:best-arm-estimate}}
		\STATE Update $\algOuterLr, \algOuterInit$  with $ \bestArmEst_{\outCurRound} $.
		\algorithmiccomment{\ref*{alg:outer-lr}.\ref{algline:outer-lr:observe}, \ref*{alg:outer-init}.\ref{algline:outer-init:observe}}
		\ENDFOR
	\end{algorithmic}
\end{algorithm}
\begin{algorithm}[tb]
	\caption{$\algInner$: \acrshort{inf} with $\q = \nicefrac{1}{2}$ and guaranteed exploration.}
	\label{alg:inner}
	\begin{algorithmic}[1]
		\STATE \textbf{Parameters:}
		\begin{math}
		\innInitial_\outCurRound,
		\innLearningRate_\outCurRound,
		\probLowerLimit
		\,.
		\end{math}
		\label{algline:inner:initialize}
		\STATE \textbf{Initialize.}
		\begin{math}
		\decision_{\outCurRound, 1} = \innInitial_\outCurRound
		\,.
		\end{math}
		\FOR{$\innCurRound = 1, \ldots, \innNumRounds$}
		\STATE \textbf{Choose} (sample) arm
		\begin{math}
			\mabDecision_{\outCurRound, \innCurRound} \sim \decision_{\outCurRound, \innCurRound}
			\,.
		\end{math}
		\label{algline:inner:predict}
		\STATE \textbf{Observe} loss
		\begin{math}
			\lossFunc_{\outCurRound, \innCurRound, \mabDecision_{\outCurRound, \innCurRound}}
			\,.
		\end{math}
		\label{algline:inner:observe}
		\STATE Construct unbiased loss estimate:
		\begin{math}
		\lossEst_{\outCurRound, \innCurRound, i} = 
		\frac{\lossFunc_{\outCurRound, \innCurRound, i}}{\decision_{\outCurRound, \innCurRound, i}}
		\cdot \Indic{i = \mabDecision_{\outCurRound, \innCurRound}}
		\,.
		\end{math}
		\label{algline:inner:unbiased-loss-estimate}
		\STATE Update decision point:
		\begin{equation*}
		\decision_{\outCurRound, \innCurRound + 1}
		=
		\argmin_{\substack{
				\decision \in \truncatedSimplex{\probLowerLimit}
		}}
		\innLearningRate_\outCurRound
		\innerProd{\lossEst_{\outCurRound, \innCurRound}}{\decision}
		+
		\betaDivergenceHalf{\decision}{\decision_{\outCurRound, \innCurRound}}
		\,.
		\end{equation*}
		\ENDFOR
		\STATE \textbf{Estimate} best arm in hindsight:
		\begin{math}
		\bestArmEst_{\outCurRound} = \argmin_{ i = 1, \ldots, \mabDim } \sum_{\innCurRound=1}^{\innNumRounds} \lossEst_{\outCurRound, \innCurRound, i}
		\;.
		\end{math}
		\label{algline:inner:best-arm-estimate}
	\end{algorithmic}
\end{algorithm}

A typical approach \autocite{hazan2019introduction,orabona2019modern} to a single-episode \gls{mab} problem is to estimate the full loss vector based on observed partial losses, and then run a full-information online algorithm on the estimated losses.
This full-information algorithm is usually a variant of \gls{ftrl} or \gls{omd} with an appropriate regularization function.
In each round it outputs a decision point on the simplex, and the arm to play is sampled according to this distribution.
%

%
When negative Tsallis entropy is used as regularization, this becomes an algorithm known as \acrfull{inf}, first introduced by \textcite{audibert2009minimax}.
Its per-episode regret bound is given in \cref{thm:inf-bound}.
The bound and the proof are not novel; the proof in our notations is given in \cref{sec:simple-tsallis} for convenience.
Note that for $\q \to 1$ the algorithm converges to the Exp3 \autocite{auer2002nonstochastic} algorithm.


\begin{lemma}
	\label{thm:inf-bound}
	Algorithm \gls{inf} with Tsallis parameter $\q$, initialization point $\innInitial$, and learning rate $\innLearningRate$, guarantees the following regret on a single episode against a static comparator $ \innComparator $:
	\begin{equation*}\label{key}
	\begin{aligned}
	\E \regretInner \pth{\innComparator}
	&{\leq}
	\frac{1}{\innLearningRate}
	\betaDivergence{\q}{\innComparator}{\innInitial}
	+
	\frac{\innLearningRate}{2}
	\innNumRounds
	\mabDim^{\q}~.
	\end{aligned}
	\end{equation*}
\end{lemma}

In this paper we use \gls{inf} with $ \q = \nicefrac{1}{2} $ in the inner algorithm $ \algInner$, 
which is the optimal parameter in the worst case of a uniform prior \autocite{audibert2009minimax}.
Further motivation for this choice is provided in \cref{sec:q-choice}.
We leave other choices of $\q$ and meta-learning the optimal $\q$ for future work.
The inner algorithm is summarized in \cref{alg:inner}, which includes an exploration-ensuring modification explained next.


\subsection{Best arm identification}
\label{ssec:best-arm-identification}

The cornerstone of our meta-learning scheme is the ability to identify the best arm of each episode in hindsight.
In the full information settings, such as \autocite{khodak2019adaptive,meunier2021meta}, this is simply the point that minimizes the sum of loss functions.
In the bandit setting it has to be estimated based on the observed partial losses.

Based on the analysis in \autocite{abbasi2018best}, high-probability identification of the best arm can be performed if the sampling probability of each arm on each round is bounded from below (\cref{thm:identification-probability}).
To ensure this property, the inner algorithm is modified to take decisions on the truncated simplex.
We refer to this as ``guaranteed exploration''.
The best arm is then estimated as the one with the lowest cumulative estimated loss.
The resulting inner algorithm is summarized in \cref{alg:inner}.

\begin{definition}
\label{def:truncated-simplex}
For
$ \probLowerLimit \in \interval{0}{\nicefrac{1}{\mabDim}} $
the truncated simplex is
\begin{math}
\truncatedSimplex{\probLowerLimit} \triangleq \braces{ \decision :\; \decision_i \geq \probLowerLimit, \,\forall i, \; \sum_{i=1}^{\mabDim} \decision_i = 1 }
.
\end{math}
\end{definition}

\begin{lemma}
\label{thm:identification-probability}
Let $ \braces{\decision_{\outCurRound, \innCurRound}}_{\innCurRound=1}^{\innNumRounds} $ be decision points of an inner algorithm in episode $\outCurRound$, such that
$ \decision_{\outCurRound, \innCurRound} \in \truncatedSimplex{\probLowerLimit},\; \forall \innCurRound$.
Let the true losses $ {\lossFunc_{\outCurRound, \innCurRound}} $ satisfy \cref{thm:assum:minimal-gap}.
Let $ \lossEst_{\outCurRound, \innCurRound} $ be the unbiased estimate of the losses, used to compute the estimated best arm  $\bestArmEst_{\outCurRound}$ (see \cref{algline:inner:unbiased-loss-estimate,algline:inner:best-arm-estimate} of \cref{alg:inner}).
%
Then the probability of correct estimation is bounded from below by
\begin{equation*}
\begin{aligned}
\Prob \pth{ \bestArmEst_{\outCurRound} = {\bestArmTrue_{\outCurRound}} }
&\geq
%
1 - \anyWrongArmProb
,
\quad
\wrongArmProb
\triangleq
\exp{
	- \frac{3}{28}
	\gapMin^2 \probLowerLimit \innNumRounds 
}
.
\end{aligned}
\end{equation*}
\end{lemma}
\begin{proof}
The proof for a specific choice of $\probLowerLimit$ can be found in \autocite[Appendix~F]{abbasi2018best}.
Using similar ideas, the general case is shown in \cref{sec:best-arm-identification-full}.
\end{proof}

This bound is only meaningful when $ \mabDim \wrongArmProb < 1 $.
In order to simplify the analysis we consider a stronger requirement $ \mabDim^2 \wrongArmProb \leq 1 $, which holds when \cref{thm:assum:positive-probability} is satisfied.

\begin{assumption}
	\label{thm:assum:positive-probability}
	Problem parameters $ \innNumRounds, \mabDim, \gapMin $, and algorithmic parameter $ \probLowerLimit $, ensure that 
	\begin{equation*}\label{key}
	\begin{aligned}
	\frac{56 \log{\mabDim}}{3 \gapMin^2 \innNumRounds}
	\leq
	\probLowerLimit
	\leq
	\frac{1}{\mabDim}
	\;.
	\end{aligned}
	\end{equation*} 
\end{assumption}

\begin{algorithm}[tb]		
	\caption{$\algOuterLr$: \acrshort{eps-ewoo}.}
	\label{alg:outer-lr}
	\begin{algorithmic}[1]
		\STATE \textbf{Parameters:}
		\begin{math}
		\probLowerLimit,
		\lrEpsilon
		\,.		
		\end{math}
		\label{algline:outer-lr:initialize}
		\STATE \textbf{Initialize.}
		\begin{math}
		\lrBLimit = {\frac{\sqrt{2}}{
			\sqrt{
				1 -
				\anyWrongArmProb
			}
			\sqrt[4]{\probLowerLimit}
		}},\;
		%
		%
		\lrGamma = \frac{2}{\lrAlpha \lrBLimit} \min\braces{\frac{\lrEpsilon^2}{\lrBLimit^2}, 1}
		\;.
		\end{math}		
		\FOR{$\outCurRound = 1, \ldots, \outNumRounds$}
		\STATE \textbf{Set}
		\begin{equation*}
		\innLearningRate_\outCurRound
		=
		\lrAlpha \;
		\frac{
			\int_{\lrEpsilon}^{\sqrt{\lrBLimit^2 + \lrEpsilon^2}}
			\innPseudoLearningRate \;
			\exp{ - \lrGamma \sum_{\tau=1}^{\outCurRound-1} \lrLossSurrogate_{\tau} \pth{\innPseudoLearningRate} } \diff \innPseudoLearningRate
		}{
			\int_{\lrEpsilon}^{\sqrt{\lrBLimit^2 + \lrEpsilon^2}}
			\exp{ - \lrGamma \sum_{\tau=1}^{\outCurRound-1} \lrLossSurrogate_{\tau} \pth{\innPseudoLearningRate} } \diff \innPseudoLearningRate
		}
		\;.
		\end{equation*}
		\label{algline:outer-lr:predict}
		\STATE \textbf{Observe} the initialization point $ \innInitial_\outCurRound $ chosen by $ \algOuterInit $ and the estimated best arm $ \bestArmEst_{\outCurRound} $.
		\label{algline:outer-lr:observe}
		\STATE Construct the regularized loss
		\begin{equation*}
		\lrLossSurrogate_\outCurRound \pth{\innPseudoLearningRate}
		{=}
		\lrAlpha
		\pth{
			\pth{
				\frac{1}{
					{
						1 -
						\anyWrongArmProb
					}
				}
				\betaDivergenceHalf{\basisVector{\bestArmEst_{\outCurRound}}^\probLowerLimit}{\innInitial_\outCurRound} 
				+
				\lrEpsilon^2
			}
			\frac{1}{\innPseudoLearningRate}
			+
			\innPseudoLearningRate
		}
		\;.
		\end{equation*}
		\ENDFOR
		\STATE \textbf{Note:} For $\outCurRound = 1$ the decision is
		\begin{math}
		\innLearningRate_\outCurRound = \frac{1}{2} \lrAlpha \pth{\sqrt{\lrBLimit^2 + \lrEpsilon^2} + \lrEpsilon}
		\end{math}.
	\end{algorithmic}
\end{algorithm}
\begin{algorithm}[tb]
	\caption{$\algOuterInit$: \acrshort{ftl}.}
	\label{alg:outer-init}
	\begin{algorithmic}[1]
		\STATE \textbf{Parameters:}
		\begin{math}
		\probLowerLimit
		\end{math}.
		\label{algline:outer-init:initialize}
		\FOR{$\outCurRound = 1, \ldots, \outNumRounds$}
		\STATE \textbf{Set}
		\begin{math}
		\innInitial_\outCurRound
		=
		\argmin_{\innInitial} \sum_{\tau=1}^{\outCurRound-1} \betaDivergenceHalf{\basisVector{\bestArmEst_{\tau}}^\probLowerLimit}{\innInitial}
		=
		\frac{1}{\outCurRound-1} \sum_{\tau=1}^{\outCurRound-1} \basisVector{\bestArmEst_{\tau}}^\probLowerLimit
		\;.
		\end{math}
		\label{algline:outer-init:predict}
		\STATE \textbf{Observe} the estimated best arm $ \bestArmEst_{\outCurRound} $.
		\label{algline:outer-init:observe}
		\ENDFOR
		\STATE \textbf{Note:} For $\outCurRound = 1$ the decision is $ \innInitial_{\outCurRound} = \frac{1}{\mabDim} \onesVector $.
	\end{algorithmic}
\end{algorithm}

Taking decisions on the truncated simplex induces additional regret, compared to the unmodified \gls{inf} algorithm.
But the resulting high probability of correct best arm estimation allows us to bound the regret against the true best arm using the estimated best arm,
as shown in \cref{thm:observable-inner-regret}.
Such a bound can then be used as the loss function of the outer algorithms.

We use the following notations throughout the rest of the paper.
The vector $ \basisVector{i}^{\probLowerLimit} $ denotes a mixture between one-hot and uniform distributions with weight $\probLowerLimit \mabDim$,
and $\problemScale$ denotes the problem's scale,
\begin{equation}
\label{eq:def:mix-one-hot-uniform-problemScale}
\basisVector{i}^{\probLowerLimit}
{\triangleq}
\pth{1 - \probLowerLimit \mabDim} \basisVector{i}
+
\probLowerLimit \onesVector
,\qquad
\problemScale \triangleq \frac{1}{\sqrt{2}} \sqrt{\innNumRounds}\sqrt[4]{\mabDim}
\;.
\end{equation}

\begin{lemma}
\label{thm:observable-inner-regret}
The expected regret on episode $\outCurRound$ using \acrshort{inf} with $\q = \nicefrac{1}{2}$ and guaranteed exploration (\cref{alg:inner}) is bounded from above as
\begin{equation*}
\begin{aligned}
\E \regretInner \pth{\basisVector{{\bestArmTrue_\outCurRound}}}
&{\leq}
%
\frac{1}{
	\innLearningRate_\outCurRound \pth{
		1 -
		\anyWrongArmProb
}}
\E
\betaDivergenceHalf{\basisVector{{\bestArmEst_\outCurRound}}^{\probLowerLimit}}{\innInitial_\outCurRound}
%
%
+ 
\problemScale^2
\innLearningRate_\outCurRound
+
\boundExplCost
\end{aligned}
\end{equation*}
where
$ \boundExplCost $ is the cost of guaranteed exploration
\begin{math}
\boundExplCost
=
\probLowerLimit \innNumRounds
\pth{\mabDim - 1}
\,.
\end{math}
\end{lemma}
\begin{proof}
The regret is bounded in two steps:
\begin{equation*}
\begin{aligned}
\E \regretInner \pth{\basisVector{{\bestArmTrue_\outCurRound}}}
&\labelrel{\leq}{subeq:truncated-simplex}
\frac{
	\betaDivergenceHalf{\basisVector{{\bestArmTrue_\outCurRound}}^{\probLowerLimit}}{\innInitial_\outCurRound}
}{
	\innLearningRate_\outCurRound
}
+ 
\problemScale^2
\innLearningRate_\outCurRound
+
\boundExplCost
%
\,\labelrel{\leq}{subeq:observed-divergence}\,
\frac{
	\E
	\betaDivergenceHalf{\basisVector{{\bestArmEst_\outCurRound}}^{\probLowerLimit}}{\innInitial_\outCurRound}
}{\innLearningRate_\outCurRound \pth{
	1 -
	\anyWrongArmProb
}}
%
%
+
\problemScale^2
\innLearningRate_\outCurRound
+
\boundExplCost
\end{aligned}
\end{equation*}

Inequality~\eqref{subeq:truncated-simplex} follows from considering the algorithm's decisions and the comparator to belong to $ \truncatedSimplex{\probLowerLimit} $, and adding the regret between this limited comparator and the true one.
Inequality~\eqref{subeq:observed-divergence} is an application of \cref{thm:identification-probability} to the expectation.
See \cref{sec:tsallis-truncated-simplex} for details.
\end{proof}


The bound in
\Cref{thm:observable-inner-regret}
relies on computing the expectation of the beta-divergence with respect to best arm estimation, and this cannot be performed by the player.
Instead, a single sample is used as an unbiased estimator of the bound, 
\begin{equation}\label{eq:observable-inner-regret-sample}
\begin{aligned}
&
\frac{
	1
}{\innLearningRate_\outCurRound \pth{
		1 -
		\anyWrongArmProb
}}
	\betaDivergenceHalf{\basisVector{{\bestArmEst_\outCurRound}}^{\probLowerLimit}}{\innInitial_\outCurRound}	
%
%
+
\problemScale^2
\innLearningRate_\outCurRound
+
\boundExplCost
.
\end{aligned}
\end{equation}

\subsection{Meta-learning the learning rate}

The estimate of the single-episode regret bound \eqref{eq:observable-inner-regret-sample} is used as the loss to be minimized by the outer algorithms.
It depends (indirectly) on $ \bestArmTrue_\outCurRound $ that is chosen by the adversary.
Additionally, it is affected by the noise of the best arm estimation,  according to \cref{thm:identification-probability}.
However it is important to note that these functions become known to the outer learner after each step, \ie{} this is full information setting.
We will treat these losses as if they were just generated by an adversary (that randomized its decision).
Following \autocite{khodak2019adaptive}, the meta-learning is divided into two separate problems: $\algOuterLr$ treats the learning rate and $\algOuterInit$ treats the initialization point.


After variable change 
$ \innLearningRate = \problemScale \innPseudoLearningRate $ 
the losses of $\algOuterLr$ are
\begin{equation}\label{eq:lr-meta-loss}
\begin{aligned}
\lrLoss_\outCurRound \pth{\innPseudoLearningRate}
&{=}
\problemScale
\pth{
	{
		\frac{
			1
		}{
			{
				1 -
				\anyWrongArmProb
			}
		}
	}
	\betaDivergenceHalf{\basisVector{\bestArmEst_{\outCurRound}}^\probLowerLimit}{\innInitial_\outCurRound} 
	\frac{1}{\innPseudoLearningRate}
	+
	\innPseudoLearningRate
}
\,.
\end{aligned}
\end{equation}

Similarly to \autocite{khodak2019adaptive}, we use the \gls{eps-ewoo} algorithm, shown in \cref{alg:outer-lr}.
The addition of $ \frac{\lrAlpha \lrEpsilon^2}{\innPseudoLearningRate} $ term makes the loss well-behaved and suitable to be used with the \acrshort{ewoo} algorithm \autocite{hazan2007logarithmic}, which is a one-dimensional continuous version of the exponentiated weights method.

\begin{lemma}
	\label{thm:lr-meta-regret}
	The regret of \gls{eps-ewoo} (\cref{alg:outer-lr}) on the loss function defined in \cref{eq:lr-meta-loss} is bounded as
	\begin{equation*}
	\begin{aligned}
	\E \regretLr \pth{\innPseudoLearningRate}
	\leq
	\boundLrRegret \pth{\innPseudoLearningRate}
	%
	%
	%
	\leq
	\bigOSymbol \bigg(
	\problemScale \outNumRounds 
	\min\braces{\frac{\lrEpsilon^2}{\innPseudoLearningRate}, \lrEpsilon}
	%
	%
	+
	\frac{
		\problemScale
		\log\pth{\outNumRounds}
	}
	{ \lrEpsilon^2 {\pth{1 - \anyWrongArmProb}}^{\nicefrac{3}{2}} \probLowerLimit^{\nicefrac{3}{4}} }
	\bigg)
	.
	\end{aligned}
	\end{equation*}
\end{lemma}
\begin{proof}
As shown in \cref{sec:lr-metalearning},
the beta-divergence on the truncated simplex can be bounded as
\begin{math}
\betaDivergenceHalf{\basisVector{\bestArmEst_{\outCurRound}}^\probLowerLimit}{\innInitial_\outCurRound}
{\leq}
\frac{2}{
	\sqrt{\probLowerLimit}
}
\end{math},
allowing us to use the \gls{eps-ewoo} regret bound \autocite{khodak2019adaptive} with the appropriate parameters.
\end{proof}

\subsection{Meta-learning the initialization point}

The loss of $\algOuterInit$ is the part of \cref{eq:observable-inner-regret-sample} that depends on $ \innInitial_\outCurRound $, 
\begin{equation}\label{eq:init-meta-loss}
\begin{aligned}
\initLoss_\outCurRound \pth{\innInitial}
&{=}
\betaDivergenceHalf{\basisVector{\bestArmEst_{\outCurRound}}^\probLowerLimit}{\innInitial}. 
\end{aligned}
\end{equation}
We use \gls{ftl} to cope with these non-convex losses.

\begin{lemma}
	\label{thm:init-meta-regret}
	The regret of \gls{ftl} (\cref{alg:outer-init}) w.r.t.~the loss sequence \cref{eq:init-meta-loss} is bounded as
	\begin{equation*}
	\begin{aligned}
	\E \regretInit \pth{\innInitial}
	\leq
	\boundInitRegret
	%
	%
	%
	\leq
	\bigO{}{
		\frac{\sqrt{\mabDim}}{ \sqrt{\probLowerLimit}}
		\log{\outNumRounds}
	}
	.
	\end{aligned}
	\end{equation*}
\end{lemma}
\begin{proof}
	These losses are not convex.
	However, they comprise a sequence of Bregman divergences of \strongcvx{1} functions, and therefore a strongly convex coupling argument
	\autocite[Proposition~B.1]{khodak2019adaptive}
	can be applied.
	See \cref{sec:init-metalearning-regret} for details.
\end{proof}

The overall regret bound (see \cref{sec:total-regret}) additionally depends on the loss associated with the best initialization point in hindsight.
In the full information setting that would just be the entropy of the empirical distribution of the best arms.
On the other hand, the losses in \cref{eq:init-meta-loss} depend on $ \bestArmEst_\outCurRound $, which is an estimator constructed by the player.
The expected entropy of the estimated best arm's distribution can be higher than the entropy of the empirical distribution of the true best arms.
In \cref{thm:init-metalearning-bih-loss} we show that the difference between these entropies is bounded,
and the bound does not depend on specific player actions.

\begin{lemma}
\label{thm:init-metalearning-bih-loss}
The expected loss of the best-in-hindsight point of the loss functions defined in \cref{eq:init-meta-loss} is bounded as
\begin{equation*}
\begin{aligned}
\E
\min_{\innInitial \in \truncatedSimplex{\probLowerLimit}}
\sum_{\outCurRound=1}^{\outNumRounds}
\betaDivergenceHalf{\basisVector{\hat{j}_{\outCurRound}}^\probLowerLimit}{\innInitial}
&\leq
\boundInitEstimationCost
+
\EntropyTsallisHalfSc{\initComparator}
\outNumRounds
%
%
%
\leq
\bigO{}{
\frac{\anyWrongArmProb}{\sqrt{\probLowerLimit}}
\outNumRounds 
+
\EntropyTsallisHalfSc{\initComparator}
\outNumRounds
}.
\end{aligned}
\end{equation*}
\end{lemma}
\begin{proof}
We consider the expected loss of the best initialization under the worst-case distribution of the best arm estimation, which satisfies \cref{thm:identification-probability}.
We then use convexity of negative Tsallis entropy to bound the difference from the empirical distribution of true best arms.
See \cref{sec:init-metalearning-hindsight-loss}.
\end{proof}

\subsection{Total regret upper bound}
\label{sec:total-regret}

The overall algorithm, described in \cref{alg:meta-mab-overview,alg:inner,alg:outer-lr,alg:outer-init}, achieves a problem-dependent bound on the total regret, as shown in \cref{thm:total-regret}.
Its proof follows from \cref{thm:observable-inner-regret,thm:lr-meta-regret,thm:init-meta-regret,thm:init-metalearning-bih-loss} and is detailed in \cref{sec:total-regret-proof}.

\begin{theorem}[Regret upper bound]
	\label{thm:total-regret}
	Under \cref{thm:assum:minimal-gap,thm:assum:positive-probability}, the expected total regret guaranteed by \acrshort{meta-inf} is
	\begin{equation*}\label{key}
	\begin{aligned}
	\E \regretTotal
	&{\leq}
	\min_{\innPseudoLearningRate}
	\bigg{[}
	 \boundExplCost \outNumRounds
	+
	\boundLrRegret \pth{\innPseudoLearningRate}
	+ 
	\problemScale
	\outNumRounds
	\innPseudoLearningRate
	%
	%
	+
	\frac{\problemScale}{1 - \mabDim \wrongArmProb}
	\pth{
		\boundInitRegret 
		+
		\boundInitEstimationCost
		+
		\EntropyTsallisHalfSc{\initComparator}
		\outNumRounds
	}
	\frac{1}{\innPseudoLearningRate}
	\bigg{]}.
	\end{aligned}
	\end{equation*}
	Here
		$\boundExplCost$ is the cost associated with guaranteed exploration (see \cref{thm:observable-inner-regret}).
		$\boundLrRegret$ is an upper bound on the regret of meta-learning the learning-rate (see \cref{thm:lr-meta-regret}).
		$\boundInitRegret$ is an upper bound on the regret of meta-learning the initialization point (see \cref{thm:init-meta-regret}).
		$\boundInitEstimationCost$ is an upper bound on the difference between the entropy of empirical distribution of the true best arms, and the expected estimated entropy (see \cref{thm:init-metalearning-bih-loss}).
\end{theorem}

The bound depends on the algorithmic parameters $\lrEpsilon$ and $\probLowerLimit$, and also on the parameter $\innPseudoLearningRate$  which is used in the analysis but does not affect the algorithm's behavior.
For known $\gapMin$, the asymptotic bound achieved by optimizing these parameters is given in \cref{thm:total-regret-asymptotic}.
The bound for unknown $\gapMin$ is provided in \cref{sec:parameter-optimization}.

\begin{corollary}
	\label{thm:total-regret-asymptotic}
	Let \cref{thm:assum:minimal-gap,thm:assum:positive-probability} hold, and additionally assume $ \innNumRounds \geq \bigOmega{}{\frac{\mabDim {\pth{\log \mabDim}}^{\nicefrac{7}{3}}}{\gapMin^{\nicefrac{10}{3}}} } $.
	For parameter values
\begin{math}
	\lrEpsilon
	=
	\bigTheta{}{
		\frac{
			\sqrt[3]{\log{\outNumRounds}}
		}{
			\sqrt[3]{\outNumRounds}
			\sqrt{ 1 - \mabDim \wrongArmProb }
			\sqrt[4]{\probLowerLimit}
		}
	}
	,\,
	\probLowerLimit
	=
	\bigTheta{}{
		\frac{1}{ 
			\gapMin ^ {\nicefrac{4}{7}}
			\innNumRounds ^ {\nicefrac{4}{7}}
			\mabDim ^ {\nicefrac{3}{7}}
		}
	},
\end{math}
	the expected total regret of \acrshort{meta-inf} satisfies the asymptotic bound
	\begin{equation*}\label{key}
	\begin{aligned}
	\E \regretTotal
	&{\leq}
	\bigOSymbol \Bigg{(}
	\outNumRounds \sqrt{\innNumRounds} \sqrt{\mabDim}
	\Bigg{(}
	\sqrt{\frac{
			\EntropyTsallisHalfSc{\initComparator}
		}{
			\sqrt{\mabDim}
	}}
	%
	%
	+
	\frac{
		\mabDim ^ {\nicefrac{1}{14}}
	}{
		\gapMin ^ {\nicefrac{4}{7}}
		\innNumRounds ^ {\nicefrac{1}{14}}
	}
	+
	\frac{
		\sqrt[3]{\log{\outNumRounds}}
		\innNumRounds ^ {\nicefrac{1}{7}}	
	}{
		\sqrt[3]{\outNumRounds}
		\mabDim ^ {\nicefrac{1}{7}}
	}
	+
	\frac{
		\sqrt{ \log{\outNumRounds} }
		\innNumRounds ^ {\nicefrac{1}{7}}
		\mabDim ^ {\nicefrac{3}{28}}
	}{
		\sqrt{\outNumRounds}
	}
	%
	%
	\Bigg{)}
	\Bigg{)}
	.
	\end{aligned}
	\end{equation*}
\end{corollary}
\begin{proof}
	For simplicity, during parameter optimization we consider $ \lrEpsilon $ instead of $ \min\braces{\frac{\lrEpsilon^2}{\innPseudoLearningRate}, \lrEpsilon} $ in $\boundLrRegret$.
	We then optimize $\lrEpsilon$ and $\innPseudoLearningRate$ separately. 
	Optimization of $\innPseudoLearningRate$ produces a square root of a sum of several terms, which is bounded by the sum of square roots.
	Finally, out of all terms that depend on $\probLowerLimit$ we minimize the sum of those that originate from $ \boundInitEstimationCost $ and $ \boundExplCost $.
	See \cref{sec:parameter-optimization} for details.
\end{proof}

\begin{table*}[t]
	\caption{Asymptotic regret bounds for different problem size regimes.
	}
	\label{tab:proposed-bound}
	\begin{center}
		\begin{small}
			\begin{tabular}{lccc|c}
				\toprule
				\textsc{Algorithm} & $ \outNumRounds $ & $ \innNumRounds $ & $ \EntropyTsallisHalfSc{\initComparator} $ & \textsc{Total regret bound} \\
				\midrule
				\Acrshort{inf} &
				any & any & any &
				\begin{math}
				\bigO{}{
					\outNumRounds \sqrt{\innNumRounds \mabDim}
				}
				\end{math}
				\\
				\midrule
				\Acrshort{inf} + prior &
				any & any & any &
				\begin{math}
				\qquad\quad
				\bigO{}{
					\outNumRounds \sqrt{\innNumRounds} \sqrt[4]{\mabDim}
					\sqrt{\EntropyTsallisHalfSc{\knownPrior}}
				}
				\qquad\quad
				\refstepcounter{equation}(\theequation)
				\label{eq:bound-know-prior-tab}
				\end{math}
				\\
				\Acrshort{inf} + prior%
				\footnotemark  
				&
				any & any & 
				$ \bigO{}{1} $ 
				&
				\begin{math}
				\bigO{}{
					\outNumRounds \sqrt{\innNumRounds} \sqrt[4]{\mabDim}
				}
				\end{math}
				\\
				\midrule
				%
				\Acrshort{meta-inf} &
				$ \bigOmega{}{\innNumRounds} $ &
				$\bigOmega{}{\frac{\mabDim {\pth{\log \mabDim}}^{\nicefrac{7}{3}}}{\gapMin^{\nicefrac{10}{3}}} }$ & 
				any &
				\begin{math}
				\bigO{}{
					\outNumRounds \sqrt{\innNumRounds \mabDim}
					\pth{
						\frac{
							\sqrt{\EntropyTsallisHalfSc{\initComparator}}
						}{
							\sqrt[4]{\mabDim}
						}
						+
						\frac{
							\mabDim ^ {\nicefrac{1}{14}}
						}{
							\gapMin ^ {\nicefrac{4}{7}}
							\innNumRounds ^ {\nicefrac{1}{14}}
						}
						+
						\frac{1}{\sqrt[7]{\outNumRounds}}
					}
				}
				\end{math}

				\\
				\Acrshort{meta-inf} &
				$ \bigOmega{}{\innNumRounds} $ &
				$\bigOmega{}{\frac{\mabDim^{\nicefrac{9}{2}}}{\gapMin^{8} }}$ & 
				any &
				\begin{math}
				\bigO{}{
					\outNumRounds \sqrt{\innNumRounds} \sqrt[4]{\mabDim}
					\pth{
						\sqrt{\EntropyTsallisHalfSc{\initComparator}}
						+ 1
					}
				}
				\end{math}
				\\
				\Acrshort{meta-inf} &
				$ \bigOmega{}{\innNumRounds} $ &
				$\bigOmega{}{\frac{\mabDim^{\nicefrac{9}{2}}}{\gapMin^{8} }}$ & 
				$ \bigO{}{1} $ &
				\begin{math}
				\bigO{}{
					\outNumRounds \sqrt{\innNumRounds}
					\sqrt[4]{\mabDim}
				}
				\end{math}
				\\
				\bottomrule
			\end{tabular}
		\end{small}
	\end{center}
	\vskip -0.1in
\end{table*}

Compared to the bound of \gls{inf} in the known prior setting (as shown in
 \cref{eq:bound-know-prior-tab}, derived in \cref{sec:q-choice}), there are additional terms added.
The term
\begin{math}
\frac{
	\mabDim ^ {\nicefrac{1}{14}}
}{
	\gapMin ^ {\nicefrac{4}{7}}
	\innNumRounds ^ {\nicefrac{1}{14}}
}
\end{math}
originates from $ \boundInitEstimationCost $ and $ \boundExplCost $, being the optimal point in the trade-off between the cost of exploration and the cost of incorrect best arm identification.
The two last terms originate from $\boundLrRegret$ and $\boundInitRegret$ respectively.

In \cref{tab:proposed-bound} we show a few special cases of assumptions on the problem parameters and the resulting simplified bound.
The regret bound of resetting \gls{inf} before each episode is presented as a baseline.
Note that since \gls{inf} meets the single-episode lower bound, this is also the best bound that can be achieved by any online algorithm executed on each episode individually without meta-learning.

Another baseline that we considered is the Exp3.S algorithm \autocite{auer2002nonstochastic}.
In a game of $ T' $ rounds, its regret against a comparator sequence with up to $S'$ switches is
\begin{math}
\bigO{}{\sqrt{
		S' T' \mabDim
		\log \pth{T' \mabDim}
}}
\end{math}.
In our setting $ T' = \outNumRounds \innNumRounds$, $S' = \outNumRounds $,
resulting in the bound 
\begin{math}
\bigO{}{
	\outNumRounds 
	\sqrt{
		\innNumRounds \mabDim
		\log \pth{\outNumRounds \innNumRounds \mabDim}
}}
\end{math}, which is worse than \gls{inf}.
It is not surprising, since Exp3.S is designed for a different setting, in which the task similarity is not quantified and the transitions between the episodes are not revealed to the player.

The known prior rows in the table represent the ideal case that the distribution of the best arm per-episode is known in advance, and does not need to be meta-learned.
Its bound is multiplied by $\outNumRounds$, to fit the number of rounds of meta-learning.
It can be seen, that for sufficiently large $\innNumRounds$ and $\outNumRounds$,
the regret bound of \acrshort{meta-inf} behaves asymptotically like the known prior case.

\footnotetext{This represents the few-good-arms case (\cref{def:few-good-arms}) with constant $\numGoodArms \ll \mabDim$ and $\badArmsWeight \leq \bigO{}{\nicefrac{1}{\mabDim}}$}

\section{Summary}
\label{sec:summary}

We presented a framework for meta-learning in an episodic adversarial bandit setting, where an adversary selects the best arm in each episode, and a player, lacking any prior knowledge, aims to minimize its regret w.r.t.~the unknown best arms. We introduced an online-within-online framework, comprised of an inner algorithm that operates within episodes based on parameters it obtains from the outer algorithm. The outer algorithm  works across episodes, based on a loss function constructed within each episode. We present problem-dependent regret bounds for the algorithm, and demonstrate its merits compared to a non-meta-learning approach. As far as we are aware this is the first regret bound for such a setting, extending previous results from the full information case. Open questions include adaptively estimating optimal parameter values for the Tsallis regularizer, generalizing the gap structure of the per-episode arm losses, and developing regret bounds that hold simultaneously for both the adversarial and stochastic settings.
Further research may consider other partial information settings, such as contextual bandits, or bandit convex optimization.

\pdfbookmark[1]{References}{references}
\bibliography{meta-mab.bib}
\bibliographystyle{plainnat}

\newpage
\appendix

\renewcommand{\proofname}{Proof}

\section{Notations}


\begin{table}[H]
	\caption{List of notations}
	\label{tab:notations}
	\vskip 0.15in
	\begin{center}
		\begin{small}
				\begin{tabular}{cp{0.6\linewidth}c}
					\toprule
					Symbol & Description & Definition \\
					\midrule
					$ \onesVector $ & Vector of ones & \\
					$ \basisVector{i} $ & Vector with component $i$ equal to $1$ and all other components equal to $0$ & \\
					$ \outNumRounds $ & Number of episodes & \\
					$ \innNumRounds $ & Number of rounds per episode & \\
					$ \mabDim $ & Number of arms & \\
					$ \algInner $ & The inner algorithm & \cref{alg:inner} \\ 
					$ \algOuterLr $ & The outer algorithm for learning rate & \cref{alg:outer-lr} \\ 
					$ \algOuterInit $ & The outer algorithm for initialization point & \cref{alg:outer-init} \\ 
					$ \lossFunc_{\outCurRound, \innCurRound, i} $ & Loss of arm $i$ at round $\innCurRound$ of episode $\outCurRound$ & \\
					$ \lossEst_{\outCurRound, \innCurRound, i} $ & Estimated loss of arm $i$ at round $\innCurRound$ of episode $\outCurRound$ & \cref{algline:inner:unbiased-loss-estimate} of \cref{alg:inner} \\
					$ \bestArmTrue_\outCurRound $ & Best arm on episode $\outCurRound$ & \cref{def:true-best-arm} \\
					$ \bestArmEst_\outCurRound $ & Estimated best arm on episode $\outCurRound$ & \cref{algline:inner:best-arm-estimate} of \cref{alg:inner} \\
					$ \initComparator $ & Empirical distribution of best arms & \cref{def:initComparator} \\
					$ \EntropyTsallisScaled{\q}{x} $ & Tsallis entropy with parameter $ \q $ of distribution $ x $ & \cref{def:EntropyTsallisScaled} \\
					$ \betaDivergence{\q}{x}{y} $ & Beta-divergence with parameter $ \q $ between distributions $ x $ , $ y $ & \cref{def:betaDivergence} \\
					$ \gapMin $ & Minimal gap across all arms and episodes & \cref{thm:assum:minimal-gap} \\
					$ \knownPrior $ & Distribution of best arms in the known prior setting & \cref{ssec:known-prior} \\
					$ \numGoodArms $ & Number of ``good'' arms in the few-good-arms case & \cref{def:few-good-arms} \\
					$ \badArmsWeight $ & Total probability of ``bad'' arms in the few-good-arms case & \cref{def:few-good-arms} \\
					$ \innComparator $ & Comparator for regret of the inner algorithm & \\
					$ \innLearningRate $ & Learning rate of the inner algorithm (hyper-parameter) & \cref{alg:inner} \\
					$ \innInitial $ & Initialization point of the inner algorithm (hyper-parameter) & \cref{alg:inner} \\
					$ \lrEpsilon $ & Regularization parameter for learning rate meta-learning & \cref{alg:outer-lr} \\
					$ \probLowerLimit $ & Lower limit on arm sampling probability & \cref{def:truncated-simplex} \\
					$ \truncatedSimplex{\probLowerLimit} $ & Truncated simplex that ensures arm sampling probability at least $\probLowerLimit$ & \cref{def:truncated-simplex} \\
					$ \basisVector{i}^\probLowerLimit $ & Mixture between one-hot and uniform distributions with weight $\probLowerLimit \mabDim$ & \cref{eq:def:mix-one-hot-uniform-problemScale} \\
					$ \problemScale $ & The problem scale factor & \cref{eq:def:mix-one-hot-uniform-problemScale} \\
					$ \wrongArmProb $ & Upper bound on probability of wrong identification of best arm & \cref{thm:identification-probability} \\
					$\boundExplCost$ & The cost associated with guaranteed exploration & \cref{thm:observable-inner-regret} \\
					$\boundLrRegret$ & The upper bound on the regret of meta-learning the learning-rate & \cref{thm:lr-meta-regret} \\
					$\boundInitRegret$ & The upper bound on the regret of meta-learning the initialization point & \cref{thm:init-meta-regret} \\
					$\boundInitEstimationCost$ & The upper bound on the difference between the entropy of empirical distribution of the true best arms, and the expected estimated entropy & \cref{thm:init-metalearning-bih-loss} \\
					\bottomrule
				\end{tabular}
		\end{small}
	\end{center}
	\vskip -0.1in
\end{table}


\printglossary[type=\acronymtype]



\section{Regret bound of INF}
\label{sec:simple-tsallis}
{
\newcommand*{\regul}{\ensuremath{\varphi}}
\newcommand*{\comparator}{\innComparator}

\begin{algorithm}[tb]
	\caption{\acrshort{inf} algorithm.}
	\label{alg:simple-tsallis}
	\begin{algorithmic}[1]
		\STATE \textbf{Parameters:}
		\begin{math}
		\innInitial,
		\innLearningRate,
		\q
		\end{math}.
		\STATE \textbf{Initialize.}
		\begin{math}
		\decision_{1} = \innInitial
		\end{math}.
		\FOR{$\innCurRound = 1, \ldots, \innNumRounds$}
		\STATE \textbf{Choose} (sample) arm $ \mabDecision_{\innCurRound} \sim \decision_{\innCurRound} $.
		\STATE \textbf{Observe} loss $ \lossFunc_{\innCurRound, \mabDecision_{\innCurRound}} $.
		\STATE Construct unbiased loss estimate:
		\begin{equation*}
		\lossEst_{\innCurRound, i} = 
		\begin{cases}
		\frac{\lossFunc_{\innCurRound, i}}{\decision_{\innCurRound, i}}
		,&
		i = \mabDecision_{\innCurRound}
		\\
		0
		,&
		\otherwise
		\end{cases}
		\end{equation*}
		\STATE Update decision point:
		\begin{equation*}
		\decision_{\innCurRound + 1}
		=
		\argmin_{\substack{
				\decision
				\\ \subalign{
					\text{s.t.} \;&
					\decision_i \geq 0, \; \forall i
					\\&
					\sum_{i=1}^{\mabDim} \decision_i = 1
				}
		}}
		\innLearningRate
		\innerProd{\lossEst_{\innCurRound}}{\decision}
		+
		\betaDivergence{\q}{\decision}{\decision_{\innCurRound}}
		\end{equation*}
		\ENDFOR
	\end{algorithmic}
\end{algorithm}

We follow \autocite[Section~10.1.2]{orabona2019modern} to prove the regret bound of \gls{inf} algorithm on a single episode of \gls{mab}.
They analyze \gls{inf} as \gls{omd} with negative Tsallis entropy as regularizer, that operates on unbiased estimate of the loss.
Other ways to analyze it exist as well \autocite{audibert2009minimax,audibert2011minimax,abernethy2015fighting,zimmert2021tsallis}.
The differences between the analysis presented here and \autocite{orabona2019modern} are:
\begin{itemize}
	\item We use Tsallis entropy as defined in \cref{def:EntropyTsallisScaled}, which is scaled by the factor $\frac{1}{\q}$ compared to the common definition.
		This does not affect the regret bounds, because it can be compensated by the learning rate.
	\item The Tsallis parameter $\q$ is kept as a free parameter, instead of being set to $\nicefrac{1}{2}$ in the last step.
	\item The initialization point $\innInitial$ is kept as a free parameter, instead of being set to uniform ($\frac{1}{\mabDim} \onesVector $).
	\item The comparator $\innComparator $ is kept as a free parameter, instead of assuming one-hot ($\basisVector{\bestArmTrue}$).
	\item The learning rate $\innLearningRate$ is kept as a free parameter, instead of being optimized.
\end{itemize}

\begin{lemma}[\cref{thm:inf-bound}]
	\label{thm:inf-bound-appendix}
	Algorithm \gls{inf} (\cref{alg:simple-tsallis}) with Tsallis parameter $\q$, initialization point $\innInitial$, and learning rate $\innLearningRate$, guarantees the following expected regret bound on a single episode of oblivious adversarial \gls{mab} against a static comparator $ \innComparator $:
	\begin{equation*}\label{key}
	\begin{aligned}
	\E \regretInner \pth{\innComparator}
	\;&{=}\;
	\E \sum_{\innCurRound=1}^{\innNumRounds} \lossFunc_{\innCurRound, \mabDecision_{\innCurRound}}
	-
	\sum_{\innCurRound=1}^{\innNumRounds} \sum_{i=1}^{\mabDim} \lossFunc_{\innCurRound, i} \innComparator_i
	\;{\leq}\;
	\frac{1}{\innLearningRate}
	\betaDivergence{\q}{\innComparator}{\innInitial}
	+
	\frac{\innLearningRate}{2}
	\innNumRounds
	\mabDim^{\q}~.
	\end{aligned}
	\end{equation*}
\end{lemma}
\begin{proof}
	
The regularizer is negative scaled Tsallis entropy:
\begin{equation}\label{eq:def:regul}
\begin{aligned}
\regul_\q \pth{\decision}
\triangleq
- \EntropyTsallisScaled{\q}{\decision}
=
\frac{1}{\q \pth{1 - \q}}
\pth{
	1 - \sum_{i=1}^{\mabDim} \decision_i^\q
}
\end{aligned}
\end{equation}

Its first and second derivatives are (Hessian is diagonal and we only state the values on the diagonal):
\begin{equation}\label{eq:tsallis-derivative-1}
\begin{aligned}
\PartDiv{\regul_\q \pth{\decision}}{\decision_i}
=
-
\frac{1}{1 - \q}
\decision_i^{\q - 1}
\end{aligned}
\end{equation}

\begin{equation}\label{eq:tsallis-derivative-2}
\begin{aligned}
\PartDiv*{2}{\regul_\q \pth{\decision}}{\decision_i}
=
\decision_i^{\q - 2}
\end{aligned}
\end{equation}

By \autocite[Lemma~6.14]{orabona2019modern} the regret bound of OMD with this regularizer is:
\begin{equation}\label{eq:OMD_Tsallis}
\begin{aligned}
\sum_{\innCurRound=1}^{\innNumRounds}
\innerProd{
	\lossEst_\innCurRound
}{\pth{
		\decision_\innCurRound - \comparator	
}}
&{\leq}
\frac{1}{\innLearningRate}
\betaDivergence{\q}{\comparator}{\innInitial}
+
\frac{\innLearningRate}{2}
\sum_{\innCurRound=1}^{\innNumRounds}
\sum_{i=1}^{\mabDim}
\lossEst_{\innCurRound, i}^2
\decision_{\innCurRound, i}^{2-\q}
\end{aligned}
\end{equation}

Similar to \autocite[Theorem~10.3]{orabona2019modern} but with parametric $\q$, we need to calculate:
\begin{equation}\label{eq:inf-d-q-term}
\begin{aligned}
\E \brk{
	\sum_{i=1}^{\mabDim}
	\lossEst_{\innCurRound, i}^2
	\decision_{\innCurRound, i}^{2-\q}
}
&{=}
\E \brk{
	\sum_{i=1}^{\mabDim}
	\frac{\lossFunc_{\innCurRound, i}^2}{\decision_{\innCurRound, i}^2}
	\decision_{\innCurRound, i}^{2-\q}
	\decision_{\innCurRound, i}
}
\\&{=}
\E \brk{
	\sum_{i=1}^{\mabDim}
	{\lossFunc_{\innCurRound, i}^2}
	\decision_{\innCurRound, i}^{1-\q}
}
\\\overset{\text{\Holder{}}}&{\leq}
\E \brk{
	{\pth{
			\sum_{i=1}^{\mabDim}
			{\pth{\lossFunc_{\innCurRound, i}^2}}^{\frac{1}{\q}}
	}}^{\q}
	{\pth{
			\sum_{i=1}^{\mabDim}
			{\pth{\decision_{\innCurRound, i}^{1-\q}}}^{\frac{1}{1-\q}}
	}}^{1-\q}
}
\\&{\leq}
\E \brk{
	{\pth{
			\sum_{i=1}^{\mabDim}
			1^{\frac{2}{\q}}
	}}^{\q}
	{\pth{
			\sum_{i=1}^{\mabDim}
			\decision_{\innCurRound, i}
	}}^{1-\q}
}
\\&{=}
\E \brk{
	{\pth{
			\mabDim
	}}^{\q}
	{\pth{
			1
	}}^{1-\q}
}
\\&{=}
%
\mabDim^{\q}
\end{aligned}
\end{equation}

The result follows.

\end{proof}

}

\section{INF with known prior}
\label{sec:known-prior-tsallis}

\begin{lemma}[\cref{eq:bound-know-prior}]
	Let the adversary choose the best arm according to the prior distribution $\knownPrior$, known to the player:
	\begin{math}
	\Prob\pth{\bestArmTrue = j}
	=
	\knownPrior_j
	\end{math}.
	Then the \gls{inf} algorithm (\cref{alg:simple-tsallis}) with Tsallis parameter $\q$ and optimal $\pth{\innInitial, \innLearningRate}$, guarantees the following expected regret bound on a single episode of oblivious adversarial \gls{mab}:
	\begin{equation*}\label{key}
	\begin{aligned}
	\E \regretKnownPrior \pth{\knownPrior}
	&{\leq}
	%
	\sqrt{
		2
		\EntropyTsallisScaled{\q}{\knownPrior}
		\innNumRounds
		\mabDim^{\q}
	}~.
	\end{aligned}
	\end{equation*}
\end{lemma}

\begin{proof}
We set the initialization point to be equal to the prior:
\begin{math}
\innInitial = \knownPrior
\end{math}.

Then we apply \cref{thm:inf-bound-appendix} and calculate expectation \wrt{} the prior.

\begin{equation}\label{key}
\begin{aligned}
	\E \regretKnownPrior \pth{\knownPrior}
	&{\leq}
	\Excpt{\bestArmTrue}{\knownPrior}\brk{
		\frac{1}{\innLearningRate}
		\sum_{j=1}^{\mabDim}
		\betaDivergence{\q}{\basisVector{\bestArmTrue}}{\knownPrior}
		+
		\frac{\innLearningRate}{2}
		\innNumRounds
		\mabDim^{\q}
		}
	\\&{=}
	\frac{1}{\innLearningRate}
	\sum_{j=1}^{\mabDim}
	\knownPrior_j
	\betaDivergence{\q}{\basisVector{j}}{\knownPrior}
	+
	\frac{\innLearningRate}{2}
	\innNumRounds
	\mabDim^{\q}
	\\&{=}
	\frac{1}{\innLearningRate}
	\sum_{j=1}^{\mabDim}
	\knownPrior_j
	\pth{
		\frac{1}{\q \pth{1 - \q}}
		\sum_{i=1}^{\mabDim}
		\pth{
			\q \frac{\pth{\basisVector{j}}_i}{\knownPrior_i^{1 - \q}}
			-
			\pth{\basisVector{j}}_i^\q
		}
		+
		\frac{1}{\q}
		\sum_{i=1}^{\mabDim}
		\knownPrior_i^\q
	}
	+
	\frac{\innLearningRate}{2}
	\innNumRounds
	\mabDim^{\q}
	\\&{=}
	\frac{1}{\innLearningRate}
	\pth{
		\frac{1}{\pth{1 - \q}}
		\sum_{i=1}^{\mabDim}
		\frac{\knownPrior_i}{\knownPrior_i^{1 - \q}}
		-
		\frac{1}{\q \pth{1 - \q}}
		\sum_{i=1}^{\mabDim}
		\knownPrior_i
		+
		\frac{1}{\q}
		\sum_{i=1}^{\mabDim}
		\knownPrior_i^\q
		\sum_{j=1}^{\mabDim}
		\knownPrior_j
	}
	+
	\frac{\innLearningRate}{2}
	\innNumRounds
	\mabDim^{\q}
	\\&{=}
	\frac{1}{\innLearningRate}
	\pth{
		\frac{
			\sum_{i=1}^{\mabDim}
			\knownPrior_i^{\q}
			- 1
		}{\q \pth{1 - \q}}
	}
	+
	\frac{\innLearningRate}{2}
	\innNumRounds
	\mabDim^{\q}
	\\&{=}
	\frac{1}{\innLearningRate}
	\EntropyTsallisScaled{\q}{\knownPrior}
	+
	\frac{\innLearningRate}{2}
	\innNumRounds
	\mabDim^{\q}
	\\&{=}
	\sqrt{
		2
		\EntropyTsallisScaled{\q}{\knownPrior}
		\innNumRounds
		\mabDim^{\q}
	}
\end{aligned}
\end{equation}

In the last line we set
\begin{equation}\label{key}
\begin{aligned}
		\innLearningRate =
\sqrt{\frac{
		2 \EntropyTsallisScaled{\q}{\knownPrior}
	}{
		\innNumRounds\mabDim^{\q}
}}
\end{aligned}
\end{equation}

\end{proof}

\section{Choice of q motivation}
\label{sec:q-choice}

The known-prior setting (see \cref{ssec:known-prior}) will be used to guide the choice of $\q$.
In this case $ \innComparator = \basisVector{\bestArmTrue}$, $\bestArmTrue \sim \knownPrior $.
The regret can then be bounded by setting $ \innInitial = \knownPrior $, and calculating expectation \wrt{} both the randomization of the learner and the choice of $\bestArmTrue$ (see \cref{sec:known-prior-tsallis}).
We abuse the notation to indicate that the expected regret depends on the prior distribution $\knownPrior$, 
\begin{equation}\label{eq:bound-know-prior}
\begin{aligned}
\E \regretKnownPrior \pth{\knownPrior}
&\leq
\frac{1}{\innLearningRate}
\EntropyTsallisScaled{\q}{\knownPrior}
+
\frac{\innLearningRate}{2}
\innNumRounds
\mabDim^{\q}
%
\overset{\textnormal{set } \innLearningRate}{=}
\sqrt{
	2
	\EntropyTsallisScaled{\q}{\knownPrior}
	\innNumRounds
	\mabDim^{\q}
}~.
\end{aligned}
\end{equation}
If the prior is uniform (worst case setting), and $ \mabDim \to \infty $, the optimal value of the Tsallis parameter is $\q = \nicefrac{1}{2}$. 
For comparison, in this case
$ \q = 1 $ (Exp3) gives a bound
\begin{math}
\bigO{}{ \sqrt{ \innNumRounds \mabDim \log{\mabDim} } }
\end{math}%
.
On the other hand,
$\q = \nicefrac{1}{2}$ guarantees 
\begin{math}
\bigO{}{ \sqrt{ \innNumRounds \mabDim } } 
\end{math}
and meets the lower bound.

In the few-good-arms scenario (\cref{def:few-good-arms}) with constant $\numGoodArms \ll \mabDim$ and $\badArmsWeight \leq \bigO{}{\nicefrac{1}{\mabDim}}$,
the bounds for $ \q = 1 $ and $\q = \nicefrac{1}{2}$ become 
\begin{math}
\bigO{}{ \sqrt{ \innNumRounds \mabDim } }
\end{math},
and
\begin{math}
\bigO{}{ \sqrt{ \innNumRounds } \sqrt[4]{\mabDim} }
\end{math}
respectively.
It can be shown that $\q = \nicefrac{1}{2}$ is approximately optimal for the case $\badArmsWeight = \bigTheta{}{\nicefrac{1}{\mabDim}}$:
\begin{equation}\label{key}
\begin{aligned}
\EntropyTsallisScaled{\q}{\knownPrior}
\mabDim^{\q}
&=
\frac{1}{\q \pth{1 - \q}}
\pth{
	\numGoodArms^{\pth{1 - \q}}
	\pth{1 - \badArmsWeight}^{\q}
	+
	\pth{ \mabDim - \numGoodArms}^{1 - \q}
	\badArmsWeight^{\q}
	- 1
}
\mabDim^{\q}
\\&\approx
\frac{1}{\q \pth{1 - \q}}
\pth{
	\mabDim^{\q}
	+
	\mabDim^{1 - \q}
}
\end{aligned}
\end{equation}

For smaller $\badArmsWeight$ it should be possible to achieve better dependence of the bound on $\mabDim$.
For example, for $\badArmsWeight = 0$, the optimal bound is $ \bigO{}{ \sqrt{ \innNumRounds \numGoodArms } } $, since this is a degenerate case where bad arms can be discarded a priori.
However, we should remember that the known prior setting is not practical and is only used as guidance.
The proposed meta-learning architecture incorporates best-arm identification, which requires exploration, and its cost depends on $d$. Therefore, even in the optimal case the regret bound will still grow with $d$, though perhaps slower than $ \sqrt[4]{d} $.

%


In this paper we use \gls{inf} with $ \q = \nicefrac{1}{2} $ in the inner algorithm $ \algInner $.
This choice is based on the fact that it is optimal in the worst-case scenario, and performs much better compared to using Shannon Entropy ($\q=1$) in the few-good-arms case.
We leave other choices of $\q$ and meta-learning the optimal $\q$ for future work.
%


\section{Best arm identification analysis (proof of Lemma~\ref{thm:identification-probability})}
\label{sec:best-arm-identification-full}

{
\newcommand*{\gap}{\ensuremath{\Delta}}
\newcommand*{\decisionRand}{\mabDecision}

\begin{definition}
The per-arm loss gap is
\begin{equation}\label{key}
\begin{aligned}
\gap_{i} = \frac{1}{\innNumRounds} \abs{
	\sum_{\innCurRound = 1}^{\innNumRounds} \lossFunc_{\innCurRound, i}
	-
	\min_{j \neq i} \sum_{\innCurRound = 1}^{\innNumRounds} \lossFunc_{\innCurRound, j}
}
,\quad
\forall i \in \braces{1,\ldots,\mabDim}
.
\end{aligned}
\end{equation}
\end{definition}

Note that this allows an alternative definition of the overall gap (\cref{thm:assum:minimal-gap}) as
\begin{equation}\label{key}
\begin{aligned}
\gapMin = \min_{i \in \braces{1,\ldots,\mabDim}} \gap_{i}
\;.
\end{aligned}
\end{equation}

%

\begin{lemma}[\cref{thm:identification-probability}]
	\label{thm:identification-probability-appendix}
	Let $ \braces{\decision_{\innCurRound}}_{\innCurRound=1}^{\innNumRounds} $ be decision points of an inner algorithm, such that
	$ \decision_{\innCurRound} \in \truncatedSimplex{\probLowerLimit},\; \forall \innCurRound$.
	Let the true losses $ {\lossFunc_{\innCurRound}} $ satisfy \cref{thm:assum:minimal-gap}.
	Let $ \bestArmTrue $ be the arm with the lowest cumulative loss.
	Let $ \lossEst_{\innCurRound} $ be the unbiased estimate of the losses, used to compute the estimated best arm  $\bestArmEst$ (see \cref{algline:inner:unbiased-loss-estimate,algline:inner:best-arm-estimate} of \cref{alg:inner}).
	%
	Then the probability of correct estimation is bounded from below by
	\begin{equation*}
	\begin{aligned}
	\Prob \pth{ \bestArmEst_{\outCurRound} = {\bestArmTrue_{\outCurRound}} }
	&\geq
	%
	1 - \anyWrongArmProb
	,
	\end{aligned}
	\end{equation*}
	\begin{equation*}
	\begin{aligned}
	\wrongArmProb
	&{\triangleq}
	\exp{
		- \frac{3}{28}
		\gapMin^2 \probLowerLimit \innNumRounds 
	}
	.
	\end{aligned}
	\end{equation*}
\end{lemma}
\begin{proof}
Following \autocite{abbasi2018best} we will derive a bound on probability of best arm identification in the adversarial MAB.
%

We use the standard unbiased loss estimation
\begin{math}
\lossEst_{\innCurRound, i}
\triangleq
\frac{\lossFunc_{\innCurRound, i}}{\decision_{\innCurRound, i}} \Indic{\decisionRand_\innCurRound = i}
\end{math}
to declare the observed best arm. Note that a different estimation might have been used by the algorithm that decided on $\decision_{\innCurRound, i}$.

Now, following \autocite[Appendix~B,F]{abbasi2018best}, we have:
\begin{equation}\label{key}
\begin{aligned}
0
\leq
\lossEst_{\innCurRound, i}
&\leq
\frac{1}{\probLowerLimit} \lossFunc_{\innCurRound, i}
\\
- \frac{1}{\probLowerLimit}
\leq
-1
\leq
- \lossFunc_{\innCurRound, i}
\leq
\lossEst_{\innCurRound, i} - \lossFunc_{\innCurRound, i}
&\leq
\pth {\frac{1}{\probLowerLimit} - 1} \lossFunc_{\innCurRound, i}
\leq
\frac{1}{\probLowerLimit}
\\
\abs{ \lossEst_{\innCurRound, i} - \lossFunc_{\innCurRound, i} }
&\leq
\frac{1}{\probLowerLimit}
\end{aligned}
\end{equation}

\begin{equation}\label{eq:lossEstVariance}
\begin{aligned}
\Var{\lossEst_{\innCurRound, i} - \lossFunc_{\innCurRound, i}}
&=
\Var{\lossEst_{\innCurRound, i}}
\\\overset{\text{scaled Bernoulli}}&{=}
\decision_{\innCurRound, i} \pth{1 - \decision_{\innCurRound, i}}
\pth{ \frac{1}{\decision_{\innCurRound, i}} \lossFunc_{\innCurRound, i}}^2
\\&\leq
\frac{1}{\decision_{\innCurRound, i}} \lossFunc_{\innCurRound, i}^2
\\&\leq
\frac{1}{\probLowerLimit}
\end{aligned}
\end{equation}

We now prepare to use the  martingale concentration inequality from \autocite[Corollary~2.1]{fan2012hoeffding}.
For this we define a bounded martingale difference sequence
\begin{equation}\label{key}
\begin{aligned}
- \probLowerLimit
\leq
\probLowerLimit \pth{ \lossEst_{\innCurRound, i} - \lossFunc_{\innCurRound, i} }
&\leq
1
\end{aligned}
\end{equation}

And then
\begin{equation}\label{eq:boundProbErrorA}
\begin{aligned}
\Prob \pth{ \sum_{\innCurRound = 1}^{\innNumRounds} \pth{\lossEst_{\innCurRound, i} - \lossFunc_{\innCurRound, i}} \geq \frac{\innNumRounds \gap_{i}}{2} }
&=
\Prob \pth{ \sum_{\innCurRound = 1}^{\innNumRounds} \probLowerLimit \pth{\lossEst_{\innCurRound, i} - \lossFunc_{\innCurRound, i}} \geq \frac{\probLowerLimit \innNumRounds \gap_{i}}{2} }
\\&\leq
\Prob \pth{ \max_{\tau \in \brk{\innNumRounds}} \sum_{\innCurRound = 1}^{\tau} \probLowerLimit \pth{\lossEst_{\innCurRound, i} - \lossFunc_{\innCurRound, i}} \geq \frac{\probLowerLimit \innNumRounds \gap_{i}}{2} }
\\&\leq
\exp{ - \frac{ 2 \pth{ \frac{\probLowerLimit \innNumRounds \gap_{i}}{2}}^2 }{\min\braces{
			\innNumRounds \pth{1 + \probLowerLimit}^2,
			4 \pth{ \probLowerLimit \innNumRounds + \frac{1}{3} \frac{\probLowerLimit \innNumRounds \gap_{i}}{2} }
}}}
\\&\leq
\exp{ - \frac{ 2 \pth{ \frac{\probLowerLimit \innNumRounds \gap_{i}}{2}}^2 }{
		4 \pth{ \probLowerLimit \innNumRounds + \frac{1}{3} \frac{\probLowerLimit \innNumRounds \gap_{i}}{2} }
}}
\\&=
\exp{ - \frac{ 3 \probLowerLimit \innNumRounds \gap_{i}^2 }{
		4 \pth{ 6 + { \gap_{i}} }
}}
\\\overset{\gap_{i} \leq 1}&{\leq}
\exp{ - \frac{ 3 \probLowerLimit \innNumRounds \gap_{i}^2 }{
		28
}}
\end{aligned}
\end{equation}

We can also define a symmetric sequence:
\begin{equation}\label{key}
\begin{aligned}
- \frac{1}{\probLowerLimit}
\leq
\pth{ \lossFunc_{\innCurRound, i} - \lossEst_{\innCurRound, i}}
&\leq
1
\end{aligned}
\end{equation}

And then
\begin{equation}\label{eq:boundProbErrorB}
\begin{aligned}
\Prob \pth{ \sum_{\innCurRound = 1}^{\innNumRounds} \pth{\lossEst_{\innCurRound, i} - \lossFunc_{\innCurRound, i}} \leq -\frac{\innNumRounds \gap_{i}}{2} }
&=
\Prob \pth{ \sum_{\innCurRound = 1}^{\innNumRounds} \pth{ \lossFunc_{\innCurRound, i} - \lossEst_{\innCurRound, i}} \geq \frac{\innNumRounds \gap_{i}}{2} }
%
\\&\leq
\exp{-\frac{
		2 \pth{\frac{\innNumRounds \gap_{i}}{2}}^2
	}{
		4 \pth{ \frac{\innNumRounds}{\probLowerLimit} + \frac{1}{3} \frac{\innNumRounds \gap_{i}}{2} }
}}
\\&=
\exp{-\frac{
		3 \probLowerLimit \innNumRounds \gap_{i}^2
	}{
		4 \pth{ 6 + \gap_{i} \probLowerLimit }
}}
\\&\leq
\exp{-\frac{
		3 \probLowerLimit \innNumRounds \gap_{i}^2
	}{
		28
}}
\end{aligned}
\end{equation}

Overall, we can bound the probability of error, \ie{} declaring an arm different from $ \bestArmTrue $ as the best.
\begin{equation}\label{key}
\begin{aligned}
\Prob \pth{ \bestArmEst \neq \bestArmTrue }
&=
\Prob \pth{ \exists i \neq \bestArmTrue : \sum_{\innCurRound = 1}^{\innNumRounds} \lossEst_{\innCurRound, i} \leq \sum_{\innCurRound = 1}^{\innNumRounds} \lossEst_{\innCurRound, \bestArmTrue} }
\\&\leq
\Prob \pth{
	\sum_{\innCurRound = 1}^{\innNumRounds} \lossEst_{\innCurRound, \bestArmTrue} \geq \sum_{\innCurRound = 1}^{\innNumRounds} \lossFunc_{\innCurRound, \bestArmTrue} + \frac{\innNumRounds \gap_{\bestArmTrue}}{2}
	\quad\text{or}\quad
	\exists i \neq \bestArmTrue : \sum_{\innCurRound = 1}^{\innNumRounds} \lossEst_{\innCurRound, i} \leq \sum_{\innCurRound = 1}^{\innNumRounds} \lossFunc_{\innCurRound, i} - \frac{\innNumRounds \gap_{i}}{2} }
\\&\leq
\Prob \pth{
	\sum_{\innCurRound = 1}^{\innNumRounds} \lossEst_{\innCurRound, \bestArmTrue} \geq \sum_{\innCurRound = 1}^{\innNumRounds} \lossFunc_{\innCurRound, \bestArmTrue} + \frac{\innNumRounds \gap_{\bestArmTrue}}{2}
} + \sum_{i \neq \bestArmTrue} \Prob \pth{
	\sum_{\innCurRound = 1}^{\innNumRounds} \lossEst_{\innCurRound, i} \leq \sum_{\innCurRound = 1}^{\innNumRounds} \lossFunc_{\innCurRound, i} - \frac{\innNumRounds \gap_{i}}{2} }
\\&\leq
\Prob \pth{
	\sum_{\innCurRound = 1}^{\innNumRounds} \pth{ \lossEst_{\innCurRound, \bestArmTrue} - \lossFunc_{\innCurRound, \bestArmTrue} } \geq \frac{\innNumRounds \gap_{\bestArmTrue}}{2}
} + \sum_{i \neq \bestArmTrue} \Prob \pth{
	\sum_{\innCurRound = 1}^{\innNumRounds} \pth{ \lossEst_{\innCurRound, i} - \lossFunc_{\innCurRound, i} } \leq - \frac{\innNumRounds \gap_{i}}{2} }
\\\overset{\cref{eq:boundProbErrorA}}&{\leq}
\exp{ - \frac{ 3 \probLowerLimit \innNumRounds \gap_{\bestArmTrue}^2 }{
		28
}}
+ 
\sum_{i \neq \bestArmTrue} \Prob \pth{
	\sum_{\innCurRound = 1}^{\innNumRounds} \pth{ \lossEst_{\innCurRound, i} - \lossFunc_{\innCurRound, i} } \leq - \frac{\innNumRounds \gap_{i}}{2} }
\\\overset{\cref{eq:boundProbErrorB}}&{\leq}
\exp{ - \frac{ 3 \probLowerLimit \innNumRounds \gap_{\bestArmTrue}^2 }{
		28
}}
+ 
\sum_{i \neq \bestArmTrue}
\exp{ - \frac{ 3 \probLowerLimit \innNumRounds \gap_{i}^2 }{
		28
}}
\\\overset{\gapMin \leq \gap_{i}}&{\leq}
\mabDim \;
\exp{ - \frac{ 3 \probLowerLimit \innNumRounds \gap^2 }{
		28
}}
\end{aligned}
\end{equation}
\end{proof}


Similarly we can bound the probability of a specific arm to be wrongly chosen as the best:
\begin{equation}\label{eq:identification-probability-pair}
\begin{aligned}
\forall i \neq \bestArmTrue : \quad
\Prob \pth{ \bestArmEst = \basisVector{i} }
&=
\Prob \pth{ \sum_{\innCurRound = 1}^{\innNumRounds} \lossEst_{\innCurRound, i} \leq \sum_{\innCurRound = 1}^{\innNumRounds} \lossEst_{\innCurRound, \bestArmTrue}  }
\\&\leq
\Prob \pth{
	\sum_{\innCurRound = 1}^{\innNumRounds} \lossEst_{\innCurRound, \bestArmTrue} \geq \sum_{\innCurRound = 1}^{\innNumRounds} \lossFunc_{\innCurRound, \bestArmTrue} + \frac{\innNumRounds \gap_{\bestArmTrue}}{2}
	\quad\text{or}\quad
	\sum_{\innCurRound = 1}^{\innNumRounds} \lossEst_{\innCurRound, i} \leq \sum_{\innCurRound = 1}^{\innNumRounds} \lossFunc_{\innCurRound, i} - \frac{\innNumRounds \gap_{i}}{2} 
}
\\&\leq
\Prob \pth{
	\sum_{\innCurRound = 1}^{\innNumRounds} \lossEst_{\innCurRound, \bestArmTrue} \geq \sum_{\innCurRound = 1}^{\innNumRounds} \lossFunc_{\innCurRound, \bestArmTrue} + \frac{\innNumRounds \gap_{\bestArmTrue}}{2}
} + \Prob \pth{
	\sum_{\innCurRound = 1}^{\innNumRounds} \lossEst_{\innCurRound, i} \leq \sum_{\innCurRound = 1}^{\innNumRounds} \lossFunc_{\innCurRound, i} - \frac{\innNumRounds \gap_{i}}{2} 
}
\\&\leq
\Prob \pth{
	\sum_{\innCurRound = 1}^{\innNumRounds} \pth{ \lossEst_{\innCurRound, \bestArmTrue} - \lossFunc_{\innCurRound, \bestArmTrue} } \geq \frac{\innNumRounds \gap_{\bestArmTrue}}{2}
} + \Prob \pth{
	\sum_{\innCurRound = 1}^{\innNumRounds} \pth{ \lossEst_{\innCurRound, i} - \lossFunc_{\innCurRound, i} } \leq - \frac{\innNumRounds \gap_{i}}{2} 
}
\\\overset{\cref{eq:boundProbErrorA}}&{\leq}
\exp{ - \frac{ 3 \probLowerLimit \innNumRounds \gap_{\bestArmTrue}^2 }{
		28
}}
+ 
\Prob \pth{
	\sum_{\innCurRound = 1}^{\innNumRounds} \pth{ \lossEst_{\innCurRound, i} - \lossFunc_{\innCurRound, i} } \leq - \frac{\innNumRounds \gap_{i}}{2} 
}
\\\overset{\cref{eq:boundProbErrorB}}&{\leq}
\exp{ - \frac{ 3 \probLowerLimit \innNumRounds \gap_{\bestArmTrue}^2 }{
		28
}}
+ 
\sum_{i \neq \bestArmTrue}
\exp{ - \frac{ 3 \probLowerLimit \innNumRounds \gap_{i}^2 }{
		28
}}
\\\overset{\gapMin \leq \gap_{i}}&{\leq}
2 \;
\exp{ - \frac{ 3 \probLowerLimit \innNumRounds \gap^2 }{
		28
}}
\end{aligned}
\end{equation}

}

\section{INF on truncated simplex (proof of Lemma~\ref{thm:observable-inner-regret})}
\label{sec:tsallis-truncated-simplex}

{

\begin{lemma}[\cref{thm:observable-inner-regret}]
	\label{thm:observable-inner-regret-appendix}
	The expected regret on episode $\outCurRound$ using \acrshort{inf} with $\q = \nicefrac{1}{2}$ and guaranteed exploration (\cref{alg:inner}) is bounded from above as
	\begin{equation*}
	\begin{aligned}
	\E \regretInner \pth{\basisVector{{\bestArmTrue_\outCurRound}}}
	&{\leq}
	%
	\frac{
		\E
		\betaDivergenceHalf{\basisVector{{\bestArmEst_\outCurRound}}^{\probLowerLimit}}{\innInitial_\outCurRound}
	}{
		\innLearningRate_\outCurRound \pth{
			1 -
			\anyWrongArmProb
	}}
	%
	%
	+ 
	\problemScale^2
	\innLearningRate_\outCurRound
	+
	\boundExplCost
	\end{aligned}
	\end{equation*}
	where
	$ \boundExplCost $ is the cost of guaranteed exploration
	\begin{equation*}
	\boundExplCost
	{=}
	\probLowerLimit \innNumRounds
	\pth{\mabDim - 1}
	\end{equation*}
\end{lemma}
\begin{proof}
Inspired by the FKM algorithm analysis in \autocite[Lemma~6.3,~Theorem~6.9]{hazan2019introduction}, we consider the algorithm's decisions and the comparator to belong to truncated simplex $ \truncatedSimplex{\probLowerLimit} $, and then add the regret between this limited comparator and the true one.
The true comparator is $\basisVector{\bestArmTrue}$.
Its projection on to the truncated simplex is
\begin{math}
\basisVector{\bestArmTrue}^{\probLowerLimit}
=
\pth{1 - \probLowerLimit \mabDim} \basisVector{\bestArmTrue}
+
\probLowerLimit \onesVector
\end{math}.
The maximal loss difference between the true and the projected comparator is,
\begin{equation}\label{eq:projected-comparator-cost}
\begin{aligned}
\innerProd{\lossFunc_{1:\innNumRounds}}{\basisVector{\bestArmTrue}^{\probLowerLimit}}
-
\innerProd{\lossFunc_{1:\innNumRounds}}{\basisVector{\bestArmTrue}}
\leq
\probLowerLimit \pth{\mabDim - 1} \innNumRounds
\triangleq
\boundExplCost
\,.
\end{aligned}
\end{equation}

And then
\begin{equation}\label{eq:tsallis_shrinked_set}
\begin{aligned}
\E \brk{ \sum_{\innCurRound = 1}^{\innNumRounds} \pth{ \lossFunc_{\innCurRound, \mabDecision_\innCurRound} - \lossFunc_{\innCurRound, j}} }
&{\leq}
\E \brk{ \sum_{\innCurRound = 1}^{\innNumRounds} { \lossFunc_{\innCurRound, \mabDecision_\innCurRound}}  - \innerProd{\lossFunc_{1:\innNumRounds}}{\basisVector{\bestArmTrue}} }
\\&\leq
\E \brk{ \sum_{\innCurRound = 1}^{\innNumRounds} { \lossFunc_{\innCurRound, \mabDecision_\innCurRound}}  - \innerProd{\lossFunc_{1:\innNumRounds}}{\basisVector{\bestArmTrue}^{\probLowerLimit}} }
+ \probLowerLimit \pth{\mabDim - 1} \innNumRounds
\\&\leq
\E \brk{ \sum_{\innCurRound = 1}^{\innNumRounds} \innerProd{\lossEst_{\innCurRound}}{\decision_\innCurRound}  - \innerProd{\lossEst_{1:\innNumRounds}}{\basisVector{\bestArmTrue}^{\probLowerLimit}} }
+ \probLowerLimit \pth{\mabDim - 1} \innNumRounds
\\\overset{\tcref{eq:OMD_Tsallis}}&\leq
\E \brk{
	\frac{1}{\innLearningRate}
	\betaDivergenceHalf{\basisVector{\bestArmTrue}^{\probLowerLimit}}{\innInitial}
	+
	\frac{\innLearningRate}{2}
	\sum_{\innCurRound=1}^{\innNumRounds}
	\sum_{i=1}^{\mabDim}
	\lossEst_{\innCurRound, i}^2
	\decision_{\innCurRound, i}^{\nicefrac{3}{2}}
}
+ \probLowerLimit \pth{\mabDim - 1} \innNumRounds
\\\overset{\tcref{eq:inf-d-q-term}}&\leq
\frac{1}{\innLearningRate}
\betaDivergenceHalf{\basisVector{\bestArmTrue}^{\probLowerLimit}}{\innInitial}
+ 
\frac{\innLearningRate}{2}
\innNumRounds
\sqrt{\mabDim}
+
\probLowerLimit \pth{\mabDim - 1} \innNumRounds
\\&{=}
\frac{
	\betaDivergenceHalf{\basisVector{{\bestArmTrue}}^{\probLowerLimit}}{\innInitial}
}{
	\innLearningRate
}
+ 
\problemScale^2
\innLearningRate
+
\boundExplCost
\end{aligned}
\end{equation}

Now we need to bound 
\begin{math}
\betaDivergenceHalf{\basisVector{\bestArmTrue}^{\probLowerLimit}}{\innInitial}
\end{math}
in an observable way.
\begin{equation}\label{eq:episodeRegretUsingOBIHTsallis}
\begin{aligned}
\betaDivergenceHalf{\basisVector{\bestArmTrue_{\outCurRound}}^{\probLowerLimit}}{\innInitial_\outCurRound}
&=
\sum_{i=1}^{\mabDim} \Indic{i = \bestArmTrue_{\outCurRound}} \betaDivergenceHalf{\basisVector{i}^\probLowerLimit}{\innInitial_\outCurRound}
\\&\leq
\sum_{i=1}^{\mabDim}
\frac{\Prob\pth{\hat{j}_{\outCurRound} = i}}{\Prob\pth{\hat{j}_{\outCurRound} = j^*_{\outCurRound}}}
\betaDivergenceHalf{\basisVector{i}^\probLowerLimit}{\innInitial_\outCurRound}
\\\overset{\tcref{thm:identification-probability}}&\leq
\frac{1}{
	1
	-
	\mabDim
	\wrongArmProb
}
\sum_{i=1}^{\mabDim} 
\Prob\pth{\hat{j}_{\outCurRound} = i}
\betaDivergenceHalf{\basisVector{i}^\probLowerLimit}{\innInitial_\outCurRound}
\\&{=}
\frac{1}{
	1
	-
	\mabDim
	\wrongArmProb
}
\E \betaDivergenceHalf{\basisVector{\hat{j}_{\outCurRound}}^\probLowerLimit}{\innInitial_\outCurRound}
\end{aligned}
\end{equation}

The regret is then bounded in two steps:
\begin{equation*}
\begin{aligned}
\E \regretInner \pth{\basisVector{{\bestArmTrue_\outCurRound}}}
\overset{\tcref{eq:tsallis_shrinked_set}}&{\leq}
\frac{
	\betaDivergenceHalf{\basisVector{{\bestArmTrue_\outCurRound}}^{\probLowerLimit}}{\innInitial_\outCurRound}
}{
	\innLearningRate_\outCurRound
}
+ 
\problemScale^2
\innLearningRate_\outCurRound
+
\boundExplCost
%
\\\overset{\tcref{eq:episodeRegretUsingOBIHTsallis}}&{\leq}
\frac{
	\E
	\betaDivergenceHalf{\basisVector{{\bestArmEst_\outCurRound}}^{\probLowerLimit}}{\innInitial_\outCurRound}
}{\innLearningRate_\outCurRound \pth{
		1 -
		\anyWrongArmProb
}}
%
%
+
\problemScale^2
\innLearningRate_\outCurRound
+
\boundExplCost
\end{aligned}
\end{equation*}

\end{proof}

}

\section{Regret of learning-rate meta-learning (proof of Lemma~\ref{thm:lr-meta-regret})}
\label{sec:lr-metalearning}

{
	
\newcommand*{\regret}{\ensuremath{\mathrm{{Reg}}}}
\newcommand*{\lrB}{\ensuremath{B}}
\newcommand*{\decisionBest}{\ensuremath{x^*}}
	
For convenience, we state \autocite[Corollary~C.2]{khodak2019adaptive} in our notations:
\begin{lemma}
	\label{thm:epsEWOO}
	Let
	\begin{math}
	\braces{\lrLoss_\outCurRound: \reals_{+} \to \reals}_{\outCurRound=1}^{\outNumRounds}
	\end{math}
	be a sequence of functions of form
	\begin{math}
	\lrLoss_\outCurRound \pth{\decision}
	=
	\lrAlpha_\outCurRound \pth{\lrB_\outCurRound^2 \frac{1}{\decision} + \decision}
	\end{math}
	for any positive scalars
	\begin{math}
	\lrAlpha_\outCurRound \in \reals_{+}
	\end{math}
	and adversarily chosen
	\begin{math}
	\lrB_\outCurRound \in \interval{0}{\lrBLimit}
	\end{math}%
	.
	Then the \acrshort{eps-ewoo} algorithm, which for
	\begin{math}
	\lrEpsilon > 0
	\end{math}
	uses the actions of \acrshort{ewoo} \parencite{hazan2007logarithmic} run on the functions
	\begin{math}
	\lrLossSurrogate_\outCurRound \pth{\decision}
	=
	\lrAlpha_\outCurRound \pth{\pth{\lrB_\outCurRound^2 + \lrEpsilon^2} \frac{1}{\decision} + \decision}
	\end{math}
	over the domain
	\begin{math}
	\interval{\lrEpsilon}{\sqrt{\lrBLimit^2 + \lrEpsilon^2}}
	\end{math}%
	, \ie{} sets
	\begin{equation*}\label{key}
	\begin{aligned}
	\decision_\outCurRound
	=
	\frac{
		\int_{\lrEpsilon}^{\sqrt{\lrBLimit^2 + \lrEpsilon^2}}
		\decision \;
		\exp{ - \lrGamma \lrLossSurrogate_{1:\outCurRound-1} \pth{\decision} } \diff \decision
	}{
		\int_{\lrEpsilon}^{\sqrt{\lrBLimit^2 + \lrEpsilon^2}}
		\exp{ - \lrGamma \lrLossSurrogate_{1:\outCurRound-1} \pth{\decision} } \diff \decision
	}
	\end{aligned}
	\end{equation*}
	for
	\begin{math}
	\lrGamma
	=
	\frac{2}{\lrAlpha_{\max{}} \lrBLimit} \min\braces{\frac{\lrEpsilon^2}{\lrBLimit^2}, 1 } 
	\end{math}%
	, achieves cumulative regret
	\begin{equation*}\label{key}
	\begin{aligned}
	\regret\pth{\decisionBest}
	\leq
	\min\braces{\frac{\lrEpsilon^2}{\decisionBest}, \lrEpsilon}
	\lrAlpha_{1:\outNumRounds}
	+
	\frac{1}{2} \lrAlpha_{\max{}} \lrBLimit
	\max\braces{\frac{\lrBLimit^2}{\lrEpsilon^2}, 1 }
	\pth{1 + \log\pth{\outNumRounds + 1}}
	\end{aligned}
	\end{equation*}
	for any
	\begin{math}
	\decisionBest > 0
	\end{math}%
	.
\end{lemma}

\begin{lemma}[\cref{thm:lr-meta-regret}]
	\label{thm:lr-meta-regret-appendix}
	The regret of \gls{eps-ewoo} (\cref{alg:outer-lr}) on the loss function defined in \cref{eq:lr-meta-loss} is bounded as
	\begin{equation*}
	\begin{aligned}
	\E \regretLr \pth{\innPseudoLearningRate}
	&{\leq}
	\boundLrRegret \pth{\innPseudoLearningRate}
	%
	%
	%
	\\&{\leq}
	\bigOSymbol \bigg(
	\problemScale \outNumRounds 
	\min\braces{\frac{\lrEpsilon^2}{\innPseudoLearningRate}, \lrEpsilon}
	%
	%
	+
	\frac{
		\problemScale
		\log\pth{\outNumRounds}
	}
	{ \lrEpsilon^2 {\pth{1 - \anyWrongArmProb}}^{\nicefrac{3}{2}} \probLowerLimit^{\nicefrac{3}{4}} }
	\bigg)
	\end{aligned}
	\end{equation*}
\end{lemma}
\begin{proof}

Note: the expectation here is taken \wrt{} the estimated best arm.

The loss functions of meta-learner are:
\begin{equation}
\begin{aligned}
\lrLoss_\outCurRound \pth{\innPseudoLearningRate}
&{=}
\problemScale
\pth{
	{
		\frac{
			\betaDivergenceHalf{\basisVector{\bestArmEst_{\outCurRound}}^\probLowerLimit}{\innInitial_\outCurRound} 
		}{
			{
				1 -
				\anyWrongArmProb
			}
		}
	}
	\frac{1}{\innPseudoLearningRate}
	+
	\innPseudoLearningRate
}
\end{aligned}
\end{equation}

We assume
\begin{math}
\innInitial_\outCurRound \in \truncatedSimplex{\probLowerLimit}
\end{math}
and bound the beta-divergence using the worst case
\begin{math}
\innInitial_\outCurRound = \basisVector{j \neq \hat{j}_{\outCurRound}}^\probLowerLimit
\end{math}
:
\begin{equation}\label{key}
\begin{aligned}
\betaDivergenceHalf{\basisVector{\hat{j}_{\outCurRound}}^\probLowerLimit}{\innInitial_\outCurRound}
&{\leq}
\betaDivergenceHalf{\basisVector{\hat{j}_{\outCurRound}}^\probLowerLimit}{\basisVector{j \neq \hat{j}_{\outCurRound}}^\probLowerLimit}
\\&{=}
2
\sum_{i=1}^{\mabDim}
\pth{
	-
	2 \sqrt{{( \basisVector{\hat{j}_{\outCurRound}}^\probLowerLimit )}_{i}}
	+
	\sqrt{{( \basisVector{j}^\probLowerLimit )}_{i}}
	+
	\frac{
		{( \basisVector{\hat{j}_{\outCurRound}}^\probLowerLimit )}_{i} 
	}{
		\sqrt{{( \basisVector{j}^\probLowerLimit )}_{i}}
	}
}
\\&{=}
2
\pth{\mabDim - 2}
\pth{
	-
	2 \sqrt{\probLowerLimit}
	+
	\sqrt{\probLowerLimit}
	+
	\frac{
		\probLowerLimit 
	}{
		\sqrt{\probLowerLimit}
	}
}
\\&\quad
+
2
\pth{
	-
	2 \sqrt{1 - \pth{\mabDim-1} \probLowerLimit}
	+
	\sqrt{\probLowerLimit}
	+
	\frac{
		1 - \pth{\mabDim-1} \probLowerLimit 
	}{
		\sqrt{\probLowerLimit}
	}
}
\\&\quad
+
2
\pth{
	-
	2 \sqrt{\probLowerLimit}
	+
	\sqrt{1 - \pth{\mabDim-1} \probLowerLimit}
	+
	\frac{
		\probLowerLimit
	}{
		\sqrt{1 - \pth{\mabDim-1} \probLowerLimit}
	}
}
\\&{=}
2
\pth{
	\frac{
		1 - \pth{\mabDim-1} \probLowerLimit 
	}{
		\sqrt{\probLowerLimit}
	}
	+
	\frac{
		\probLowerLimit
	}{
		\sqrt{1 - \pth{\mabDim-1} \probLowerLimit}
	}
	-
	\sqrt{\probLowerLimit}
	-
	\sqrt{1 - \pth{\mabDim-1} \probLowerLimit}
}
\\&{=}
2
\pth{
	\frac{
		1 - \mabDim \probLowerLimit + \probLowerLimit - \probLowerLimit
	}{
		\sqrt{\probLowerLimit}
	}
	-
	\frac{
		1 - \mabDim \probLowerLimit + \probLowerLimit - \probLowerLimit
	}{
		\sqrt{1 - \pth{\mabDim-1} \probLowerLimit}
	}
}
\\&{=}
2
\pth{1 - \mabDim \probLowerLimit}
\pth{
	\frac{1}{\sqrt{\probLowerLimit}}
	-
	\frac{1}{\sqrt{1 - \pth{\mabDim-1} \probLowerLimit}}
}
\\\overset{\text{neglect}}&{\leq}
\frac{2}{\sqrt{\probLowerLimit}}
\end{aligned}
\end{equation}

The conditions of \cref{thm:epsEWOO} are satisfied with 
\begin{math}
\lrAlpha_\outCurRound = \problemScale
\end{math}%
,
\begin{math}
\lrBLimit =
\sqrt{
	\frac{2}{
		\pth{1 - \mabDim \wrongArmProb}
		\sqrt{\probLowerLimit}
	}
}
\end{math}%
, and therefore
\begin{equation}\label{key}
\begin{aligned}
\regretLr \pth{\innPseudoLearningRate}
&{\leq}
\min\braces{\frac{\lrEpsilon^2}{\innPseudoLearningRate}, \lrEpsilon}
\outNumRounds \problemScale
+
\frac{ \problemScale }
{ \sqrt{2} \sqrt{1 - \mabDim \wrongArmProb} \sqrt[4]{\probLowerLimit} }
\max\braces{
	\frac{2}{\lrEpsilon^2 \pth{1 - \mabDim \wrongArmProb} \sqrt{\probLowerLimit}}
	, 1 
}
\pth{1 + \log\pth{\outNumRounds + 1}}
\end{aligned}
\end{equation}

The bound is independent of the estimated best arm, and therefore
\begin{equation}\label{eq:lrMetaRegretTsallis1}
\begin{aligned}
\E
\regretLr \pth{\innPseudoLearningRate}
&{\leq}
\min\braces{\frac{\lrEpsilon^2}{\innPseudoLearningRate}, \lrEpsilon}
\outNumRounds \problemScale
+
\frac{ \problemScale }
{ \sqrt{2} \sqrt{1 - \mabDim \wrongArmProb} \sqrt[4]{\probLowerLimit} }
\max\braces{
	\frac{2}{\lrEpsilon^2 \pth{1 - \mabDim \wrongArmProb} \sqrt{\probLowerLimit}}
	, 1 
}
\pth{1 + \log\pth{\outNumRounds + 1}}
\end{aligned}
\end{equation}

To simplify the $\max{}$ expression we use:
\begin{equation}\label{eq:simplifyMax}
\begin{aligned}
\frac{2}{\lrEpsilon^2 \pth{1 - \mabDim \wrongArmProb} \sqrt{\probLowerLimit}}
&{=}
\frac{
	2	
	\sqrt{\outNumRounds}
}{
	\pth{1 - \mabDim \wrongArmProb} \sqrt{\probLowerLimit}
}
\\\overset{\probLowerLimit \leq \frac{1}{\mabDim}}&{\geq}
\frac{
	2	
	\sqrt{\outNumRounds}
	\sqrt{\mabDim}
}{
	\pth{1 - \mabDim \wrongArmProb} 
}
\\&{\geq}
2	
\sqrt{\outNumRounds}
\sqrt{\mabDim}
\\&{\geq}
1
\end{aligned}
\end{equation}

Overall,
\begin{equation}\label{eq:lrMetaRegretTsallis}
\begin{aligned}
\E
\regretLr \pth{\innPseudoLearningRate}
&{\leq}
\min\braces{\frac{\lrEpsilon^2}{\innPseudoLearningRate}, \lrEpsilon}
\outNumRounds \problemScale
+
\frac{ \sqrt{2} \problemScale }
{ \lrEpsilon^2 {\pth{1 - \anyWrongArmProb}}^{\nicefrac{3}{2}} \probLowerLimit^{\nicefrac{3}{4}} }
\pth{1 + \log\pth{\outNumRounds + 1}}
\\&{\triangleq}
\boundLrRegret \pth{\innPseudoLearningRate}
\\&{\leq}
	\bigOSymbol \bigg(
\problemScale \outNumRounds 
\min\braces{\frac{\lrEpsilon^2}{\innPseudoLearningRate}, \lrEpsilon}
%
%
+
\frac{
	\problemScale
	\log\pth{\outNumRounds}
}
{ \lrEpsilon^2 {\pth{1 - \anyWrongArmProb}}^{\nicefrac{3}{2}} \probLowerLimit^{\nicefrac{3}{4}} }
\bigg)
\end{aligned}
\end{equation}

\end{proof}

Additionally, the following rearrangement will be used later.
The expected meta-learner regret is:
\begin{equation}\label{key}
\begin{aligned}
\E
\regretLr \pth{\innPseudoLearningRate}
&{\triangleq}
\E
\sum_{\outCurRound=1}^{\outNumRounds} \pth{
	\frac{1}{
		{\innPseudoLearningRate_\outCurRound}
		\pth{1 - \mabDim \wrongArmProb}
	}
	\betaDivergenceHalf{\basisVector{\hat{j}_{\outCurRound}}^\probLowerLimit}{\innInitial_\outCurRound}
	+ 
	\innPseudoLearningRate_\outCurRound
}
\problemScale
\\&\quad
-
\E
\sum_{\outCurRound=1}^{\outNumRounds} \pth{
	\frac{1}{
		{\innPseudoLearningRate}
		\pth{1 - \mabDim \wrongArmProb}
	}
	\betaDivergenceHalf{\basisVector{\hat{j}_{\outCurRound}}^\probLowerLimit}{\innInitial_\outCurRound}
	+ 
	\innPseudoLearningRate
}
\problemScale
\end{aligned}
\end{equation}
implying
\begin{equation}\label{eq:mabOutLrBoundTsallis}
\begin{aligned}
\forall \innPseudoLearningRate
:\quad
&\E
\sum_{\outCurRound=1}^{\outNumRounds} \pth{
	\frac{1}{
		{\innPseudoLearningRate_\outCurRound}
		\pth{1 - \mabDim \wrongArmProb}
	}
	\betaDivergenceHalf{\basisVector{\hat{j}_{\outCurRound}}^\probLowerLimit}{\innInitial_\outCurRound}
	+ 
	\innPseudoLearningRate_\outCurRound
}
\problemScale
\\&\qquad =
\E
\regretLr \pth{\innPseudoLearningRate}
+
\E
\sum_{\outCurRound=1}^{\outNumRounds} \pth{
	\frac{1}{
		{\innPseudoLearningRate}
		\pth{1 - \mabDim \wrongArmProb}
	}
	\betaDivergenceHalf{\basisVector{\hat{j}_{\outCurRound}}^\probLowerLimit}{\innInitial_\outCurRound}
	+ 
	\innPseudoLearningRate
}
\problemScale
\\&\qquad \leq
\boundLrRegret \pth{\innPseudoLearningRate}
+
\E
\sum_{\outCurRound=1}^{\outNumRounds} \pth{
	\frac{1}{
		{\innPseudoLearningRate}
		\pth{1 - \mabDim \wrongArmProb}
	}
	\betaDivergenceHalf{\basisVector{\hat{j}_{\outCurRound}}^\probLowerLimit}{\innInitial_\outCurRound}
	+ 
	\innPseudoLearningRate
}
\problemScale
\end{aligned}
\end{equation}

}

\section{Regret of initialization-point meta-learning (proof of Lemma~\ref{thm:init-meta-regret})}
\label{sec:init-metalearning-regret}
{
\newcommand*{\regul}{\ensuremath{\varphi}}

\begin{lemma}[\cref{thm:init-meta-regret}]
	\label{thm:init-meta-regret-appendix}
	The regret of \gls{ftl} (\cref{alg:outer-init}) w.r.t.~the loss sequence \cref{eq:init-meta-loss} is bounded as
	\begin{equation*}
	\begin{aligned}
	\E \regretInit \pth{\innInitial}
	&{\leq}
	\boundInitRegret
	%
	%
	%
	{\leq}
	\bigO{}{
		\frac{\sqrt{\mabDim}}{ \sqrt{\probLowerLimit}}
		\log{\outNumRounds}
	}
	.
	\end{aligned}
	\end{equation*}
\end{lemma}
\begin{proof}
The losses are a sequence of beta-divergences.
They are not convex in the second argument.

However, they are Bregman divergences of negative Tsallis entropy regularizer $\regul_{\nicefrac{1}{2}}$ (see \cref{eq:def:regul}), which is \strongcvx{1} \wrt{}
\begin{math}
\norm{\cdot}_2
\end{math}
over $\truncatedSimplex{\probLowerLimit}$
(see \cref{eq:tsallis-derivative-2}).

For
\begin{math}
\decision \in \truncatedSimplex{\probLowerLimit}
\end{math}
we have
\begin{math}
\probLowerLimit
\leq
\decision_i
\leq
1 - \pth{\mabDim - 1} \probLowerLimit
\end{math}%
,
and therefore
\begin{equation}\label{key}
\begin{aligned}
\norm{
	\Grad{\decision} \regul_{\nicefrac{1}{2}} \pth{\decision}
}_2
&{=}
\sqrt{
	\sum_{i=1}^{\mabDim}
	\pth{
		\PartDiv{\regul_{\nicefrac{1}{2}} \pth{\decision}}{\decision_i}
	}
}
\\\overset{\tcref{eq:tsallis-derivative-1}}&{\leq}
\sqrt{
	\sum_{i=1}^{\mabDim}
	\frac{4}{\decision_i}
}
\\&{\leq}
\sqrt{
	\frac{4 \mabDim}{{\probLowerLimit}}
}
\end{aligned}
\end{equation}
Implying that the reguralizer is
\begin{math}
\sqrt{
	\frac{4 \mabDim}{{\probLowerLimit}}
}
\end{math}%
-Lipschits \wrt{}
\begin{math}
\norm{\cdot}_2
\end{math}%
.

The diameter of decision set is
\begin{equation}\label{key}
\begin{aligned}
D
&{=}
\max_{x, y \in \truncatedSimplex{\probLowerLimit}}
\norm{x - y}_2
\\&{\leq}
\max_{x, y \in \simplex{\mabDim}}
\norm{x - y}_2
\\&{=}
\sqrt{2}
\end{aligned}
\end{equation}

Therefore we can apply strongly convex coupling \autocite[Proposition~B.1]{khodak2019adaptive}, and get
\begin{equation}\label{eq:init_regret_tsallis}
\begin{aligned}
\regretInit
&{=}
\sum_{\outCurRound=1}^{\outNumRounds}	
\betaDivergenceHalf{\basisVector{\hat{j}_{\outCurRound}}^\probLowerLimit}{\innInitial_\outCurRound}
-
\min_{\innInitial \in \truncatedSimplex{\probLowerLimit}}
\sum_{\outCurRound=1}^{\outNumRounds}
\betaDivergenceHalf{\basisVector{\hat{j}_{\outCurRound}}^\probLowerLimit}{\innInitial}
\\&{\leq}
2
\sqrt{2}
\sqrt{
	\frac{4 \mabDim}{{\probLowerLimit}}
}
\sum_{\outCurRound=1}^{\outNumRounds}
\frac{1}{1 + 2 \pth{\outCurRound - 1}}
\\&{\leq}
4
\sqrt{2}
\sqrt{
	\frac{\mabDim}{{\probLowerLimit}}
}
\sum_{\outCurRound=1}^{\outNumRounds}
\frac{1}{\outCurRound}
\\\overset{\text{harmonic}}&{\leq}
4
\sqrt{2}
\sqrt{
	\frac{\mabDim}{{\probLowerLimit}}
}
\pth{ \log{\outNumRounds} + 1 }
\\&\triangleq
\boundInitRegret
\\&{\leq}
\bigO{}{
	\frac{\sqrt{\mabDim}}{ \sqrt{\probLowerLimit}}
	\log{\outNumRounds}
}
\end{aligned}
\end{equation}

\end{proof}

}

\section{Loss of best initialization in hindsight (proof of Lemma~\ref{thm:init-metalearning-bih-loss})}
\label{sec:init-metalearning-hindsight-loss}

{

\newcommand*{\mixWithUniform}[2]{\ensuremath{\mathcal{U}^{#1}}\pth{#2}}
\newcommand*{\mixWithUniformEl}[3]{\ensuremath{\mathcal{U}^{#1}_{#3}}\pth{#2}}
\newcommand*{\mixUniDelta}[1]{\mixWithUniform{\probLowerLimit \mabDim}{#1}}
\newcommand*{\mixUniDeltaEl}[2]{\mixWithUniformEl{\probLowerLimit \mabDim}{#1}{#2}}
\newcommand*{\pbar}{\ensuremath{\bar{p}}}
\newcommand*{\shrunkEntTs}{\ensuremath{h}}

\begin{lemma}[\cref{thm:init-metalearning-bih-loss}]
	\label{thm:init-metalearning-bih-loss-appendix}
	The expected loss of the best-in-hindsight point of the loss functions defined in \cref{eq:init-meta-loss} is bounded as
	\begin{equation*}
	\begin{aligned}
	\E
	\min_{\innInitial \in \truncatedSimplex{\probLowerLimit}}
	\sum_{\outCurRound=1}^{\outNumRounds}
	\betaDivergenceHalf{\basisVector{\hat{j}_{\outCurRound}}^\probLowerLimit}{\innInitial}
	&\leq
	\boundInitEstimationCost
	+
	\EntropyTsallisHalfSc{\initComparator}
	\outNumRounds,
	%
	%
	%
	\\&{\leq}
	\bigO{}{
		\frac{\anyWrongArmProb}{\sqrt{\probLowerLimit}}
		\outNumRounds 
		+
		\EntropyTsallisHalfSc{\initComparator}
		\outNumRounds
	}.
	\end{aligned}
	\end{equation*}
\end{lemma}
\begin{proof}
We consider the expected loss of the best initialization under the worst-case distribution of best arm estimation, which satisfies \cref{thm:identification-probability,eq:identification-probability-pair}.

\begin{equation}\label{eq:init_comparator_loss_tsallis_1}
\begin{aligned}
\E
\min_{\innInitial \in \truncatedSimplex{\probLowerLimit}}
\sum_{\outCurRound=1}^{\outNumRounds}
\betaDivergenceHalf{\basisVector{\hat{j}_{\outCurRound}}^\probLowerLimit}{\innInitial}
&\leq
\min_{\innInitial \in \truncatedSimplex{\probLowerLimit}}
\E
\sum_{\outCurRound=1}^{\outNumRounds}
\betaDivergenceHalf{\basisVector{\hat{j}_{\outCurRound}}^\probLowerLimit}{\innInitial}
\\&=
\min_{\innInitial \in \truncatedSimplex{\probLowerLimit}}
\sum_{\outCurRound=1}^{\outNumRounds}
\sum_{i=1}^{\mabDim}
\Prob\pth{\hat{j}_{\outCurRound} = i}
\betaDivergenceHalf{\basisVector{i}^\probLowerLimit}{\innInitial}
\\&\leq
\max_{\substack{
		\braces{p_{\outCurRound, i}}_{\outCurRound \in \brk{\outNumRounds}, i \in \brk{\mabDim}}
		\\ \subalign{
			\text{s.t.} \;&
			\sum_{i \in \brk{\mabDim}} p_{\outCurRound, i} = 1, \; \forall \outCurRound \in \brk{\outNumRounds}	
			\\&
			0 \leq p_{\outCurRound, i} \leq 2 \wrongArmProb, \; \forall \outCurRound \in \brk{\outNumRounds}, i \neq j^*_{\outCurRound}
			\\&
			1 - \mabDim \wrongArmProb \leq p_{\outCurRound, i} \leq 1, \; \forall \outCurRound \in \brk{\outNumRounds}, i = j^*_{\outCurRound}
}}}
\min_{\innInitial \in \truncatedSimplex{\probLowerLimit}}
\sum_{\outCurRound=1}^{\outNumRounds}
\sum_{i=1}^{\mabDim}
{p_{\outCurRound, i}}
\betaDivergenceHalf{\basisVector{i}^\probLowerLimit}{\innInitial}
\end{aligned}
\end{equation}

\paragraph{Finding the optimal $ \innInitial $.}

We will first find optimal $ \innInitial $ as function of \begin{math}
\braces{p_{\outCurRound, i}}_{\outCurRound \in \brk{\outNumRounds}, i \in \brk{\mabDim}}
\end{math}%
.

We will denote weighted mixture with uniform distribution as
\begin{math}
\mixWithUniform{\beta}{p}
\triangleq
{\pth{1 - \beta} p + \frac{\beta}{\mabDim} \onesVector }
\end{math}%
, and one element of such mixture as
\begin{math}
\mixWithUniformEl{\beta}{p}{i}
\triangleq
{\pth{1 - \beta} p_i + \beta \frac{1}{\mabDim} }
\end{math}

\begin{equation}\label{eq:minimizing_phi_tsallis}
\begin{aligned}
\min_{\innInitial \in \truncatedSimplex{\probLowerLimit}}
\sum_{\outCurRound=1}^{\outNumRounds}
\sum_{i=1}^{\mabDim}
&
{p_{\outCurRound, i}}
\betaDivergenceHalf{\basisVector{i}^\probLowerLimit}{\innInitial}
\\&=
2
\min_{\innInitial \in \truncatedSimplex{\probLowerLimit}}
\sum_{\outCurRound=1}^{\outNumRounds}
\sum_{i=1}^{\mabDim}
{p_{\outCurRound, i}}
\sum_{k=1}^{\mabDim}
\pth{
	-
	2 \sqrt{{( \basisVector{i}^\probLowerLimit )}_{k}}
	+
	\sqrt{\innInitial_k}
	+
	\frac{
		{( \basisVector{i}^\probLowerLimit )}_{k} 
	}{
		\sqrt{\innInitial_k}
	}
}
\\&{=}
-4
\sum_{\outCurRound=1}^{\outNumRounds}
\sum_{i=1}^{\mabDim}
{p_{\outCurRound, i}}
\sum_{k=1}^{\mabDim}
\sqrt{{( \basisVector{i}^\probLowerLimit )}_{k}}
\\&\quad
+
2
\min_{\innInitial \in \truncatedSimplex{\probLowerLimit}}
\sum_{k=1}^{\mabDim}
\pth{
	\sqrt{\innInitial_k}
	\outNumRounds
	+
	\frac{
		1
	}{
		\sqrt{\innInitial_k}
	}
	\sum_{\outCurRound=1}^{\outNumRounds}
	\sum_{i=1}^{\mabDim}
	{p_{\outCurRound, i}}
	{( \basisVector{i}^\probLowerLimit )}_{k} 
}
\\&{=}
-4
\sum_{\outCurRound=1}^{\outNumRounds}
\sum_{i=1}^{\mabDim}
{p_{\outCurRound, i}}
\pth{
	\pth{\mabDim - 1}
	\sqrt{\probLowerLimit}
	+
	\sqrt{1 - \pth{\mabDim - 1} \probLowerLimit}
}
\\&\quad
+
2
\min_{\innInitial \in \truncatedSimplex{\probLowerLimit}}
\sum_{k=1}^{\mabDim}
\pth{
	\sqrt{\innInitial_k}
	\outNumRounds
	+
	\frac{
		1
	}{
		\sqrt{\innInitial_k}
	}
	\pth{
		\probLowerLimit
		\outNumRounds
		+
		\pth{1 - \mabDim \probLowerLimit}
		\sum_{\outCurRound=1}^{\outNumRounds}
		{p_{\outCurRound, k}}
	}
}
\\&{=}
-4
\sum_{\outCurRound=1}^{\outNumRounds}
\pth{
	\pth{\mabDim - 1}
	\sqrt{\probLowerLimit}
	+
	\sqrt{1 - \pth{\mabDim - 1} \probLowerLimit}
}
\\&\quad
+
2
\outNumRounds
\min_{\innInitial \in \truncatedSimplex{\probLowerLimit}}
\sum_{k=1}^{\mabDim}
\pth{
	\sqrt{\innInitial_k}
	+
	\frac{
		1
	}{
		\sqrt{\innInitial_k}
	}
	\pth{
		\probLowerLimit
		+
		\pth{1 - \mabDim \probLowerLimit}
		\frac{
			\sum_{\outCurRound=1}^{\outNumRounds}
			{p_{\outCurRound, k}}
		}{
			\sum_{i=1}^{\mabDim}
			\sum_{\outCurRound=1}^{\outNumRounds}
			{p_{\outCurRound, i}}
		}
	}
}
\\&{=}
-4
\outNumRounds
\pth{
	\pth{\mabDim - 1}
	\sqrt{\probLowerLimit}
	+
	\sqrt{1 - \pth{\mabDim - 1} \probLowerLimit}
}
\\&\quad
+
2
\outNumRounds
\min_{\innInitial \in \truncatedSimplex{\probLowerLimit}}
\sum_{k=1}^{\mabDim}
\pth{
	\sqrt{\innInitial_k}
	+
	\frac{
		1
	}{
		\sqrt{\innInitial_k}
	}
	\pth{
		\probLowerLimit
		+
		\pth{1 - \mabDim \probLowerLimit}
		\pbar_k
	}
}
\\&{=}
-4
\outNumRounds
\sum_{k=1}^{\mabDim}
\sqrt{
	\mixUniDeltaEl{\basisVector{1}}{k}
}
\\&\quad
+
2
\outNumRounds
\min_{\innInitial \in \truncatedSimplex{\probLowerLimit}}
\sum_{k=1}^{\mabDim}
\pth{
	\sqrt{\innInitial_k}
	+
	\frac{
		1
	}{
		\sqrt{\innInitial_k}
	}
	\mixUniDeltaEl{\pbar}{k}
}
\\&{=}
-4
\outNumRounds
\sum_{k=1}^{\mabDim}
\sqrt{
	\mixUniDeltaEl{\basisVector{1}}{k}
}
+
4
\outNumRounds
\sum_{k=1}^{\mabDim}
\sqrt{\mixUniDeltaEl{\pbar}{k}}
\\&\quad
+
2
\outNumRounds
\min_{\innInitial \in \truncatedSimplex{\probLowerLimit}}
\sum_{k=1}^{\mabDim}
\pth{
	-2
	\sqrt{\mixUniDeltaEl{\pbar}{k}}
	+
	\sqrt{\innInitial_k}
	+
	\frac{
		1
	}{
		\sqrt{\innInitial_k}
	}
	\mixUniDeltaEl{\pbar}{k}
}
\\&{=}
-4
\outNumRounds
\sum_{k=1}^{\mabDim}
\pth{
	\sqrt{
		\mixUniDeltaEl{\basisVector{1}}{k}
	}
	- 1
}
+
4
\outNumRounds
\sum_{k=1}^{\mabDim}
\pth{
	\sqrt{\mixUniDeltaEl{\pbar}{k}}
	- 1
}
\\&\quad
+
2
\outNumRounds
\min_{\innInitial \in \truncatedSimplex{\probLowerLimit}}
\sum_{k=1}^{\mabDim}
\pth{
	-2
	\sqrt{\mixUniDeltaEl{\pbar}{k}}
	+
	\sqrt{\innInitial_k}
	+
	\frac{
		1
	}{
		\sqrt{\innInitial_k}
	}
	\mixUniDeltaEl{\pbar}{k}
}
\\&{=}
-
\outNumRounds
\EntropyTsallisHalfSc{
	\mixUniDelta{\basisVector{1}}
}
+
\outNumRounds
\EntropyTsallisHalfSc{
	\mixUniDelta{\pbar}
}
+
\outNumRounds
\min_{\innInitial \in \truncatedSimplex{\probLowerLimit}}
\betaDivergenceHalf{\mixUniDelta{\pbar}}{\innInitial}
\\&{=}
-
\outNumRounds
\EntropyTsallisHalfSc{
	\mixUniDelta{\basisVector{1}}
}
+
\outNumRounds
\EntropyTsallisHalfSc{
	\mixUniDelta{\pbar}
}
\end{aligned}
\end{equation}

The last equality follows from the fact that
\begin{math}
\mixUniDelta{\pbar}
\in
\truncatedSimplex{\probLowerLimit}
\end{math}%
, allowing us to set
\begin{math}
\innInitial
=
\mixUniDelta{\pbar}
\end{math}
and achieve
\begin{math}
\betaDivergenceHalf{\mixUniDelta{\pbar}}{\innInitial}
=
0
\end{math}%
.

\paragraph{Bounding the optimization of $ p $.}

Now we develop a bound for the worst-case $p$.

\begin{equation}\label{eq:init_comparator_loss_tsallis_2}
\begin{aligned}
\E
\min_{\innInitial \in \truncatedSimplex{\probLowerLimit}}
&
\sum_{\outCurRound=1}^{\outNumRounds}
\betaDivergenceHalf{\basisVector{\hat{j}_{\outCurRound}}^\probLowerLimit}{\innInitial}
\\&\leq
\max_{\substack{
		\braces{p_{\outCurRound, i}}_{\outCurRound \in \brk{\outNumRounds}, i \in \brk{\mabDim}}
		\\ \subalign{
			\text{s.t.} \;&
			\sum_{i \in \brk{\mabDim}} p_{\outCurRound, i} = 1, \; \forall \outCurRound \in \brk{\outNumRounds}	
			\\&
			0 \leq p_{\outCurRound, i} \leq 2 \wrongArmProb, \; \forall \outCurRound \in \brk{\outNumRounds}, i \neq j^*_{\outCurRound}
			\\&
			1 - \mabDim \wrongArmProb \leq p_{\outCurRound, i} \leq 1, \; \forall \outCurRound \in \brk{\outNumRounds}, i = j^*_{\outCurRound}
}}}
\min_{\innInitial \in \truncatedSimplex{\probLowerLimit}}
\sum_{\outCurRound=1}^{\outNumRounds}
\sum_{i=1}^{\mabDim}
{p_{\outCurRound, i}}
\betaDivergenceHalf{\basisVector{i}^\probLowerLimit}{\innInitial}
\\&=
\max_{\substack{
		\braces{p_{\outCurRound, i}}_{\outCurRound \in \brk{\outNumRounds}, i \in \brk{\mabDim}}
		\\ \subalign{
			\text{s.t.} \;&
			\sum_{i \in \brk{\mabDim}} p_{\outCurRound, i} = 1, \; \forall \outCurRound \in \brk{\outNumRounds}	
			\\&
			0 \leq p_{\outCurRound, i} \leq 2 \wrongArmProb, \; \forall \outCurRound \in \brk{\outNumRounds}, i \neq j^*_{\outCurRound}
			\\&
			1 - \mabDim \wrongArmProb \leq p_{\outCurRound, i} \leq 1, \; \forall \outCurRound \in \brk{\outNumRounds}, i = j^*_{\outCurRound}
}}}
\pth{
	-
	\outNumRounds
	\EntropyTsallisHalfSc{
		\mixUniDelta{\basisVector{1}}
	}
	+
	\outNumRounds
	\EntropyTsallisHalfSc{
		\mixUniDelta{\pbar}
	}
}
\\&=
-
\outNumRounds
\EntropyTsallisHalfSc{
	\mixUniDelta{\basisVector{1}}
}
+
\outNumRounds
\cdot
\max_{\substack{
		\braces{p_{\outCurRound, i}}_{\outCurRound \in \brk{\outNumRounds}, i \in \brk{\mabDim}}
		\\ \subalign{
			\text{s.t.} \;&
			\sum_{i \in \brk{\mabDim}} p_{\outCurRound, i} = 1, \; \forall \outCurRound \in \brk{\outNumRounds}	
			\\&
			0 \leq p_{\outCurRound, i} \leq 2 \wrongArmProb, \; \forall \outCurRound \in \brk{\outNumRounds}, i \neq j^*_{\outCurRound}
			\\&
			1 - \mabDim \wrongArmProb \leq p_{\outCurRound, i} \leq 1, \; \forall \outCurRound \in \brk{\outNumRounds}, i = j^*_{\outCurRound}
}}}
\EntropyTsallisHalfSc{
	\mixUniDelta{\pbar}
}
\\&\leq
%
-
\outNumRounds
\EntropyTsallisHalfSc{
	\mixUniDelta{\basisVector{1}}
}
+
\outNumRounds
\cdot
\max_{\substack{
		\braces{p_{\outCurRound,i}}_{\outCurRound \in \brk{\outNumRounds}, i \in \brk{\mabDim}}
		\\ \subalign{
			\text{s.t.} \;&
			\sum_{\outCurRound \in \brk{\outNumRounds}} \sum_{i \in \brk{\mabDim}} p_{\outCurRound,i} = \outNumRounds
			\\&
			\sum_{\outCurRound \in \brk{\outNumRounds}} p_{\outCurRound,i}
			\geq
			0 \cdot \pth{1 - \initComparator_i} \outNumRounds + \pth{1 - \mabDim \wrongArmProb} \initComparator_i \outNumRounds
			, \; \forall i
			\\&
			\sum_{\outCurRound \in \brk{\outNumRounds}} p_{\outCurRound,i}
			\leq
			\pth{2 \wrongArmProb \pth{1 - \initComparator_i} \outNumRounds + 1 \cdot \initComparator_i \outNumRounds }
			, \; \forall i
}}}
\EntropyTsallisHalfSc{
	\mixUniDelta{\pbar}
}
\\&=
-
\outNumRounds
\EntropyTsallisHalfSc{
	\mixUniDelta{\basisVector{1}}
}
+
\outNumRounds
\cdot
\max_{\substack{
		\bar{p}
		\\ \subalign{
			\text{s.t.} \;&
			\sum_{i \in \brk{\mabDim}} \bar{p}_{i} = 1
			\\&
			\bar{p}_{i}
			\;\geq\;
			\pth{1 - \mabDim \wrongArmProb} \initComparator_i
			, \; \forall i
			\\&
			\bar{p}_{i}	
			\;\leq\;
			2 \wrongArmProb + \pth{1 - 2 \wrongArmProb} \initComparator_i
			, \; \forall i
}}}
\EntropyTsallisHalfSc{
	\mixUniDelta{\pbar}
}
\\&=
-
\outNumRounds
\EntropyTsallisHalfSc{
	\mixUniDelta{\basisVector{1}}
}
-
\outNumRounds
\cdot
\min_{\substack{
		\bar{p}
		\\ \subalign{
			\text{s.t.} \;&
			\sum_{i \in \brk{\mabDim}} \bar{p}_{i} = 1
			\\&
			\bar{p}_{i}
			\;\geq\;
			\pth{1 - \mabDim \wrongArmProb} \initComparator_i
			, \; \forall i
			\\&
			\bar{p}_{i}	
			\;\leq\;
			2 \wrongArmProb + \pth{1 - 2 \wrongArmProb} \initComparator_i
			, \; \forall i
}}}
\pth{
	-
	\EntropyTsallisHalfSc{
		\mixUniDelta{\pbar}
	}
}
\end{aligned}
\end{equation}

We now proceed to bound the second term for the worst case of $\pbar$.
The function
\begin{math}
\shrunkEntTs \pth{\cdot}
\triangleq
-
\EntropyTsallisHalfSc{
	\mixUniDelta{\cdot}
}
\end{math}
is convex, and therefore
\begin{equation}\label{key}
\begin{aligned}
\shrunkEntTs \pth{x}
&{\geq}
\shrunkEntTs \pth{y}
+
\innerProd{
	\Grad{} \shrunkEntTs \pth{y}
}{
	\pth{x - y}
}
\end{aligned}
\end{equation}

\begin{equation}\label{key}
\begin{aligned}
\PartDiv{
	\shrunkEntTs \pth{x}
}{
	x_i
}
&{=}
-
2
\frac{
	1 - \probLowerLimit \mabDim
}{
	\sqrt{\pth{1 - \probLowerLimit \mabDim} x_i + \probLowerLimit}
}
\end{aligned}
\end{equation}

\begin{equation}\label{key}
\begin{aligned}
\norm{
	\Grad{} \shrunkEntTs \pth{x}
}_\infty
&{\leq}
2
\max_{i}
\abs{
	\frac{
		1 - \probLowerLimit \mabDim
	}{
		\sqrt{\pth{1 - \probLowerLimit \mabDim} x_i + \probLowerLimit}
	}
}
\\\overset{\probLowerLimit \leq \frac{1}{\mabDim}}&{=}
2
\max_{i}
\frac{
	1 - \probLowerLimit \mabDim
}{
	\sqrt{\pth{1 - \probLowerLimit \mabDim} x_i + \probLowerLimit}
}
\\\overset{0 \leq x_i \leq 1}&{\leq}
2
\frac{
	1 - \probLowerLimit \mabDim
}{
	\sqrt{\probLowerLimit}
}
\\&{\leq}
\frac{2}{\sqrt{\probLowerLimit}}
\end{aligned}
\end{equation}

We will need to bound from above the $l_1$ distance between $\pbar$ and $\initComparator$.
A very coarse way to do it is to count contributions of each coordinate as if they simultaneously attain the lower and upper limit of their range:
\begin{equation}\label{eq:l1_pbar_phi}
\begin{aligned}
\max_{\substack{
		\bar{p}
		\\ \subalign{
			\text{s.t.} \;&
			\sum_{i \in \brk{\mabDim}} \bar{p}_{i} = 1
			\\&
			\bar{p}_{i}
			\;\geq\;
			\pth{1 - \mabDim \wrongArmProb} \initComparator_i
			, \; \forall i
			\\&
			\bar{p}_{i}	
			\;\leq\;
			2 \wrongArmProb + \pth{1 - 2 \wrongArmProb} \initComparator_i
			, \; \forall i
}}}
{\norm{
		\bar{p}
		-
		\initComparator
}}_1
&=
\max_{\substack{
		a
		\\ \subalign{
			\text{s.t.} \;&
			\sum_{i \in \brk{\mabDim}} a_{i} = 0
			\\&
			a_{i}
			\;\geq\;
			- \mabDim \wrongArmProb \initComparator_i
			, \; \forall i
			\\&
			a_{i}	
			\;\leq\;
			2 \wrongArmProb \pth{1 - \initComparator_i}
			, \; \forall i
}}}
{\norm{
		a
}}_1
\\&\leq
\sum_{i=1}^{\mabDim}
\mabDim \wrongArmProb \initComparator_i
+
\sum_{i=1}^{\mabDim}
2 \wrongArmProb \pth{1 - \initComparator_i}
\\&=
\mabDim \wrongArmProb
+
2 \wrongArmProb \pth{\mabDim - 1}
\\&=
\pth{3 \mabDim - 2} \wrongArmProb
\\&{\leq}
3 \mabDim \wrongArmProb
\end{aligned}
\end{equation}


And then
\begin{equation}\label{key}
\begin{aligned}
\min_{\substack{
		\bar{p}
		\\ \subalign{
			\text{s.t.} \;&
			\sum_{i \in \brk{\mabDim}} \bar{p}_{i} = 1
			\\&
			\bar{p}_{i}
			\;\geq\;
			\pth{1 - \mabDim \wrongArmProb} \initComparator_i
			, \; \forall i
			\\&
			\bar{p}_{i}	
			\;\leq\;
			2 \wrongArmProb + \pth{1 - 2 \wrongArmProb} \initComparator_i
			, \; \forall i
}}}
&
\pth{
	-
	\EntropyTsallisHalfSc{
		\mixUniDelta{\pbar}
	}
}
\\\\&{=}
\min_{\substack{
		\bar{p}
		\\ \subalign{
			\text{s.t.} \;&
			\sum_{i \in \brk{\mabDim}} \bar{p}_{i} = 1
			\\&
			\bar{p}_{i}
			\;\geq\;
			\pth{1 - \mabDim \wrongArmProb} \initComparator_i
			, \; \forall i
			\\&
			\bar{p}_{i}	
			\;\leq\;
			2 \wrongArmProb + \pth{1 - 2 \wrongArmProb} \initComparator_i
			, \; \forall i
}}}
\shrunkEntTs \pth{\pbar}
\\&{\geq}
\min_{\substack{
		\bar{p}
		\\ \subalign{
			\text{s.t.} \;&
			\sum_{i \in \brk{\mabDim}} \bar{p}_{i} = 1
			\\&
			\bar{p}_{i}
			\;\geq\;
			\pth{1 - \mabDim \wrongArmProb} \initComparator_i
			, \; \forall i
			\\&
			\bar{p}_{i}	
			\;\leq\;
			2 \wrongArmProb + \pth{1 - 2 \wrongArmProb} \initComparator_i
			, \; \forall i
}}}
\pth{
	\shrunkEntTs \pth{\initComparator}
	+
	\innerProd{
		\Grad{} \shrunkEntTs \pth{\initComparator}
	}{
		\pth{\pbar - \initComparator}
	}
}
\\&{\geq}
\min_{\substack{
		\bar{p}
		\\ \subalign{
			\text{s.t.} \;&
			\sum_{i \in \brk{\mabDim}} \bar{p}_{i} = 1
			\\&
			\bar{p}_{i}
			\;\geq\;
			\pth{1 - \mabDim \wrongArmProb} \initComparator_i
			, \; \forall i
			\\&
			\bar{p}_{i}	
			\;\leq\;
			2 \wrongArmProb + \pth{1 - 2 \wrongArmProb} \initComparator_i
			, \; \forall i
}}}
\pth{
	\shrunkEntTs \pth{\initComparator}
	-
	\norm{
		\Grad{} \shrunkEntTs \pth{\initComparator}
	}_\infty
	\norm{
		\pth{\pbar - \initComparator}
	}_1
}
\\&{=}
\shrunkEntTs \pth{\initComparator}
-
\norm{
	\Grad{} \shrunkEntTs \pth{\initComparator}
}_\infty
\cdot
\max_{\substack{
		\bar{p}
		\\ \subalign{
			\text{s.t.} \;&
			\sum_{i \in \brk{\mabDim}} \bar{p}_{i} = 1
			\\&
			\bar{p}_{i}
			\;\geq\;
			\pth{1 - \mabDim \wrongArmProb} \initComparator_i
			, \; \forall i
			\\&
			\bar{p}_{i}	
			\;\leq\;
			2 \wrongArmProb + \pth{1 - 2 \wrongArmProb} \initComparator_i
			, \; \forall i
}}}
\norm{
	\pth{\pbar - \initComparator}
}_1
\\&{\geq}
\shrunkEntTs \pth{\initComparator}
-
3 \mabDim \wrongArmProb
\;
\norm{
	\Grad{} \shrunkEntTs \pth{\initComparator}
}_\infty
\\&{\geq}
\shrunkEntTs \pth{\initComparator}
-
6 \mabDim \wrongArmProb
\frac{1}{\sqrt{\probLowerLimit}}
\\&{=}
-
\EntropyTsallisHalfSc{
	\mixUniDelta{\initComparator}
}
-
6 \mabDim \wrongArmProb
\frac{1}{\sqrt{\probLowerLimit}}
\end{aligned}
\end{equation}

Substituting into \cref{eq:init_comparator_loss_tsallis_2} we get:
\begin{equation}\label{eq:init_comparator_loss_tsallis_3}
\begin{aligned}
\E
\min_{\innInitial \in \truncatedSimplex{\probLowerLimit}}
\sum_{\outCurRound=1}^{\outNumRounds}
\betaDivergenceHalf{\basisVector{\hat{j}_{\outCurRound}}^\probLowerLimit}{\innInitial}
&\leq
-
\outNumRounds
\EntropyTsallisHalfSc{
	\mixUniDelta{\basisVector{1}}
}
\\&\qquad
-
\outNumRounds
\cdot
\min_{\substack{
		\bar{p}
		\\ \subalign{
			\text{s.t.} \;&
			\sum_{i \in \brk{\mabDim}} \bar{p}_{i} = 1
			\\&
			\bar{p}_{i}
			\;\geq\;
			\pth{1 - \mabDim \wrongArmProb} \initComparator_i
			, \; \forall i
			\\&
			\bar{p}_{i}	
			\;\leq\;
			2 \wrongArmProb + \pth{1 - 2 \wrongArmProb} \initComparator_i
			, \; \forall i
}}}
\pth{
	-
	\EntropyTsallisHalfSc{
		\mixUniDelta{\pbar}
	}
}
\\&{\leq}
-
\outNumRounds
\EntropyTsallisHalfSc{
	\mixUniDelta{\basisVector{1}}
}
+
\outNumRounds
\EntropyTsallisHalfSc{
	\mixUniDelta{\initComparator}
}
+
6 
\outNumRounds
\mabDim \wrongArmProb
\frac{1}{\sqrt{\probLowerLimit}}
\end{aligned}
\end{equation}

\paragraph{Simplifying dependence on $\probLowerLimit$.}

We now look at the component:
\begin{equation}\label{key}
\begin{aligned}
\EntropyTsallisHalfSc{
	\mixUniDelta{\initComparator}
}
-
\EntropyTsallisHalfSc{
	\mixUniDelta{\basisVector{1}}
}
&{=}
4
\sum_{i=1}^{\mabDim}
\pth{
	\sqrt{\pth{1 - \probLowerLimit \mabDim} \initComparator_i + \probLowerLimit}
	-
	\sqrt{\pth{1 - \probLowerLimit \mabDim} \pth{\basisVector{1}}_i + \probLowerLimit}
}
\end{aligned}
\end{equation}

Derivative of the shrunk entropy \wrt{} $\probLowerLimit$ is
\begin{equation}\label{key}
\begin{aligned}
\frac{1}{2}
\PartDiv{\pth{
		\EntropyTsallisHalfSc{
			\mixUniDelta{p}
		}
}}{
	\probLowerLimit
}
&{=}
\sum_{i=1}^{\mabDim}
\pth{
	\frac{1 - \mabDim p_i}{
		\sqrt{\pth{1 - \probLowerLimit \mabDim} p_i + \probLowerLimit}
	}
}
\\&{=}
\sum_{i=1}^{\mabDim-1}
\pth{
	\frac{1 - \mabDim p_i}{
		\sqrt{\pth{1 - \probLowerLimit \mabDim} p_i + \probLowerLimit}
	}
}
+
{
	\frac{1 - \mabDim p_\mabDim}{
		\sqrt{\pth{1 - \probLowerLimit \mabDim} p_\mabDim + \probLowerLimit}
	}
}
\\\overset{\sum p_i = 1}&{=}
\sum_{i=1}^{\mabDim-1}
\pth{
	\frac{1 - \mabDim p_i}{
		\sqrt{\pth{1 - \probLowerLimit \mabDim} p_i + \probLowerLimit}
	}
}
+
{
	\frac{1 - \mabDim \pth{1 - \sum_{i=1}^{\mabDim} p_i}}{
		\sqrt{\pth{1 - \probLowerLimit \mabDim} p_\mabDim + \probLowerLimit}
	}
}
\\&{=}
\sum_{i=1}^{\mabDim-1}
\pth{
	\frac{1 - \mabDim p_i}{
		\sqrt{\pth{1 - \probLowerLimit \mabDim} p_i + \probLowerLimit}
	}
}
-
\sum_{i=1}^{\mabDim-1}
\pth{
	\frac{1 - \mabDim p_i}{
		\sqrt{\pth{1 - \probLowerLimit \mabDim} p_\mabDim + \probLowerLimit}
	}
}
\\&{=}
\sum_{i=1}^{\mabDim-1}
\pth{1 - \mabDim p_i}
\pth{
	\frac{1}{
		\sqrt{\pth{1 - \probLowerLimit \mabDim} p_i + \probLowerLimit}
	}
	-
	\frac{1}{
		\sqrt{\pth{1 - \probLowerLimit \mabDim} p_\mabDim + \probLowerLimit}
	}
}
\end{aligned}
\end{equation}

And then the derivative of this component is:
\begin{equation}\label{key}
\begin{aligned}
\frac{1}{2}
&
\PartDiv{\pth{
		\EntropyTsallisHalfSc{
			\mixUniDelta{\initComparator}
		}
		-
		\EntropyTsallisHalfSc{
			\mixUniDelta{\basisVector{1}}
		}
}}{
	\probLowerLimit
}
\\&{=}
\sum_{i=1}^{\mabDim-1}
\pth{1 - \mabDim \initComparator_i}
\pth{
	\frac{1}{
		\sqrt{\pth{1 - \probLowerLimit \mabDim} \initComparator_i + \probLowerLimit}
	}
	-
	\frac{1}{
		\sqrt{\pth{1 - \probLowerLimit \mabDim} \initComparator_\mabDim + \probLowerLimit}
	}
}
\\&\quad
-
\sum_{i=1}^{\mabDim-1}
\pth{
	\frac{1}{
		\sqrt{\pth{1 - \probLowerLimit \mabDim} \cdot 0 + \probLowerLimit}
	}
	-
	\frac{1}{
		\sqrt{\pth{1 - \probLowerLimit \mabDim} \cdot 1 + \probLowerLimit}
	}
}
\\&{=}
\sum_{i=1}^{\mabDim-1}
\pth{
	\frac{1}{
		\sqrt{\pth{1 - \probLowerLimit \mabDim} \initComparator_i + \probLowerLimit}
	}
	-
	\frac{1}{
		\sqrt{\pth{1 - \probLowerLimit \mabDim} \cdot 0 + \probLowerLimit}
	}
}
\\&\quad
+
\sum_{i=1}^{\mabDim-1}
\pth{
	\frac{1}{
		\sqrt{\pth{1 - \probLowerLimit \mabDim} \cdot 1 + \probLowerLimit}
	}
	-
	\frac{1}{
		\sqrt{\pth{1 - \probLowerLimit \mabDim} \initComparator_\mabDim + \probLowerLimit}
	}	
}
\\&\quad
+
\mabDim
\sum_{i=1}^{\mabDim-1}
{\initComparator_i}
\pth{
	\frac{1}{
		\sqrt{\pth{1 - \probLowerLimit \mabDim} \initComparator_\mabDim + \probLowerLimit}
	}
	-
	\frac{1}{
		\sqrt{\pth{1 - \probLowerLimit \mabDim} \initComparator_i + \probLowerLimit}
	}
}
\end{aligned}
\end{equation}

If we assume, without loss of generality, that
\begin{math}
\initComparator_1 \leq \initComparator_2 \leq \dots \leq \initComparator_\mabDim
\end{math}%
, then each of the summands above become non-positive.
So for
\begin{math}
\probLowerLimit \in \interval[open left]{0}{1/\mabDim}
\end{math}
the derivative 
is non-positive, and for 
\begin{math}
\probLowerLimit \to 0^+
\end{math}
it goes to 
\begin{math}
- \infty
\end{math}%
.
Thus
this function
is monotonically non-increasing in $ \probLowerLimit $ for 
\begin{math}
\probLowerLimit \in \interval{0}{1/\mabDim}
\end{math}
.

For
\begin{math}
\probLowerLimit = 0
\end{math}
we have
\begin{math}
\EntropyTsallisHalfSc{
	\mixUniDelta{\initComparator}
}
-
\EntropyTsallisHalfSc{
	\mixUniDelta{\basisVector{1}}
}
=
\EntropyTsallisHalfSc{\initComparator}
\end{math}%
, and therefore for any
\begin{math}
\probLowerLimit \in \interval{0}{1/\mabDim}
\end{math}%
:
\begin{equation}\label{key}
\begin{aligned}
\EntropyTsallisHalfSc{
	\mixUniDelta{\initComparator}
}
-
\EntropyTsallisHalfSc{
	\mixUniDelta{\basisVector{1}}
}
\leq
\EntropyTsallisHalfSc{\initComparator}
\end{aligned}
\end{equation}

Substituting into \cref{eq:init_comparator_loss_tsallis_3}, we get:
\begin{equation}\label{eq:init_comparator_loss_tsallis}
\begin{aligned}
\E
\min_{\innInitial \in \truncatedSimplex{\probLowerLimit}}
\sum_{\outCurRound=1}^{\outNumRounds}
\betaDivergenceHalf{\basisVector{\hat{j}_{\outCurRound}}^\probLowerLimit}{\innInitial}
&\leq
-
\outNumRounds
\EntropyTsallisHalfSc{
	\mixUniDelta{\basisVector{1}}
}
+
\outNumRounds
\EntropyTsallisHalfSc{
	\mixUniDelta{\initComparator}
}
+
6 
\outNumRounds
\mabDim \wrongArmProb
\frac{1}{\sqrt{\probLowerLimit}}
\\&{\leq}
\outNumRounds
\EntropyTsallisHalfSc{\initComparator}
+
6 
\outNumRounds
\mabDim \wrongArmProb
\frac{1}{\sqrt{\probLowerLimit}}
\\&{\triangleq}
\EntropyTsallisHalfSc{\initComparator}
\outNumRounds
+
\boundInitEstimationCost
\\&{\leq}
\bigO{}{
	\EntropyTsallisHalfSc{\initComparator}
	\outNumRounds
	+
	\frac{\anyWrongArmProb}{\sqrt{\probLowerLimit}}
	\outNumRounds 
}
\end{aligned}
\end{equation}

%

\end{proof}
}

\section{Proof of Theorem~\ref{thm:total-regret}}
\label{sec:total-regret-proof}

\begin{theorem}[Regret upper bound, \cref{thm:total-regret}]
	\label{thm:total-regret-appendix}
	Under \cref{thm:assum:minimal-gap,thm:assum:positive-probability}, the expected total regret guaranteed by \acrshort{meta-inf} is
	\begin{equation*}\label{key}
	\begin{aligned}
	\E \regretTotal
	&{\leq}
	\min_{\innPseudoLearningRate}
	\bigg{[}
	\boundExplCost \outNumRounds
	+
	\boundLrRegret \pth{\innPseudoLearningRate}
	+ 
	\problemScale
	\outNumRounds
	\innPseudoLearningRate
	\\&{\quad}
	+
	\frac{\problemScale}{1 - \mabDim \wrongArmProb}
	\pth{
		\boundInitRegret 
		+
		\boundInitEstimationCost
		+
		\EntropyTsallisHalfSc{\initComparator}
		\outNumRounds
	}
	\frac{1}{\innPseudoLearningRate}
	\bigg{]}.
	\end{aligned}
	\end{equation*}
	Here
	$\boundExplCost$ is the cost associated with guaranteed exploration (see \cref{thm:observable-inner-regret}).
	$\boundLrRegret$ is an upper bound on the regret of meta-learning the learning-rate (see \cref{thm:lr-meta-regret}).
	$\boundInitRegret$ is an upper bound on the regret of meta-learning the initialization point (see \cref{thm:init-meta-regret}).
	$\boundInitEstimationCost$ is an upper bound on the difference between the entropy of empirical distribution of the true best arms, and the expected estimated entropy (see \cref{thm:init-metalearning-bih-loss}).
\end{theorem}
\begin{proof}
	
The total expected regret over all episodes can be bounded as follows:
\begin{equation}\label{key}
\begin{aligned}
\E
\regretTotal
\overset{\tcref{def:total-regret}}&{=}
\E
\sum_{\outCurRound=1}^{\outNumRounds} \sum_{\innCurRound=1}^{\innNumRounds} \pth{ \lossFunc_{\outCurRound, \innCurRound, \mabDecision_{\outCurRound, \innCurRound}} - \lossFunc_{\outCurRound, \innCurRound, j^*_{\outCurRound}} }
\\&{=}
\sum_{\outCurRound=1}^{\outNumRounds}
\E \regretInner \pth{\basisVector{{\bestArmTrue_\outCurRound}}}
\\\overset{\tcref{thm:observable-inner-regret}}&{\leq}
\sum_{\outCurRound=1}^{\outNumRounds}
\pth{
	\frac{
		\E
		\betaDivergenceHalf{\basisVector{{\bestArmEst_\outCurRound}}^{\probLowerLimit}}{\innInitial_\outCurRound}
	}{
		\innLearningRate_\outCurRound \pth{
			1 -
			\anyWrongArmProb
	}}
	%
	%
	+ 
	\problemScale^2
	\innLearningRate_\outCurRound
	+
	\boundExplCost
}
\\\overset{\tcref{eq:mabOutLrBoundTsallis,thm:lr-meta-regret}}&{\leq}
\min_{\innPseudoLearningRate} \brk{
\boundLrRegret \pth{\innPseudoLearningRate}
+
\E
\sum_{\outCurRound=1}^{\outNumRounds} \pth{
	\frac{1}{
		{\innPseudoLearningRate}
		\pth{1 - \mabDim \wrongArmProb}
	}
	\betaDivergenceHalf{\basisVector{\hat{j}_{\outCurRound}}^\probLowerLimit}{\innInitial_\outCurRound}
	+ 
	\innPseudoLearningRate
}
\problemScale
+ \boundExplCost \outNumRounds
}
\\&{=}
\min_{\innPseudoLearningRate} \brk{
	\boundExplCost \outNumRounds
	+
	\boundLrRegret \pth{\innPseudoLearningRate}
	+
	\problemScale
	\outNumRounds
	\innPseudoLearningRate
	+
	\frac{
		\problemScale
	}{
		\innPseudoLearningRate \pth{1 - \mabDim \wrongArmProb}
	}	
	\E
	\sum_{\outCurRound=1}^{\outNumRounds} {
		\betaDivergenceHalf{\basisVector{\hat{j}_{\outCurRound}}^\probLowerLimit}{\innInitial_\outCurRound}
	}
}
\\\overset{\tcref{thm:init-meta-regret}}&{\leq}
\min_{\innPseudoLearningRate} \Bigg[
	\boundExplCost \outNumRounds
	+
	\boundLrRegret \pth{\innPseudoLearningRate}
	+
	\problemScale
	\outNumRounds
	\innPseudoLearningRate
\\&\qquad\qquad
	+
	\frac{
		\problemScale
	}{
		\innPseudoLearningRate \pth{1 - \mabDim \wrongArmProb}
	}
	\pth{
		\boundInitRegret
		+
		\E
		\min_{\innInitial \in \truncatedSimplex{\probLowerLimit}}
		\sum_{\outCurRound=1}^{\outNumRounds}
		\betaDivergenceHalf{\basisVector{\hat{j}_{\outCurRound}}^\probLowerLimit}{\innInitial}
	}
\Bigg]
\\\overset{\tcref{thm:init-metalearning-bih-loss}}&{\leq}
\min_{\innPseudoLearningRate} \Bigg[
	\boundExplCost \outNumRounds
	+
	\boundLrRegret \pth{\innPseudoLearningRate}
	+
	\problemScale
	\outNumRounds
	\innPseudoLearningRate
\\&\qquad\qquad
	+
	\frac{
		\problemScale
	}{
		\innPseudoLearningRate \pth{1 - \mabDim \wrongArmProb}
	}
	\pth{
		\boundInitRegret
		+
		\boundInitEstimationCost
		+
		\EntropyTsallisHalfSc{\initComparator}
		\outNumRounds
	}
\Bigg]
\end{aligned}
\end{equation}

\end{proof}

\section{Parameter optimization (proof of Corollary~\ref{thm:total-regret-asymptotic})}
\label{sec:parameter-optimization}

{
\newcommand*{\initUpperBound}{\boundInitRegret}
\newcommand*{\explUpperBound}{\boundExplCost \outNumRounds}

\begin{corollary}[\cref{thm:total-regret-asymptotic}]
	\label{thm:total-regret-asymptotic-appendix}
	Let \cref{thm:assum:minimal-gap,thm:assum:positive-probability} hold, and additionally assume $ \innNumRounds \geq \bigOmega{}{\frac{\mabDim {\pth{\log \mabDim}}^{\nicefrac{7}{3}}}{\gapMin^{\nicefrac{10}{3}}} } $.
	For parameter values
	\begin{equation*}\label{key}
	\begin{aligned}
	&
	\lrEpsilon
	=
	\bigTheta{}{
		\frac{
			\sqrt[3]{\log{\outNumRounds}}
		}{
			\sqrt[3]{\outNumRounds}
			\sqrt{ 1 - \mabDim \wrongArmProb }
			\sqrt[4]{\probLowerLimit}
		}
	}
	,&\quad
	\probLowerLimit
	=
	\bigTheta{}{
		\frac{1}{ 
			\gapMin ^ {\nicefrac{4}{7}}
			\innNumRounds ^ {\nicefrac{4}{7}}
			\mabDim ^ {\nicefrac{3}{7}}
		}
	},
	\end{aligned}
	\end{equation*}
	the expected total regret of \acrshort{meta-inf} satisfies the asymptotic bound
	\begin{equation*}\label{key}
	\begin{aligned}
	\E \regretTotal
	&{\leq}
	\bigOSymbol \Bigg{(}
	\outNumRounds \sqrt{\innNumRounds} \sqrt{\mabDim}
	\Bigg{(}
	\sqrt{\frac{
			\EntropyTsallisHalfSc{\initComparator}
		}{
			\sqrt{\mabDim}
	}}
	%
	%
	+
	\frac{
		\mabDim ^ {\nicefrac{1}{14}}
	}{
		\gapMin ^ {\nicefrac{4}{7}}
		\innNumRounds ^ {\nicefrac{1}{14}}
	}
	+
	\frac{
		\sqrt[3]{\log{\outNumRounds}}
		\innNumRounds ^ {\nicefrac{1}{7}}	
	}{
		\sqrt[3]{\outNumRounds}
		\mabDim ^ {\nicefrac{1}{7}}
	}
	+
	\frac{
		\sqrt{ \log{\outNumRounds} }
		\innNumRounds ^ {\nicefrac{1}{7}}
		\mabDim ^ {\nicefrac{3}{28}}
	}{
		\sqrt{\outNumRounds}
	}
	%
	%
	\Bigg{)}
	\Bigg{)}
	.
	\end{aligned}
	\end{equation*}
\end{corollary}
\begin{proof}
For simplicity, during parameter optimization we consider $ \lrEpsilon $ instead of $ \min\braces{\frac{\lrEpsilon^2}{\innPseudoLearningRate}, \lrEpsilon} $ in $\boundLrRegret$.	
We then optimize $\lrEpsilon$ and $\innPseudoLearningRate$ separately.

\begin{equation}\label{key}
\begin{aligned}
\E \regretTotal
&{\leq}
\min_{\innPseudoLearningRate}
\bigg{[}
\boundExplCost \outNumRounds
+
\boundLrRegret \pth{\innPseudoLearningRate}
+ 
\problemScale
\outNumRounds
\innPseudoLearningRate
+
\frac{\problemScale}{1 - \mabDim \wrongArmProb}
\pth{
	\boundInitRegret 
	+
	\boundInitEstimationCost
	+
	\EntropyTsallisHalfSc{\initComparator}
	\outNumRounds
}
\frac{1}{\innPseudoLearningRate}
\bigg{]}
\\&{=}
\min_{\innPseudoLearningRate}
\bigg{[}
\min\braces{\frac{\lrEpsilon^2}{\innPseudoLearningRate}, \lrEpsilon}
{\frac{\outNumRounds \sqrt{\innNumRounds} \sqrt[4]{\mabDim} }{\sqrt{2}} }
+
\frac{
	\sqrt{\innNumRounds} \sqrt[4]{\mabDim} 
}{ 
	\lrEpsilon^2 {\pth{1 - \mabDim \wrongArmProb}}^{\nicefrac{3}{2}} \probLowerLimit^{\nicefrac{3}{4}} }
\pth{\log\pth{\outNumRounds + 1} + 1}
\\&{\quad}
+
{\frac{\sqrt{\innNumRounds} \sqrt[4]{\mabDim} }{\sqrt{2}} }
\pth{
	\frac{1}{1 - \mabDim \wrongArmProb}
	\pth{
		\initUpperBound 
		+
		\boundInitEstimationCost + \EntropyTsallisHalfSc{\initComparator} \outNumRounds
	}
	\frac{1}{\innPseudoLearningRate}
	+ 
	\outNumRounds
	\innPseudoLearningRate
}
+
\explUpperBound
\bigg{]}
\\&{\leq}
\min_{\innPseudoLearningRate}
\bigg{[}
\lrEpsilon
{\frac{\outNumRounds \sqrt{\innNumRounds} \sqrt[4]{\mabDim} }{\sqrt{2}} }
+
\frac{
	\sqrt{\innNumRounds} \sqrt[4]{\mabDim} 
}{ 
	\lrEpsilon^2 {\pth{1 - \mabDim \wrongArmProb}}^{\nicefrac{3}{2}} \probLowerLimit^{\nicefrac{3}{4}} }
\pth{\log\pth{\outNumRounds + 1} + 1}
\\&{\quad}
+
{\frac{\sqrt{\innNumRounds} \sqrt[4]{\mabDim} }{\sqrt{2}} }
\pth{
	\frac{1}{1 - \mabDim \wrongArmProb}
	\pth{
		\initUpperBound 
		+
		\boundInitEstimationCost + \EntropyTsallisHalfSc{\initComparator} \outNumRounds
	}
	\frac{1}{\innPseudoLearningRate}
	+ 
	\outNumRounds
	\innPseudoLearningRate
}
+
\explUpperBound
\bigg{]}
\\&{=}
{\frac{\outNumRounds \sqrt{\innNumRounds} \sqrt[4]{\mabDim} }{\sqrt{2}} }
\Bigg(
\lrEpsilon
+
\frac{
	\sqrt{2} \pth{\log\pth{\outNumRounds + 1} + 1}
}{ 
	\lrEpsilon^2 \outNumRounds {\pth{1 - \mabDim \wrongArmProb}}^{\nicefrac{3}{2}} \probLowerLimit^{\nicefrac{3}{4}} 
}
\\&{\quad}
+
\min_{\innPseudoLearningRate}
\bigg{[}
{
	\frac{1}{1 - \mabDim \wrongArmProb}
	\pth{
		\frac{\initUpperBound}{\outNumRounds} 
		+
		\frac{\boundInitEstimationCost}{\outNumRounds}
		+ 
		\EntropyTsallisHalfSc{\initComparator}
	}
	\frac{1}{\innPseudoLearningRate}
	+ 
	\innPseudoLearningRate
}
\bigg{]}
\Bigg)
+
\explUpperBound
\\&{=}
{\frac{\outNumRounds \sqrt{\innNumRounds} \sqrt[4]{\mabDim} }{\sqrt{2}} }
\Bigg(
\lrEpsilon
+
\frac{
	\sqrt{2} \pth{\log\pth{\outNumRounds + 1} + 1}
}{ 
	\lrEpsilon^2 \outNumRounds {\pth{1 - \mabDim \wrongArmProb}}^{\nicefrac{3}{2}} \probLowerLimit^{\nicefrac{3}{4}} 
}
\\&{\quad}
+
{
	\frac{2}{\sqrt{1 - \mabDim \wrongArmProb}}
	\sqrt{
		\frac{\initUpperBound}{\outNumRounds} 
		+
		\frac{\boundInitEstimationCost}{\outNumRounds}
		+ 
		\EntropyTsallisHalfSc{\initComparator}
	}
}
\Bigg)
+
\explUpperBound
\\&{\leq}
{\frac{\outNumRounds \sqrt{\innNumRounds} \sqrt[4]{\mabDim} }{\sqrt{2}} }
\Bigg(
\lrEpsilon
+
\frac{
	\sqrt{2} \pth{\log\pth{\outNumRounds + 1} + 1}
}{ 
	\lrEpsilon^2 \outNumRounds {\pth{1 - \mabDim \wrongArmProb}}^{\nicefrac{3}{2}} \probLowerLimit^{\nicefrac{3}{4}} 
}
\\&{\quad}
+
{
	\frac{2}{\sqrt{1 - \mabDim \wrongArmProb}}
	\pth{
		\sqrt{\frac{\initUpperBound}{\outNumRounds} }
		+
		\sqrt{\frac{\boundInitEstimationCost}{\outNumRounds}}
		+ 
		\sqrt{\EntropyTsallisHalfSc{\initComparator}}
	}
}
\Bigg)
+
\explUpperBound
\end{aligned}
\end{equation}

Setting
\begin{equation}\label{key}
\lrEpsilon = \sqrt[3]{\frac{
		2 \sqrt{2} \pth{\log\pth{\outNumRounds + 1} + 1}
	}{
		\outNumRounds {\pth{1 - \mabDim \wrongArmProb}}^{\nicefrac{3}{2}} \probLowerLimit^{\nicefrac{3}{4}}
}}
\end{equation}
gives:
\begin{equation}\label{key}
\begin{aligned}
\E \regretTotal
&{\leq}
{\frac{\outNumRounds \sqrt{\innNumRounds} \sqrt[4]{\mabDim} }{\sqrt{2}} }
\Bigg(
\frac{3 \sqrt{2}}{2}
\sqrt[3]{\frac{
		\pth{\log\pth{\outNumRounds + 1} + 1}
	}{
		\outNumRounds {\pth{1 - \mabDim \wrongArmProb}}^{\nicefrac{3}{2}} \probLowerLimit^{\nicefrac{3}{4}}
}}
\\&{\qquad}
+
{
	\frac{2}{\sqrt{1 - \mabDim \wrongArmProb}}
	\pth{
		\sqrt{\frac{\initUpperBound}{\outNumRounds} }
		+
		\sqrt{\frac{\boundInitEstimationCost}{\outNumRounds}}
		+ 
		\sqrt{\EntropyTsallisHalfSc{\initComparator}}
	}
}
\Bigg)
+
\explUpperBound
\\&{=}
{\sqrt{2}}{{\outNumRounds \sqrt{\innNumRounds} \sqrt[4]{\mabDim} } }
\Bigg(
{\frac{
		3 \sqrt{2}
		\sqrt[3]{\pth{\log\pth{\outNumRounds + 1} + 1}}
	}{
		4
		\sqrt[3]{\outNumRounds} \sqrt{1 - \mabDim \wrongArmProb} \probLowerLimit^{\nicefrac{1}{4}}
}}
\\&{\qquad}
+
{
	\frac{1}{\sqrt{1 - \mabDim \wrongArmProb}}
	\pth{
		\sqrt{\frac{\initUpperBound}{\outNumRounds} }
		+
		\sqrt{\frac{\boundInitEstimationCost}{\outNumRounds}}
		+ 
		\sqrt{\EntropyTsallisHalfSc{\initComparator}}
	}
}
\Bigg)
+
\explUpperBound
\\&{=}
{\sqrt{2}} \outNumRounds \sqrt{\innNumRounds} \sqrt[4]{\mabDim}
\Bigg{(}
\frac{3 \sqrt{2}}{4}
\frac{
	\sqrt[3]{\pth{\log\pth{\outNumRounds + 1} + 1}}
}{
	\sqrt[3]{\outNumRounds} \sqrt{1 - \mabDim \wrongArmProb}
	\probLowerLimit^{\nicefrac{1}{4}}
}
\\&{\qquad}
+
\frac{1}{\sqrt{1 - \mabDim \wrongArmProb}}
\pth{
	2 \sqrt[4]{2}
	\frac{
		\sqrt{ \log{\outNumRounds} + 1 }
		\sqrt[4]{\mabDim}
	}{
		\sqrt{\outNumRounds}
		\sqrt[4]{\probLowerLimit}
	}
	+
	\sqrt{
		\EntropyTsallisHalfSc{\initComparator}
	}
	+
	\sqrt{6}
	\frac{
		\sqrt{ 
			\mabDim \wrongArmProb
		}
	}{
		\sqrt[4]{\probLowerLimit}
	}
}
+
\sqrt{\innNumRounds} \mabDim^{\nicefrac{3}{4}} \probLowerLimit
\end{aligned}
\end{equation}

Now we use the following simplifications:
\begin{equation}\label{key}
\frac{1}{\sqrt{1 - \mabDim \wrongArmProb}}
=
\sqrt{1 + \frac{\mabDim \wrongArmProb}{1 - \mabDim \wrongArmProb}}
\leq
1 + \sqrt{\frac{\mabDim \wrongArmProb}{1 - \mabDim \wrongArmProb}}
\end{equation}
\begin{equation}\label{key}
\begin{aligned}
\sqrt{\frac{\mabDim \wrongArmProb}{1 - \mabDim \wrongArmProb}}
&=
\sqrt{\frac{1}{\mabDim^{-1} \wrongArmProb^{-1} - 1}}
=
\sqrt{
	\frac{
		1
	}{
		\exp{
			\frac{3}{28}
			\gapMin^2
			\innNumRounds
			\probLowerLimit 
			- \log{\mabDim}
		}
		- 1
	}
}
\leq
\sqrt{
	\frac{
		1
	}{
		\frac{3}{28}
		\gapMin^2
		\innNumRounds
		\probLowerLimit 
		- \log{\mabDim}
	}
}
\\&\leq
\sqrt{
	\frac{
		1
	}{
		\frac{3}{56}
		\gapMin^2
		\innNumRounds
		\probLowerLimit 
	}
}
=
\sqrt{\frac{56}{3}}
\frac{
	1
}{
	\gapMin
	\sqrt{\innNumRounds}
	\sqrt{\probLowerLimit}
}
\end{aligned}
\end{equation}

And then:
\begin{equation}\label{eq:total-regret-with-gap-before-delta}
\begin{aligned}
\E \regretTotal
&{\leq}
{\sqrt{2}} \outNumRounds \sqrt{\innNumRounds} \sqrt[4]{\mabDim}
\Bigg{(}
\frac{3 \sqrt{2}}{4}
\frac{
	\sqrt[3]{\pth{\log\pth{\outNumRounds + 1} + 1}}
}{
	\sqrt[3]{\outNumRounds}
}
\pth{
	\frac{1}{\probLowerLimit^{\nicefrac{1}{4}}}
	+
	\sqrt{\frac{56}{3}}
	\frac{
		1
	}{
		\gapMin
		\sqrt{\innNumRounds}
		\probLowerLimit^{\nicefrac{3}{4}}
	}
}
\\&{\qquad}
+
\pth{
	1
	+
	\sqrt{\frac{56}{3}}
	\frac{
		1
	}{
		\gapMin
		\sqrt{\innNumRounds}
		\sqrt{\probLowerLimit}
	}
}
\pth{
	2 \sqrt[4]{2}
	\frac{
		\sqrt{ \log{\outNumRounds} + 1 }
		\sqrt[4]{\mabDim}
	}{
		\sqrt{\outNumRounds}
		\sqrt[4]{\probLowerLimit}
	}
	+
	\sqrt{
		\EntropyTsallisHalfSc{\initComparator}
	}
}
\\&{\qquad}
+
4 \sqrt{7}
\frac{
	1
}{
	\gapMin
	\sqrt{\innNumRounds}
	\probLowerLimit^{\nicefrac{3}{4}}
}
+
\sqrt{\innNumRounds} \mabDim^{\nicefrac{3}{4}} \probLowerLimit
\Bigg{)}
\end{aligned}
\end{equation}

Out of all terms that depend on $\probLowerLimit$ we minimize the sum of those that originate from $ \boundInitEstimationCost $ and $ \boundExplCost $:
\begin{math}
4 \sqrt{7}
\frac{
	1
}{
	\gapMin
	\sqrt{\innNumRounds}
	\probLowerLimit^{\nicefrac{3}{4}}
}
+
\sqrt{\innNumRounds} \mabDim^{\nicefrac{3}{4}} \probLowerLimit
\end{math}.
For simplicity, since we care only for the asymptotic bound, we ignore the constants and use 
\begin{math}
\probLowerLimit
=
\frac{
	1
}{
	\gapMin ^ {\nicefrac{4}{7}}
	\innNumRounds ^ {\nicefrac{4}{7}}
	\mabDim ^ {\nicefrac{3}{7}}
}
\end{math}.
And then
\begin{equation}\label{key}
\begin{aligned}
\E \regretTotal
&{\leq}
{\sqrt{2}} \outNumRounds \sqrt{\innNumRounds} \sqrt[4]{\mabDim}
\Bigg{(}
\frac{3 \sqrt{2}}{4}
\frac{
	\sqrt[3]{\pth{\log\pth{\outNumRounds + 1} + 1}}
}{
	\sqrt[3]{\outNumRounds}
}
\pth{
	\gapMin ^ {\nicefrac{1}{7}}
	\innNumRounds ^ {\nicefrac{1}{7}}
	\mabDim ^ {\nicefrac{3}{28}}
	+
	\sqrt{\frac{56}{3}}
	\frac{
		\mabDim ^ {\nicefrac{9}{28}}
	}{
		\gapMin ^ {\nicefrac{4}{7}}
		\innNumRounds ^ {\nicefrac{1}{14}}
	}
}
\\&{\qquad}
+
\pth{
	1
	+
	\sqrt{\frac{56}{3}}
	\frac{
		\mabDim ^ {\nicefrac{3}{14}}
	}{
		\gapMin ^ {\nicefrac{5}{7}}
		\innNumRounds ^ {\nicefrac{3}{14}}
	}
}
\pth{
	2 \sqrt[4]{2}
	\frac{
		\sqrt{ \log{\outNumRounds} + 1 }
		\gapMin ^ {\nicefrac{1}{7}}
		\innNumRounds ^ {\nicefrac{1}{7}}
		\mabDim ^ {\nicefrac{3}{28}}
		\sqrt[4]{\mabDim}
	}{
		\sqrt{\outNumRounds}
	}
	+
	\sqrt{
		\EntropyTsallisHalfSc{\initComparator}
	}
}
\\&{\qquad}
+
4 \sqrt{7}
\frac{
	\mabDim ^ {\nicefrac{9}{28}}
}{
	\gapMin ^ {\nicefrac{4}{7}}
	\innNumRounds ^ {\nicefrac{1}{14}}
}
+
\frac{
	\mabDim^{\nicefrac{9}{28}}
}{
	\gapMin ^ {\nicefrac{4}{7}}
	\innNumRounds ^ {\nicefrac{1}{14}}
}
\Bigg{)}
\\&{=}
{\sqrt{2}} \outNumRounds \sqrt{\innNumRounds} \sqrt{\mabDim}
\Bigg{(}
\frac{3 \sqrt{2}}{4}
\frac{
	\sqrt[3]{\pth{\log\pth{\outNumRounds + 1} + 1}}
}{
	\sqrt[3]{\outNumRounds}
}
\pth{
	\frac{
		\gapMin ^ {\nicefrac{1}{7}}
		\innNumRounds ^ {\nicefrac{1}{7}}
	}{
		\mabDim ^ {\nicefrac{1}{7}}
	}
	+
	\sqrt{\frac{56}{3}}
	\frac{
		\mabDim ^ {\nicefrac{1}{14}}
	}{
		\gapMin ^ {\nicefrac{4}{7}}
		\innNumRounds ^ {\nicefrac{1}{14}}
	}
}
\\&{\qquad}
+
\pth{
	1
	+
	\sqrt{\frac{56}{3}}
	\frac{
		\mabDim ^ {\nicefrac{3}{14}}
	}{
		\gapMin ^ {\nicefrac{5}{7}}
		\innNumRounds ^ {\nicefrac{3}{14}}
	}
}
\pth{
	2 \sqrt[4]{2}
	\frac{
		\sqrt{ \log{\outNumRounds} + 1 }
		\gapMin ^ {\nicefrac{1}{7}}
		\innNumRounds ^ {\nicefrac{1}{7}}
		\mabDim ^ {\nicefrac{3}{28}}
	}{
		\sqrt{\outNumRounds}
	}
	+
	\sqrt{\frac{
		\EntropyTsallisHalfSc{\initComparator}
	}{
		\sqrt{\mabDim}
	}}
}
\\&{\qquad}
+
\pth{4 \sqrt{7} + 1}
\frac{
	\mabDim ^ {\nicefrac{1}{14}}
}{
	\gapMin ^ {\nicefrac{4}{7}}
	\innNumRounds ^ {\nicefrac{1}{14}}
}
\Bigg{)}
\\&{\leq}
\bigOSymbol \Bigg{(}
\outNumRounds \sqrt{\innNumRounds} \sqrt{\mabDim}
\Bigg{(}
\frac{
	\sqrt[3]{\log{\outNumRounds}}
}{
	\sqrt[3]{\outNumRounds}
}
\pth{
	\frac{
		\gapMin ^ {\nicefrac{1}{7}}
		\innNumRounds ^ {\nicefrac{1}{7}}
	}{
		\mabDim ^ {\nicefrac{1}{7}}
	}
	+
	\frac{
		\mabDim ^ {\nicefrac{1}{14}}
	}{
		\gapMin ^ {\nicefrac{4}{7}}
		\innNumRounds ^ {\nicefrac{1}{14}}
	}
}
\\&{\qquad}
+
\pth{
	1
	+
	\frac{
		\mabDim ^ {\nicefrac{3}{14}}
	}{
		\gapMin ^ {\nicefrac{5}{7}}
		\innNumRounds ^ {\nicefrac{3}{14}}
	}
}
\pth{
	\frac{
		\sqrt{ \log{\outNumRounds} }
		\gapMin ^ {\nicefrac{1}{7}}
		\innNumRounds ^ {\nicefrac{1}{7}}
		\mabDim ^ {\nicefrac{3}{28}}
	}{
		\sqrt{\outNumRounds}
	}
	+
	\sqrt{\frac{
			\EntropyTsallisHalfSc{\initComparator}
		}{
			\sqrt{\mabDim}
	}}
}
%
%
+
\frac{
	\mabDim ^ {\nicefrac{1}{14}}
}{
	\gapMin ^ {\nicefrac{4}{7}}
	\innNumRounds ^ {\nicefrac{1}{14}}
}
\Bigg{)}
\Bigg{)}
\\&{\leq}
\bigOSymbol \Bigg{(}
\outNumRounds \sqrt{\innNumRounds} \sqrt{\mabDim}
\Bigg{(}
\frac{
	\sqrt[3]{\log{\outNumRounds}}
}{
	\sqrt[3]{\outNumRounds}
}
\frac{
	\gapMin ^ {\nicefrac{1}{7}}
	\innNumRounds ^ {\nicefrac{1}{7}}
}{
	\mabDim ^ {\nicefrac{1}{7}}
}
\\&{\qquad}
+
\pth{
	1
	+
	\frac{
		\mabDim ^ {\nicefrac{3}{14}}
	}{
		\gapMin ^ {\nicefrac{5}{7}}
		\innNumRounds ^ {\nicefrac{3}{14}}
	}
}
\pth{
	\frac{
		\sqrt{ \log{\outNumRounds} }
		\gapMin ^ {\nicefrac{1}{7}}
		\innNumRounds ^ {\nicefrac{1}{7}}
		\mabDim ^ {\nicefrac{3}{28}}
	}{
		\sqrt{\outNumRounds}
	}
	+
	\sqrt{\frac{
			\EntropyTsallisHalfSc{\initComparator}
		}{
			\sqrt{\mabDim}
	}}
}
\\&{\qquad}
+
\pth{
	1
	+
	\frac{
		\sqrt[3]{\log{\outNumRounds}}
	}{
		\sqrt[3]{\outNumRounds}
	}
}
\frac{
	\mabDim ^ {\nicefrac{1}{14}}
}{
	\gapMin ^ {\nicefrac{4}{7}}
	\innNumRounds ^ {\nicefrac{1}{14}}
}
\Bigg{)}
\Bigg{)}
\\&{\leq}
\bigOSymbol \Bigg{(}
\outNumRounds \sqrt{\innNumRounds} \sqrt{\mabDim}
\Bigg{(}
\frac{
	\sqrt[3]{\log{\outNumRounds}}
	\innNumRounds ^ {\nicefrac{1}{7}}	
}{
	\sqrt[3]{\outNumRounds}
	\mabDim ^ {\nicefrac{1}{7}}
}
+
\frac{
	\sqrt{ \log{\outNumRounds} }
	\innNumRounds ^ {\nicefrac{1}{7}}
	\mabDim ^ {\nicefrac{3}{28}}
}{
	\sqrt{\outNumRounds}
}
%
%
+
\sqrt{\frac{
		\EntropyTsallisHalfSc{\initComparator}
	}{
		\sqrt{\mabDim}
}}
%
%
+
\frac{
	\mabDim ^ {\nicefrac{1}{14}}
}{
	\gapMin ^ {\nicefrac{4}{7}}
	\innNumRounds ^ {\nicefrac{1}{14}}
}
\Bigg{)}
\Bigg{)}
\end{aligned}
\end{equation}

In the last inequation we used the following considerations:
\begin{equation}\label{key}
\begin{matrix}
\gapMin \leq 1
&\implies&
\gapMin^ {\nicefrac{1}{7}} \leq \bigO{}{1}
\\
\innNumRounds \geq \bigOmega{}{\frac{\mabDim {\pth{\log \mabDim}}^{\nicefrac{7}{3}}}{\gapMin^{\nicefrac{10}{3}}} }
&\implies&
	\frac{
	\mabDim ^ {\nicefrac{3}{14}}
}{
	\gapMin ^ {\nicefrac{5}{7}}
	\innNumRounds ^ {\nicefrac{3}{14}}
}
\leq
\bigO{}{1}
\\
&&
\frac{
	\sqrt[3]{\log{\outNumRounds}}
}{
	\sqrt[3]{\outNumRounds}
}
\leq
\bigO{}{1}
\end{matrix}
\end{equation}

\end{proof}

\newcommand*{\gapAlpha}{\alpha}
\newcommand*{\gapBeta}{\beta}

\begin{remark}
	\label{thm:total-regret-asymptotic-powerGap-appendix}
	If $ \gapMin = \bigTheta{}{\frac{1}{\innNumRounds^\gapAlpha \mabDim^\gapBeta}} $ for some $\gapAlpha \geq 0$, $\gapBeta \geq 0$, then the requirement on $\innNumRounds$ in \cref{thm:total-regret-asymptotic-appendix} becomes $ \innNumRounds \geq \bigOmega{}{\mabDim^{\pth{\frac{3 + 10 \gapBeta}{3 - 10 \gapAlpha}}} {\pth{\log \mabDim}}^{\pth{\frac{7}{3 - 10 \gapAlpha}}}} $, and it implies $\gapAlpha < \nicefrac{3}{10}$.
	To try and cope with gap that goes to zero with $\innNumRounds$ at a quicker rate, we can consider the research directions mentioned in \cref{ssec:problem-owo-mab}.
	Specifically, if the minimal gap is too small to allow reliable best-arm identification, than the empirical distribution of best arms $\initComparator$ is not a good measure of the problem complexity.
	Another measure of complexity should be defined, fitting the problem.
	For example, instead of declaring one arm as the best (per episode) and characterizing their empirical distribution, one may collect contributions of several arms on each episode with weight depending on the per-arm per-episode gap.
	The corresponding algorithm would then perhaps use soft decision incorporating the likelihood each arm to be the best, instead of the hard decision used currently.
	However, defining the aforementioned complexity measure in a concise way (without explicit dependence on the per-arm gaps) does not seem trivial.
\end{remark}

\begin{remark}
	In \cref{thm:total-regret-asymptotic-appendix} the value of $\probLowerLimit$ depends on the gap $\gapMin$.
	If $\gapMin$ is unknown, we can use
\begin{math}
	\probLowerLimit =
	\bigTheta{}{
		\frac{1}{ 
			\innNumRounds ^ {\nicefrac{4}{7}}
			\mabDim ^ {\nicefrac{3}{7}}
		}
	}
\end{math}.
	In order to satisfy \cref{thm:assum:positive-probability}, the requirement now is $ \innNumRounds \geq \bigOmega{}{\frac{\mabDim {\pth{\log \mabDim}}^{\nicefrac{7}{3}}}{\gapMin^{\nicefrac{14}{3}}} } $, instead of
 $ \innNumRounds \geq \bigOmega{}{\frac{\mabDim {\pth{\log \mabDim}}^{\nicefrac{7}{3}}}{\gapMin^{\nicefrac{10}{3}}} } $.
The bound in this case is:
\begin{equation}\label{key}
\begin{aligned}
\E \regretTotal
&{\leq}
{\sqrt{2}} \outNumRounds \sqrt{\innNumRounds} \sqrt[4]{\mabDim}
\Bigg{(}
\frac{3 \sqrt{2}}{4}
\frac{
	\sqrt[3]{\pth{\log\pth{\outNumRounds + 1} + 1}}
}{
	\sqrt[3]{\outNumRounds}
}
\pth{
	\innNumRounds^{\nicefrac{1}{7}}
	\mabDim^{\nicefrac{3}{28}}
	+
	\sqrt{\frac{56}{3}}
	\frac{
		\mabDim^{\nicefrac{1}{14}}
		\sqrt[4]{\mabDim}
	}{
		\gapMin
		\innNumRounds^{\nicefrac{1}{14}}
	}
}
\\&{\qquad}
+
\pth{
	1
	+
	\sqrt{\frac{56}{3}}
	\frac{
		\mabDim^{\nicefrac{3}{14}}		
	}{
		\gapMin
		\innNumRounds^{\nicefrac{3}{14}}
	}
}
\pth{
	2 \sqrt[4]{2}
	\frac{
		\sqrt{ \log{\outNumRounds} + 1 }
		\innNumRounds^{\nicefrac{1}{7}}
		\mabDim^{\nicefrac{3}{28}}
		\sqrt[4]{\mabDim}
	}{
		\sqrt{\outNumRounds}
	}
	+
	\sqrt{
		\EntropyTsallisHalfSc{\initComparator}
	}
}
\\&{\qquad}
+
4 \sqrt{7}
\frac{
	\mabDim^{\nicefrac{1}{14}}
	\sqrt[4]{\mabDim}
}{
	\gapMin
	\innNumRounds^{\nicefrac{1}{14}}
}
+
\frac{
	\mabDim^{\nicefrac{1}{14}}
	\sqrt[4]{\mabDim}
}{
	\innNumRounds^{\nicefrac{1}{14}}
}
\Bigg{)}
\\&{=}
{\sqrt{2}} \outNumRounds \sqrt{\innNumRounds} \sqrt{\mabDim}
\Bigg{(}
\frac{3 \sqrt{2}}{4}
\frac{
	\sqrt[3]{\pth{\log\pth{\outNumRounds + 1} + 1}}
}{
	\sqrt[3]{\outNumRounds}
}
\pth{
	\frac{
		\innNumRounds^{\nicefrac{1}{7}}
	}{
		\mabDim^{\nicefrac{1}{7}}
	}
	+
	\sqrt{\frac{56}{3}}
	\frac{
		\mabDim^{\nicefrac{1}{14}}
	}{
		\gapMin
		\innNumRounds^{\nicefrac{1}{14}}
	}
}
\\&{\qquad}
+
\pth{
	1
	+
	\sqrt{\frac{56}{3}}
	\frac{
		\mabDim^{\nicefrac{3}{14}}		
	}{
		\gapMin
		\innNumRounds^{\nicefrac{3}{14}}
	}
}
\pth{
	2 \sqrt[4]{2}
	\frac{
		\sqrt{ \log{\outNumRounds} + 1 }
		\innNumRounds^{\nicefrac{1}{7}}
		\mabDim^{\nicefrac{3}{28}}
	}{
		\sqrt{\outNumRounds}
	}
	+
	\sqrt{\frac{
			\EntropyTsallisHalfSc{\initComparator}
		}{
			\sqrt{\mabDim}
	}}
}
\\&\qquad
+
\pth{
	\frac{
		4 \sqrt{7}
	}{
		\gapMin
	}
	+
	1
}
\frac{
	\mabDim^{\nicefrac{1}{14}}
}{
	\innNumRounds^{\nicefrac{1}{14}}
}
\Bigg{)}
\\&\leq
\bigOSymbol \Bigg{(}
\outNumRounds \sqrt{\innNumRounds} \sqrt{\mabDim}
\Bigg{(}
\frac{
	\sqrt[3]{\log{\outNumRounds}}
	\innNumRounds^{\nicefrac{1}{7}}
}{
	\sqrt[3]{\outNumRounds}
	\mabDim^{\nicefrac{1}{7}}
}
\\&{\qquad}
+
\pth{
	1
	+
	\frac{
		\mabDim^{\nicefrac{3}{14}}		
	}{
		\gapMin
		\innNumRounds^{\nicefrac{3}{14}}
	}
}
\pth{
	\frac{
		\sqrt{ \log{\outNumRounds} }
		\innNumRounds^{\nicefrac{1}{7}}
		\mabDim^{\nicefrac{3}{28}}
	}{
		\sqrt{\outNumRounds}
	}
	+
	\sqrt{\frac{
			\EntropyTsallisHalfSc{\initComparator}
		}{
			\sqrt{\mabDim}
	}}
}
+
\frac{
	\mabDim^{\nicefrac{1}{14}}
}{
	\gapMin
	\innNumRounds^{\nicefrac{1}{14}}
}
\Bigg{)}
\Bigg{)}
\end{aligned}
\end{equation}
	
This bound becomes
				\begin{math}
\bigO{}{
	\outNumRounds \sqrt{\innNumRounds}
	\sqrt[4]{\mabDim}
}
\end{math}
when
$ \outNumRounds \geq \bigOmega{}{\innNumRounds} $,
$ \innNumRounds \geq \bigOmega{}{\frac{\mabDim^{\nicefrac{9}{2}}}{\gapMin^{14} }}$ and
$ \EntropyTsallisHalfSc{\initComparator} \leq \bigO{}{1} $.
So the dependence of $\innNumRounds$ on $\gapMin$ is worse compared to the case that $\gapMin$ is known (in which case the corresponding condition is $ \innNumRounds \geq \bigOmega{}{\frac{\mabDim^{\nicefrac{9}{2}}}{\gapMin^{8} }}$, as shown in \cref{tab:proposed-bound}).

\end{remark}

}


\end{document}